\documentclass[conference]{IEEEtran}

\usepackage{times}
\usepackage{epsfig}
\usepackage{graphicx}
\usepackage{amsmath}
\usepackage{amssymb}
\usepackage{multirow}
\usepackage{makecell}
\usepackage{amsthm}
\usepackage{algorithm,algpseudocode}
\usepackage{cite}
\usepackage{adjustbox}
\usepackage{multicol}
\usepackage[bookmarks=true,hidelinks]{hyperref}
\hypersetup{
	colorlinks   = true, 
	urlcolor     = blue,
	linkcolor    = blue, 
	citecolor   = blue 
}
\usepackage[capitalize]{cleveref}
\usepackage{booktabs}

\usepackage{xspace}

\newcommand{\ie}{i.e.\xspace}

\newcommand{\select}[2]{{#2}}

\def\amm{\mathrm{\sf\small DABA}} 
\def\mm{\mathrm{\sf\small DUBA}} 

\def\ceres{\mathrm{\sf\small Ceres}}
\def\deeplm{\mathrm{\sf\small DeepLM}}
\def\dr{\mathrm{\sf\small DR}}
\def\admm{\mathrm{\sf\small ADMM}}

\usepackage{amsfonts}
\usepackage{amsmath}
\usepackage{amssymb}

\usepackage{graphicx,url,tabularx}

\usepackage{bm}
\usepackage[T1]{fontenc}
\usepackage{latexsym}
\usepackage{xstring}
\usepackage{relsize}

\usepackage{multirow}
\usepackage{xcolor}
\usepackage{subcaption}

\usepackage[shortlabels]{enumitem}

\newcolumntype{P}[1]{>{\centering\arraybackslash}p{#1}}
\newcolumntype{M}[1]{>{\centering\arraybackslash}m{#1}}

\newtheorem{theorem}{Theorem}
\newtheorem{problem}{Problem}

\newtheorem{lemma}[theorem]{Lemma}
\newtheorem{assumption}{Assumption}

\newtheorem{proposition}[theorem]{Proposition}

\newcommand{\bdmath}{\begin{dmath}}
\newcommand{\edmath}{\end{dmath}}
\newcommand{\beq}{\begin{equation}}
\newcommand{\eeq}{\end{equation}}
\newcommand{\bdm}{\begin{displaymath}}
\newcommand{\edm}{\end{displaymath}}
\newcommand{\bea}{\begin{eqnarray}}
\newcommand{\eea}{\end{eqnarray}}
\newcommand{\beal}{\beq \begin{array}{ll}}
\newcommand{\eeal}{\end{array} \eeq}
\newcommand{\beas}{\begin{eqnarray*}}
\newcommand{\eeas}{\end{eqnarray*}}
\newcommand{\ba}{\begin{array}}
\newcommand{\ea}{\end{array}}
\newcommand{\bit}{\begin{itemize}}
\newcommand{\eit}{\end{itemize}}
\newcommand{\ben}{\begin{enumerate}}
\newcommand{\een}{\end{enumerate}}

\newcommand{\calE}{{\cal E}}

\newcommand{\calM}{{\cal M}}

\newcommand{\calP}{{\cal P}}

\newcommand{\calS}{{\cal S}}

\newcommand{\bfI}{\mathbf{I}}

\newcommand{\bfM}{\mathbf{M}}

\newcommand{\bfR}{\mathbf{R}}

\newcommand{\bfc}{\mathbf{c}}
\newcommand{\bfd}{\mathbf{d}}
\newcommand{\bfe}{\mathbf{e}}

\newcommand{\bfg}{\mathbf{g}}

\newcommand{\bfl}{\mathbf{l}}

\newcommand{\bfp}{\mathbf{p}}

\newcommand{\bft}{\mathbf{t}}
\newcommand{\bfu}{\mathbf{u}}

\newcommand{\bfx}{\mathbf{x}}
\newcommand{\bfy}{\mathbf{y}}

\newcommand{\sfk}{\sf k}

\newcommand{\sfn}{\sf n}

\newcommand{\etal}{\emph{et~al.}\xspace}

\newcommand{\hide}[1]{}

\newcommand{\hiddenText}{{\color{gray} hidden text.}}
\newcommand{\hideWithText}[1]{\hiddenText}

\newcommand{\zero}{{\mathbf 0}}

\newcommand{\Real}[1]{ { {\mathbb R}^{#1} } }
\newcommand{\reals}{\Real{}}

\newcommand{\SEthree}{\ensuremath{\mathrm{SE}(3)}\xspace}

\newcommand{\SOthree}{\ensuremath{\mathrm{SO}(3)}\xspace}

\newcommand{\SO}[1]{\ensuremath{\mathrm{SO}(#1)}\xspace}

\newcommand{\blue}[1]{{\color{blue}#1}}

\definecolor{myred}{rgb}{0.8500, 0.3250, 0.0980}
\definecolor{myblue}{rgb}{0, 0.4470, 0.7410}

\newcommand{\linkToPdf}[1]{\href{#1}{\blue{(pdf)}}}
\newcommand{\linkToPpt}[1]{\href{#1}{\blue{(ppt)}}}
\newcommand{\linkToCode}[1]{\href{#1}{\blue{(code)}}}
\newcommand{\linkToWeb}[1]{\href{#1}{\blue{(web)}}}
\newcommand{\linkToVideo}[1]{\href{#1}{\blue{(video)}}}
\newcommand{\award}[1]{\xspace} 

\newcommand{\half}{\frac{1}{2}}

\newcommand{\Fij}{F_{ij}\big(\bfc_i,\bfl_j\big)}

\newcommand{\Pijk}{P_{ij}\big(\bfc_i|\bfxk\big)}
\newcommand{\Qijk}{Q_{ij}\big(\bfl_j|\bfxk\big)}
\newcommand{\Ealphak}{E^\alpha\big(\bfx^\alpha|\bfxk\big)}

\newcommand{\lEalphak}{E^\alpha\big(\bfx^\alpha|\bflxk\big)}

\newcommand{\calEalpha}{\calE_\alpha}
\newcommand{\rayij}{\bfp_{ij}}
\newcommand{\pntij}{\bfu_{ij}}

\newcommand{\bfxa}{\bfx^\alpha}
\newcommand{\bfxak}{\bfx^{\alpha(\sfk)}}
\newcommand{\bfxakp}{\bfx^{\alpha(\sfk+1)}}
\newcommand{\bfxakm}{\bfx^{\alpha(\sfk-1)}}

\newcommand{\bfxkm}{\iterate{\bfx}{}{k-1}}
\newcommand{\bfxkp}{\iterate{\bfx}{}{k+1}}

\newcommand{\bfxbk}{\bfx^{\beta(\sfk)}}

\newcommand{\bflxk}{\iterate{\overline{\bfx}}{}{k}}
\newcommand{\bflxak}{\overline{\bfx}^{\alpha(\sfk)}}
\newcommand{\bflxbk}{\overline{\bfx}^{\beta(\sfk)}}

\newcommand{\iterate}[3]{{#1}_{#2}^{\sf (#3)}}
\newcommand{\bfRik}{\iterate{\bfR}{i}{k}}
\newcommand{\bftik}{\iterate{\bft}{i}{k}}
\newcommand{\bfcik}{\iterate{\bfc}{i}{k}}
\newcommand{\bfdik}{\iterate{\bfd}{i}{k}}
\newcommand{\bfljk}{\iterate{\bfl}{j}{k}}
\newcommand{\bfpijk}{\iterate{\bfp}{ij}{k}}
\newcommand{\bfgijk}{\iterate{\bfg}{ij}{k}}
\newcommand{\bfeijk}{\iterate{\bfe}{ij}{k}}
\newcommand{\bfxk}{\iterate{\bfx}{}{k}}

\newcommand{\bfRikp}{\iterate{\bfR}{i}{k+1}}
\newcommand{\bftikp}{\iterate{\bft}{i}{k+1}}
\newcommand{\bfcikp}{\iterate{\bfc}{i}{k+1}}

\newcommand{\bfljkp}{\iterate{\bfl}{j}{k+1}}
\newcommand{\bfpijkp}{\iterate{\bfp}{ij}{k+1}}

\newcommand{\bfRikm}{\iterate{\bfR}{i}{k-1}}
\newcommand{\bftikm}{\iterate{\bft}{i}{k-1}}

\newcommand{\bfdikm}{\iterate{\bfd}{i}{k-1}}
\newcommand{\bfljkm}{\iterate{\bfl}{j}{k-1}}

\newcommand{\bfldik}{\iterate{\overline{\bfd}}{i}{k}}
\newcommand{\bflRik}{\iterate{\overline{\bfR}}{i}{k}}
\newcommand{\bfltik}{\iterate{\overline{\bft}}{i}{k}}
\newcommand{\bflljk}{\iterate{\overline{\bfl}}{j}{k}}

\newcommand{\lambdaijk}{\iterate{\lambda}{ij}{k}}
\newcommand{\wijk}{\iterate{w}{ij}{k}}
\newcommand{\aijk}{\iterate{a}{ij}{k}}

\newcommand{\gammaak}{\gamma^{\alpha(\sfk)}}
\newcommand{\sak}{s^{\alpha(\sfk)}}
\newcommand{\sakm}{s^{\alpha(\sfk-1)}}
\newcommand{\sakp}{s^{\alpha(\sfk+1)}}

\newcommand{\Ealpha}{E^{\alpha}}

\newcommand{\Fak}{F^{\alpha(\sfk)}}
\newcommand{\Eak}{E^{\alpha(\sfk)}}
\newcommand{\lFak}{\overline{F}^{\alpha(\sfk)}}
\newcommand{\lFk}{\overline{F}^{(\sfk)}}
\newcommand{\Finf}{F^\infty}

\newcommand{\lFakm}{\overline{F}^{\alpha(\sfk-1)}}
\newcommand{\lFkm}{\overline{F}^{(\sfk-1)}}

\newcommand{\Fakp}{F^{\alpha(\sfk+1)}}
\newcommand{\Eakp}{E^{\alpha(\sfk+1)}}
\newcommand{\lFakp}{\overline{F}^{\alpha(\sfk+1)}}
\newcommand{\lFkp}{\overline{F}^{(\sfk+1)}}

\def\grad{\mathrm{grad}\,}
\def\SVDOPlus{{\sf ProjRot3D}}
\def\ProjGrad{{\sf ProjGrad}}

\crefname{section}{Sec.}{Secs.}
\crefname{assumption}{Assumption}{Assumptions}
\crefname{proposition}{Proposition}{Propositions}

\IEEEoverridecommandlockouts

\begin{document}
\title{
\vspace{-0.75em}
Decentralization and Acceleration Enables Large-Scale Bundle Adjustment
}

 \author{
 Taosha Fan$^{1}$,
 Joseph Ortiz$^{2}$,
 Ming Hsiao$^{3}$,
 Maurizio Monge$^{3}$,
 Jing Dong$^{3}$,
 Todd Murphey$^{4}$,
 Mustafa Mukadam$^{1}$\\[5pt]
 $^{1}$Meta AI,
 $^{2}$Imperial College London,
 $^{3}$Reality Labs Research,
 $^{4}$Northwestern University\\
 \vspace{-7mm}
 }

\twocolumn[{
	\renewcommand\twocolumn[1][]{#1}
	\maketitle
	\thispagestyle{empty}
	\begin{center}
		\vspace{-2em}
		\centering
        \includegraphics[width=\linewidth,trim={1.8cm 0 1.8cm 0}]{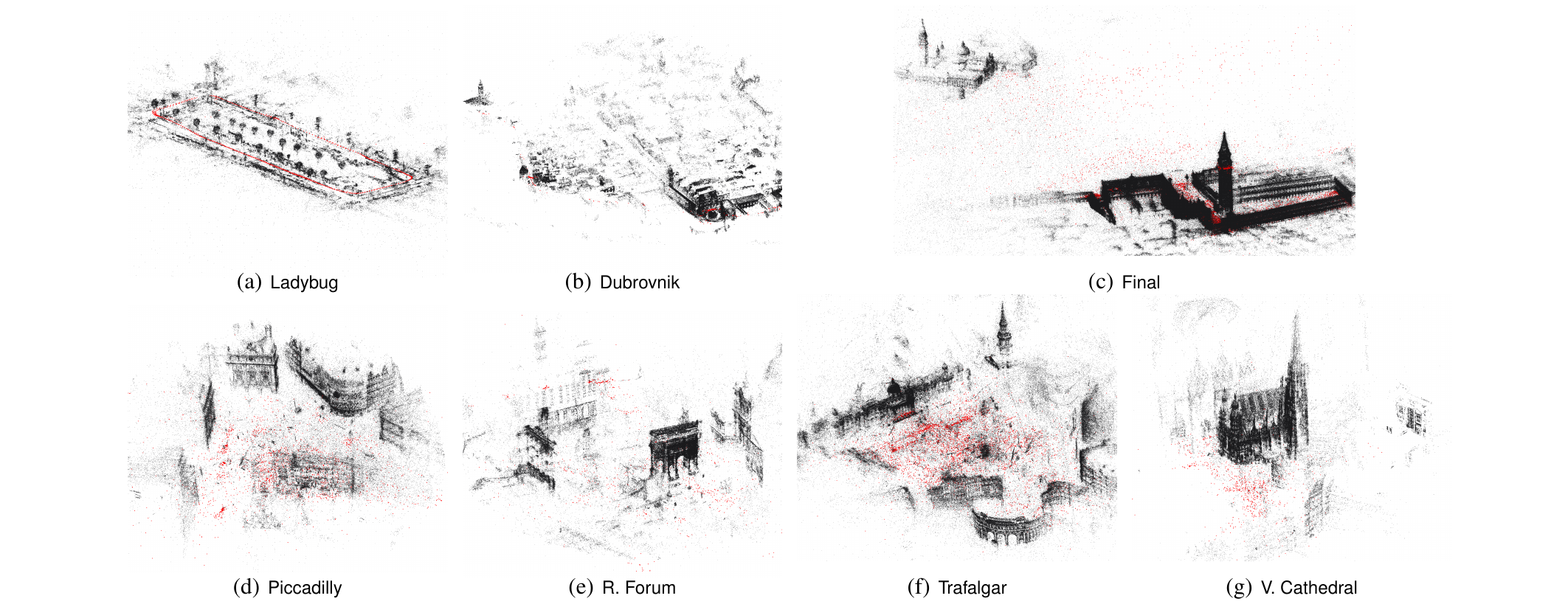}
		\captionof{figure}{3D reconstruction (black point cloud and red camera frames) of our decentralized bundle adjustment method $\amm$, with 8 devices and the Huber loss, on BAL \cite{agarwal2010bundle} (top row) and 1DSfM \cite{wilson2014robust} (bottom row) datasets.}
        \vspace{-0.25em}
		\label{fig::results}
	\end{center}
}]

\begin{abstract}
Scaling to arbitrarily large bundle adjustment problems requires data and compute to be distributed across multiple devices. Centralized methods in prior works are only able to solve small or medium size problems due to overhead in computation and communication. In this paper, we present a fully decentralized method that alleviates computation and communication bottlenecks to solve arbitrarily large bundle adjustment problems. We achieve this by reformulating the reprojection error and deriving a novel surrogate function that decouples optimization variables from different devices. This function makes it possible to use majorization minimization techniques and reduces bundle adjustment to independent optimization subproblems that can be solved in parallel. We further apply Nesterov's acceleration and adaptive restart to improve convergence while maintaining its theoretical guarantees. Despite limited peer-to-peer communication, our method has provable convergence to first-order critical points under mild conditions. On extensive benchmarks with public datasets, our method converges much faster than decentralized baselines with similar memory usage and communication load. Compared to centralized baselines using a single device, our method, while being decentralized, yields more accurate solutions with significant speedups of up to 953.7x over $\ceres$ and 174.6x over $\deeplm$. Code: \url{https://joeaortiz.github.io/daba}.
\end{abstract}

\IEEEpeerreviewmaketitle

\section{Introduction and Related Work}

\vspace{-0.5mm}
With the boom of photos and videos in recent decades, bundle adjustment has become one of the most fundamental and useful techniques in robotics \cite{cadena2016past,rosen4advances,thrun2005probabilistic}, computer vision~\cite{Agarwal:etal:ICCV2009,Jeong:etal:CVPR2010}, autonomous driving \cite{geiger2012we,song2015high}, AR/VR \cite{liu2016robust} and other areas. Bundle adjustment is the nonlinear optimization problem of estimating camera parameters and point positions from a collection of images \cite{triggs1999bundle}.
In the last decade, as image datasets are getting increasingly larger while the computing power of a single device reaches saturation, large-scale bundle adjustment on multiple devices has become more critical than ever. 

Even though efficient  solvers such as $\ceres$~\cite{ceres-solver}, {$\sf\small g2o$}~\cite{kummerle2011g}, {$\sf\small GTSAM$}~\cite{dellaert2012gtsam}, {\sf\small Theseus}~\cite{pineda2022theseus}, {\sf\small SymForce}~\cite{martiros2022symforce}, etc. have successfully solved small- to medium-scale bundle adjustment by exploiting the problem structure, they all operate on a single central device with a global view of the problem. However, these single-device centralized methods \cite{ceres-solver,kummerle2011g,dellaert2012gtsam,pineda2022theseus,martiros2022symforce} are unable to leverage parallelism and fail to scale to large problems due to the time and memory limitations.


Even though numerous multi-device methods have been proposed for bundle adjustment to exploit parallel computing, most of them \cite{ni2007out,zhou2020stochastic,huang2021deeplm,WuCVPR2011multicore,ren2022megba,tian2022distributed}  have to use a central device to maintain consistency. Ni \etal \cite{ni2007out} propose an out-of-core solution that alternates between solving for independent clusters in parallel and overlapping regions. {\sf\small PBA} \cite{WuCVPR2011multicore} implements multiple devices to compute preconditioned conjugate gradient steps. 
{\sf\small STBA} \cite{zhou2020stochastic} stochastically decomposes bundle adjustment using constraint relaxation to approximate Gauss-Newton directions. $\deeplm$ \cite{huang2021deeplm} is an efficient GPU-based Levenberg-Marquardt solver with a novel backward jacobian network. {\sf\small MegBA} \cite{ren2022megba} uses fast distributed preconditioned conjugate gradient method and Schur elimination on multiple GPUs. With a central device to collect information, these multi-device centralized methods, albeit more scalable than single-device ones \cite{ceres-solver,kummerle2011g,dellaert2012gtsam,pineda2022theseus,martiros2022symforce}, are still unsuitable for arbitrarily large bundle adjustment due to the communication bottlenecks of centralization.

Different from single/multi-device centralized methods requiring a central device \cite{ceres-solver,kummerle2011g,dellaert2012gtsam,pineda2022theseus,martiros2022symforce,ni2007out,zhou2020stochastic,huang2021deeplm,WuCVPR2011multicore,ren2022megba}, decentralized methods for bundle adjustment maintain consistency among devices via purely peer-to-peer communication. Decentralized methods can therefore scale to large problems by avoiding communication bottlenecks associated with a central device. Ortiz \etal \cite{ortiz2020bundle} use Gaussian Belief Propagation to solve bundle adjustment problems on a graph processor. Another popular family of methods are based on Douglas-Rachford ($\dr$)  and Alternating Direction Method of Multipliers ($\admm$) \cite{eriksson2016consensus, zhang2017dist, demmel2020distributed}.
Eriksson \etal \cite{eriksson2016consensus} uses point consensus among subproblems with Douglas-Rachford proximal splitting of the cameras.
To reduce the communication overhead, camera consensus and point splitting can be more efficient \cite{zhang2017dist}.
Demmel \etal \cite{demmel2020distributed} use a similar consensus method of parallel block coordinate descent for photometric bundle adjustment.
Although decentralized methods \cite{ortiz2020bundle,eriksson2016consensus, zhang2017dist, demmel2020distributed} are more scalable, they yield less accurate solutions than centralized methods while being slower and requiring careful parameter tuning. Moreover, they either lack provable convergence or make strict assumptions for it. 

We present  Decentralized and Accelerated Bundle Adjustment ($\amm$) to address the compute and communication bottleneck for bundle adjustment of arbitrary scale. Unlike prior work \cite{ortiz2020bundle,eriksson2016consensus, zhang2017dist, demmel2020distributed}, $\amm$ yields more accurate solutions than centralized methods with greater efficiency and less sensitivity to parameter tuning. $\amm$ also provides convergence guarantees to first-order critical points under less strict assumptions. 
By reformulating the reprojection error and deriving a novel  surrogate function, we decouple optimization variables from different devices to reduce bundle adjustment to independent subproblems on a single device. This is in contrast to \cite{eriksson2016consensus, zhang2017dist, demmel2020distributed} that makes local copies of optimization variables to formulate subproblems. We also implement Nesterov's acceleration \cite{nesterov1983method,nesterov2013introductory} and adaptive restart \cite{o2015adaptive} to improve convergence without loss of theoretical guarantees. On extensive benchmarks with public datasets, $\amm$ converges much faster than decentralized baselines with similar memory usage and communication load. Compared to centralized baselines using a single device, $\amm$, while being decentralized, yields more accurate solutions with significant speedups of up to 953.7x over $\ceres$ and 174.6x over $\deeplm$. 

\section{Background}\label{section::problem}

Bundle adjustment is the problem of jointly estimating camera extrinsics/intrinsics and point positions that represent the scene geometry, given a set of images showing several points from different views. Large-scale problems necessitate data to be distributed across multiple devices.
We focus on decentralized bundle adjustment that can work solely via peer-to-peer communication without the need for a central device~\cite{ortiz2020bundle,eriksson2016consensus, zhang2017dist, demmel2020distributed}. Decentralized methods are preferred if communication latency and bandwidth are much more expensive than computation, for instance in large-scale bundle adjustment.

Decentralized bundle adjustment can be formulated as an optimization problem minimizing reprojection errors, i.e., residuals between predicted and observed light reprojections~\cite{triggs1999bundle}. Given $M$ cameras and $N$ points partitioned onto $S$ devices $\calS\triangleq\{1,\,2,\,\cdots,\,S\}$ and the set $\calE$ of reprojection pairs for cameras and points, the optimization problem finds variables $\bfx$ which are camera extrinsics/intrinsics $\{\bfc_i\}_{i=1}^M$ and point positions $\{\bfl_j\}_{j=1}^N$, and 
constructs the objective function from individual penalties $\Fij$ on reprojection errors:
\begingroup
\setlength\abovedisplayskip{5pt}
\begin{equation}\label{eq::Fij}
	\Fij\triangleq \half\rho(\|\bfe_{ij}\|^2)
 \vspace{-0.25em}
\end{equation}
\endgroup
where
$\bfe_{ij}$ is any reprojection error defined on the reprojection pair $(i,\,j)\in \calE$ for camera $i$ and point $j$, and
$\rho(\cdot):\reals^+\rightarrow\reals$ is a  robust loss function for outlier rejection. 
We assume that $\rho(\cdot)$ is differentiable, concave and nondecreasing. This applies to a broad class of robust loss functions like Huber and Welsch \cite{fan2021mm_full}.
Then, $\bfx\triangleq\{\bfc_i\}_{i=1}^M\bigcup\{\bfl_j\}_{j=1}^N$ partitioned onto multiple devices can be found by minimizing the sum of all penalties over the set $\calE$ of reprojection pairs.
\begin{problem}[Decentralized Bundle Adjustment]
\label{problem::ba}
\setlength\abovedisplayskip{5pt}
\begin{equation}\label{eq::Fobj}
	\min_{\bfx} F(\bfx)\triangleq\sum_{(i,\,j)\in\calE} \Fij
 \vspace{-0.75em}
\end{equation}
where the optimization variables $\bfx\triangleq\{\bfc_i\}_{i=1}^M\bigcup\{\bfl_j\}_{j=1}^N$ are partitioned onto multiple devices. 
\end{problem}
\vspace{-0.5em}
\section{A Novel Reprojection Error}\label{section::error}
Even though there are various types of reprojection errors, all of them have cameras and points inseparable, and thus, are difficult to use in decentralized bundle adjustment where optimization variables are on multiple devices. In this section, we present a novel reprojection error to decouple camera extrinsics/intrinsics and point positions that in the next section is leveraged to formulate a surrogate function and create independent optimization subproblems solved in parallel.

In multi-view geometry, the undistorted reprojection ray is equal to the point position in the camera frame up to a scale factor \cite{hartley2003multiple}. Therefore, with the undistortion model in \cite{buj2010new,jose2009pose,kuke2013realtime}, we might assume that there exists a scalar $\lambda_{ij}\in\reals$ for each reprojection pair $(i,\,j)\in\calE$ such that
\begingroup
\setlength\abovedisplayskip{5pt}
\begin{equation}\label{eq::reprojection1}
\underbrace{\begin{bmatrix}
	\frac{1}{f_i}\bfu_{ij} \\
	1 + k_{i,1}\|\bfu_{ij}\|^2 + k_{i,2}\|\bfu_{ij}\|^4
\end{bmatrix}}_{\text{undistorted reprojection ray}} - \lambda_{ij}\cdot \underbrace{\bfR_i^\top\left(\bfl_j-\bft_i\right)}_{\substack{\text{point position in the} \\\text{camera frame}}}=\zero
\vspace{-0.1em}
\end{equation}
\endgroup
where $\bfu_{ij}\in\Real{2}$ is the observed distorted reprojection point on the camera plane, $f_i$ is the focal length, $k_{i,1}$, $k_{i,2}$ are the radial undistortion coefficients, $(\bfR_i,\,\bft_i)\in \SEthree$ is the camera extrinsics with $\bfR_i\in\SOthree$ and $\bft_i\in\Real{3}$, and $\bfl_j\in\Real{3}$ is the point position.  If we represent the camera intrinsics as  $d_{i,1}=f_i$, $d_{i,2}=f_i\cdot k_{i,1}$, $d_{i,3}=f_i\cdot k_{i,2}$,  the undistorted reprojection ray on the left-hand side of \cref{eq::reprojection1}  can be reformulated as
\begin{equation}\label{eq::ray}
	\rayij = \begin{bmatrix}
		\pntij \\ d_{i,1} + d_{i,2}\|\pntij\|^2 + d_{i,3}\|\pntij\|^4
	\end{bmatrix}\in\Real{3}.
\end{equation}
With \cref{eq::ray} and a slight abuse of notation for $\lambda_{ij}$, \cref{eq::reprojection1} is
\begingroup
\setlength\abovedisplayskip{4pt}
\begin{equation}\label{eq::err0}
\bfp_{ij} - \lambda_{ij}\cdot\bfR_i^\top(\bfl_j - \bft_i)=\zero.
	\vspace{-0.15em}
\end{equation}
\endgroup
In \cref{eq::err0}, it is possible to find $\lambda_{ij}\in\reals$ by solving
\begingroup
\setlength\abovedisplayskip{4pt}
\begin{equation}\label{eq::error_opt}
	\lambda_{ij} \leftarrow \arg\min_{\lambda_{ij}\in\reals}\|\bfp_{ij} -\lambda_{ij}\cdot\bfR_i^\top(\bfl_j - \bft_i)\|^2.
	\vspace{-0.35em}
\end{equation}
\endgroup
Since $\bfR\in\SOthree$ and $\bfR_i\bfR_i^\top=\bfI$ where $\bfI\in\Real{3\times3}$ is the identity,  we have  $\|\bfR_i^\top(\bfl_j - \bft_i)\|=\|\bfl_j-\bft_i\|$. If we assume $\|\bfR_i^\top(\bfl_j - \bft_i)\|=\|\bfl_j-\bft_i\|\neq 0$ that is common in bundle adjustment \cite{triggs1999bundle}, then \cref{eq::error_opt} has a unique solution at
\begingroup
\setlength\abovedisplayskip{4pt}
\begin{equation}\label{eq::lambdaij}
\lambda_{ij}  =\frac{(\bfl_j-\bft_i)^\top\bfR_i\bfp_{ij}}{\|\bfl_j-\bft_i\|^2}.
\end{equation}  
Substituting \cref{eq::lambdaij} into the left-hand side of \cref{eq::err0} yields the following reprojection error:
\begin{equation}\label{eq::error}
	\bfe_{ij}=\left(\bfI-\frac{\bfR_i^\top(\bfl_j-\bft_i)(\bfl_j-\bft_i)^\top\bfR_i}{\|\bfl_j-\bft_i\|^2}\right)\bfp_{ij}\in\Real{3}.
\end{equation}
\endgroup
where $\|\bfe_{ij}\|=0$ if and only if \cref{eq::err0} holds. Furthermore, $\bfe_{ij}$ geometrically results from projecting $\bfp_{ij}$ onto the normal plane of $\bfR_i^\top(\bfl_j-\bft_i)$; see \cref{fig::reprojection}. 

Recall that the undistorted reprojection ray $\bfp_{ij}$ in  \cref{eq::ray} depends on the camera intrinsics $\bfd_i\triangleq\begin{bmatrix}
d_{i,1} & d_{i,2} & d_{i,3}
\end{bmatrix}^\top\in\Real{3}$. Then, we conclude from \cref{eq::error} that $\bfe_{ij}$ is a function of camera extrinsics/intrinsics $\bfc_i\triangleq(\bfR_i,\,\bft_i,\,\bfd_i)\in \SEthree\times\Real{3}$ and point positions $\bfl_j\in\Real{3}$, with which \cref{problem::ba} is well formulated for optimization variables $\bfx=\{\bfc_i\}_{i=1}^M\bigcup\{\bfl_j\}_{j=1}^N$.

The reprojection error above is used in the next section to decouple optimization variables from different devices such that decentralized  bundle adjustment is reduced to independent optimization subproblems on a single device.

\section{Decouple Variables, Reduce to Subproblems}\label{section::mm}

Since objective function $F(\bfx)$ in \cref{problem::ba} has optimization variables $\bfx$ on multiple devices, it is difficult to minimize with restricted communication. We address this with a surrogate function $E(\bfx|\bfxk)$ that is an upper bound of $F(\bfx)$ but equal to it at the current iterate $\bfxk$:
\begingroup
\begin{equation}\label{eq::FEEF}
\setlength\abovedisplayskip{2.5pt}
F(\bfx)\leq E(\bfx|\bfxk)\,\text{ and }\, E(\bfxk|\bfxk)=F(\bfxk).
\end{equation}
\endgroup
If $E(\bfxkp|\bfxk)\leq E(\bfxk|\bfxk)$, then \cref{eq::FEEF} suggests
\begin{equation}
\nonumber
F(\bfxkp)\leq E(\bfxkp|\bfxk)\leq E(\bfxk|\bfxk) = F(\bfxk)
\end{equation}
where such a $\bfxkp$ can be found by solving
\begingroup
\setlength\abovedisplayskip{5pt}
\begin{equation}\label{eq::min_surrogate}
\bfxkp\leftarrow\min_{\bfx} E(\bfx|\bfxk).
\vspace{-0.25em}
\end{equation}
\endgroup
Therefore, even though $F(\bfx)$ is not optimized, \cref{eq::min_surrogate} still yields $F(\bfxkp)\leq F(\bfxk)$. Furthermore, if the surrogate function $E(\bfx|\bfxk)$ is in the form of
\begingroup
\setlength\abovedisplayskip{4pt}
\begin{equation}\label{eq::Esum}
 E(\bfx|\bfxk)\triangleq \sum_{\alpha\in\calS} \Ealphak
 \vspace{-0.5em}
\end{equation}
where $\Ealphak$ is a function of $\bfxa$, i.e., the camera extrinsics/intrinsics and point positions on a single device $\alpha$, then \cref{eq::min_surrogate} is equivalent to multiple independent optimization subproblems on a single device:
\begin{equation}\label{eq::update_mm}
\bfxakp\leftarrow\arg\min_{\bfxa} \Ealphak \text{ for device $\alpha\in\calS$}
\vspace{-0.5em}
\end{equation}
that  can be solved in parallel without inter-device communication once $\Ealphak$ is constructed. 
\endgroup

\begin{figure}[t]
	\centering
	\vspace{-1.5em}
	\includegraphics[width=0.275\textwidth]{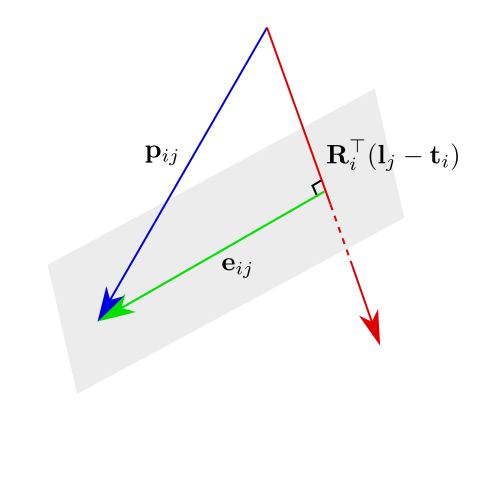}
	\vspace{-3em}
	\caption{Reprojection error $\bfe_{ij}$ in \cref{eq::error} geometrically results from projecting $\bfp_{ij}$ onto the normal plane of $\bfR_i^\top(\bfl_j-\bft_i)$.}\label{fig::reprojection}
	\vspace{-1.5em}
\end{figure}

The technique above of minimizing an upper bound of the objective function is referred as majorization minimization \cite{hunter2004tutorial,sun2016majorization}.  As its name suggests, majorization minimization has two steps: 1) constructing a surrogate function satisfying \cref{eq::FEEF} as ``majorization'' and 2) optimizing the surrogate function in \cref{eq::min_surrogate} as ``minimization''. Compared to belief propagation and consensus techniques, majorization minimization is usually faster for empirical implementation and easier for theoretical analysis as long as a proper surrogate function has been found. However, such a proper surrogate function is usually nontrivial. This is even more challenging for decentralized bundle adjustment where the surrogate function is required to satisfy not only \cref{eq::FEEF} but also \cref{eq::Esum}. This means that  $E(\bfx|\bfxk)$ must majorize objective function $F(\bfx)$ while decoupling optimization variables $\{\bfxa\}_{\alpha\in\calS}$ that are camera extrinsics/intrinsics and point positions on different devices. Fortunately, such a surrogate function $E(\bfx|\bfxk)$ exists and its derivation is the main contribution of this section.

Recall from \cref{eq::Fobj} that the objective function $F(\bfx)$ is the sum of $\Fij$. If we can majorize each individual $\Fij$ while decoupling $\bfc_i$ and $\bfl_j$, a surrogate function $E(\bfx|\bfxk)$ satisfying \cref{eq::FEEF,eq::Esum} is yielded. This is in fact possible with the reprojection error in \cref{eq::error} as long as $\|\bfl_j-\bft_i\|\neq0$, which results in the following proposition.

\begin{proposition}\label{proposition::majorize}
Suppose $\iterate{(\cdot)}{}{k}$ is the $\sfk$-th iterate of $(\cdot)$. Let
\vspace{-0.25em}
\begin{equation}\label{eq::P}
\!\!\!\Pijk \triangleq  \wijk\cdot  \left\|\bfR_i\bfp_{ij} + \lambdaijk\cdot\bft_i - \iterate{\bfg}{ij}{k}\right\|^2 + \half\aijk ,\!
\end{equation}
\begin{equation}\label{eq::Q}
\Qijk \triangleq \wijk\cdot \left\|\lambdaijk\cdot\bfl_j - \bfgijk\right\|^2 + \half\aijk
\end{equation}
where
\begingroup
\setlength\abovedisplayskip{7pt}
\begin{equation}\label{eq::a}
	\aijk \triangleq \half\rho(\|\bfeijk\|^2) - \half\nabla\rho(\|\bfeijk\|^2)\cdot \|\bfeijk\|^2,
\end{equation}
\begin{equation}\label{eq::w}
\wijk \triangleq \nabla\rho(\|\bfeijk\|^2),
\end{equation}
\begin{equation}\label{eq::gamma}
\lambdaijk \triangleq \frac{\big(\bfljk - \bftik\big)^\top\bfRik \bfpijk}{\big\|\bfljk - \bftik\big\|^2},
\end{equation}
\begin{equation}\label{eq::g}
\bfgijk \triangleq \half\bfRik\bfpijk + \half\lambdaijk\cdot\bftik + \half\lambdaijk\cdot\bfljk.
\end{equation}
\endgroup
Then,  
\begingroup
\setlength\abovedisplayskip{2pt}
\begin{equation}\label{eq::majorize}
	\Fij\leq \Pijk + \Qijk
 \vspace{-.25em}
\end{equation}
\endgroup
and the equality ``$=$'' holds if $\bfc_i=\bfcik$ and $\bfl_j=\bfljk$.
\end{proposition}
\begin{proof}
Please refer to \select{\cite[App. C.1]{fan2023daba}}{App. \hyperref[proof::majorize]{C.1}}.
\end{proof}

With \cref{proposition::majorize}, we can majorize the objective function $F(\bfx)$ in \cref{eq::Fobj}. Suppose  $\calE'$ and $\calE''$ are the sets of reprojection pairs $(i,j)\in\calE$ with cameras and points from the same/different devices. With $\calE'$ and $\calE''$, we might split these penalty functions $\Fij$ in \cref{eq::Fobj} into two parts:
\begingroup
\setlength\abovedisplayskip{5pt}
\begin{equation}
\nonumber
F(\bfx) = \sum_{(i,\,j)\in\calE'} \Fij + \sum_{(i,\,j)\in\calE''} \Fij.
\vspace{-0.5em}
\end{equation}
\endgroup
If we apply \cref{eq::majorize} on $\Fij$ where $(i,j)\in\calE''$ are inter-device reprojection pairs, the equation above results in
\begingroup
\setlength\abovedisplayskip{2pt}
\begin{multline}\label{eq::upper_bnd}
F(\bfx) \leq  
\sum_{(i,\,j)\in\calE'}\Fij +\\ \sum_{(i,\,j)\in\calE''} \Big(\Pijk + \Qijk\Big)
\vspace{-0.5em}
\end{multline}
\endgroup
where the right-hand side has  camera extrinsics/intrinsics $\bfc_i$ and point positions $\bfl_j$ from different devices decoupled as $\Pijk$ and $\Qijk$. On the right-hand side of \cref{eq::upper_bnd}, collecting  $\Fij$, $\Pijk$, $\Qijk$ where $\bfc_i$ and $\bfl_j$ are on the same device $\alpha$  yields a function $\Ealphak$ where $\bfxa$ is the camera extrinsics/intrinsics and point positions on a single device $\alpha$. Then, summing $\Ealphak$ over all the devices $\alpha\in\calS$ and applying the right-hand side of \cref{eq::upper_bnd}  leads to a surrogate function  $E(\bfx|\bfxk)$ in the form of:
\begingroup
\setlength\abovedisplayskip{2pt}
\begin{multline}\label{eq::Ealpha}
\!\!\!\!\!\!E(\bfx|\bfxk)\triangleq\sum_{\alpha\in\calS} \Ealphak =\half\xi\sum_{\alpha\in\calS}\big\|\bfxa-\bfxak\big\|^2  +\\
\hspace{-0.75em}\sum_{(i,\,j)\in\calE'}\!\!\!\! \Fij   \!+\! \sum_{(i,\,j)\in\calE''} \!\!\Big(\!\Pijk \!+\! \Qijk \Big)\!\!\! 
\end{multline}
\endgroup 
where $\xi>0$ related with $\|\bfxa-\bfxak\|^2$ is any positive number close to zero and introduced for the purpose of  convergence analysis. Furthermore, \cref{eq::upper_bnd,eq::Ealpha} suggest the following proposition about the resulting $E(\bfx|\bfxk)$.

\begin{proposition}\label{proposition::surrogate}
$E(\bfx|\bfxk)= \sum_{\alpha\in\calS} \Ealphak$ and $F(\bfx)\leq E(\bfx|\bfxk)$ where the equality ``$=$'' holds if $\bfx=\bfxk$, \ie, $\bfc_i=\bfcik$ and  $\bfl_j=\bfljk$ for all the cameras and points.
\end{proposition}
\begin{proof}
Please refer to \select{\cite[App C.2]{fan2023daba}}{App. \hyperref[proof::surrogate]{C.2}}.
\end{proof}

\cref{proposition::surrogate} indicates that $E(\bfx|\bfxk)$ in \cref{eq::Ealpha} satisfies \cref{eq::FEEF,eq::Esum}. Then, with majorization minimization, \cref{problem::ba} is reduced to independent optimization subproblems on a single device in \cref{eq::update_mm} that result in $F(\bfxkp)\leq F(\bfxk)$. Moreover, the update rule of \cref{eq::update_mm} only has  peer-to-peer communication between neighboring devices sharing inter-device reprojection pairs when constructing $\Ealphak$ with $\Pijk$, $\Qijk$; see \cref{eq::P,eq::Q,eq::Ealpha}. Besides, camera extrinsics/intrinsics $\bfc_i$ and point positions $\bfl_j$ are exchanged with other devices if and only if having inter-device reprojection pairs. Thus,  \cref{eq::update_mm} using majorization minimization enables decentralized bundle adjustment.

\setlength{\textfloatsep}{7pt}

\section{Speedup and Guarantee Convergence}\label{section::acceleration}

The majorization minimization technique in the previous section always decreases the objective function but has slow convergence compared to centralized methods due to its underlying first-order method.
This is in fact a well-known issue for first-order methods, for which Nesterov's acceleration has been applied to get significant speedup \cite{nesterov1983method,nesterov2013introductory}. Nesterov's accelerated methods  have shown convergence guarantees only for convex optimization, but if adaptive restart \cite{o2015adaptive} is used the guarantees can be extended to nonconvex optimization \cite{li2015accelerated,fan2019proximal,fan2021mm_full,fan2020mm}. Thus, we are inspired to implement Nesterov's acceleration for empirical speedup and adaptive restart for theoretical guarantees. This strategy seems straightforward but the execution is not trivial since the adaptive restart requires a central device to evaluate the objective function, which is impossible for decentralized methods. In this section, we present Nesterov's acceleration and a novel adaptive restart scheme for decentralized bundle adjustment to speedup and guarantee the convergence.

Nesterov's acceleration  \cite{nesterov1983method} extrapolates the iterate $\bfxak$ with the momentum $\bfxak-\bfxakm$, which yields the extrapolated intermediate $\bflxak$:
\begingroup
\setlength\abovedisplayskip{5pt}
\begin{equation}
\label{eq::nesterov_x0}
\bflxak ={\sf Proj}\left(\bfxak + \gammaak\big(\bfxak-\bfxakm\big)\right)
\vspace{-0.25em}
\end{equation}
\endgroup
where ${\sf Proj}(\cdot)$ is an operator projecting optimization variables to corresponding manifolds, and $\gammaak\in\reals$ is the extrapolation ratio of the momentum $\bfxak-\bfxakm$ and updated with
\begingroup
\setlength\abovedisplayskip{5pt}
\begin{equation}\label{eq::nesterov_scalar}
	\sak = \begin{cases}
		1, \!&\!\! \sfk=0,\\
		\frac{\sqrt{4{\sakm}^2+1}+1}{2},\! &\!\! \sfk>0
	\end{cases}\;\;\text{and}\;\;\gammaak = \frac{\sak-1}{\sakp}.
\vspace{-0.15em}
\end{equation}
\endgroup
In terms of bundle adjustment, $\bflxak$ has camera extrinsics/intrinsics and point positions extrapolated with
\begin{subequations}\label{eq::nesterov_x}
\setlength\abovedisplayskip{4pt}
\begin{equation}\label{eq::nesterov_R}
	\bflRik = \SVDOPlus\left(\bfRik + \gammaak\big(\bfRik-\bfRikm\big)\right),
 \vspace{-0.1em}
\end{equation}
\begin{equation}\label{eq::nesterov_t}
	\bfltik = \bftik + \gammaak\big(\bftik-\bftikm\big),
 \vspace{-0.025em}
\end{equation}
\begin{equation}\label{eq::nesterov_d}
	\bfldik = \bfdik + \gammaak\big(\bfdik-\bfdikm\big),
 \vspace{-0.025em}
\end{equation}
\begin{equation}\label{eq::nesterov_l}
	\bflljk = \bfljk + \gammaak\big(\bfljk-\bfljkm\big)
 \vspace{-0.025em}
\end{equation}
\end{subequations}
where $\SVDOPlus(\cdot):\Real{3\times 3}\rightarrow \SO3$ is an operator projecting $3\times 3$ matrices to $\SO3$:
\begingroup
\setlength\abovedisplayskip{5pt}
    \begin{equation}\label{eq::proj_rot3d}
	\SVDOPlus(\bfM) \triangleq \arg\min_{\bfR\in\SO3}\|\bfR-\bfM\|^2
	\vspace{-0.35em}
\end{equation}
\endgroup
and has a closed-form solution \cite{umeyama1991least,levinson2020analysis}.  
With  $\bflRik$, $\bfltik$, $\bfldik$, $\bflljk$, we obtain $P_{ij}(\bfc_i|\bflxk)$ and $Q_{ij}(\bfl_j|\bflxk)$ in \cref{eq::P,eq::Q}. This further yields $\lEalphak$ in \cref{eq::Ealpha} as well as the accelerated update rule:
\begin{equation}\label{eq::update_amm}
	\bfxakp\leftarrow\arg\min_{\bfxa} \lEalphak \text{ for device $\alpha\in\calS$}.
\end{equation}
Note that $\Ealphak$ and $\lEalphak$ in \cref{eq::update_amm,eq::update_mm} only differ on whether $\bfxk$ or $\bflxk$ is conditioned on. Moreover, \cref{eq::nesterov_x,eq::nesterov_x0} indicate that the extrapolated intermediate terms $\overline{(\cdot)}^{(\sfk)}$ requires no inter-device communication, and thus, Nesterov's acceleration remains decentralized.

Nesterov's acceleration has no provable convergence due to the nonconvexity of bundle adjustment. Fortunately, such an issue can be resolved with adaptive restart \cite{li2015accelerated,fan2021mm_full,fan2019proximal,fan2020mm}. If the objective value $F(\bfxk)$ is not improved by $\bfxak$ from \cref{eq::update_amm}, adaptive restart updates $\bfxak$ again with \cref{eq::update_mm}. Recall that \cref{eq::update_mm} always decreases $F(\bfxk)$, with which the convergence is guaranteed.
However, it is impossible to evaluate $F(\bfxk)$ without a central device. Thus, other metrics than $F(\bfxk)$ are needed to trigger adaptive restart for decentralized bundle adjustment. Inspired by \cite{fan2021mm_full}, we present a novel adaptive restart scheme by introducing several local per device metrics. This adaptive restart scheme is decentralized but still guarantees the convergence. Due to space limitation, only the main procedure of the adaptive restart scheme is presented while the full analysis is in 
Apps. \hyperref[proof::amm]{C.3} and \hyperref[section::app::lemma::adaptive]{D.1}.   

For notational simplicity, let  $\Delta\Ealpha\big(\bfx|\bfxk\big)$ be defined by:
\begingroup
\setlength\abovedisplayskip{2pt}
\begin{multline}\label{eq::DEalpha}
\Delta \Ealpha\big(\bfx|\bfxk\big) \triangleq  -\half\xi\|\bfxa-\bfxak\|^2+\\ 
\half\!\sum_{(i,\,j)\in\calE''_\alpha}\!\!\!\Big(\!\Fij\!-\!\Pijk \!-\! \Qijk\!\Big)\!\!\!
\vspace{-0.35em}
\end{multline}
\endgroup
where $\calEalpha''$ is the set of inter-device reprojection pairs that has either the camera or the point from device $\alpha$ but not both. Note that $\Fij-\Pijk -\Qijk$ in \cref{eq::DEalpha} is the surrogate gap for majorization minimization.
With  $\Delta\Ealpha\big(\bfx|\bfxk\big)$, we introduce adaptive restart metrics $\Fak$, $\lFak$, $\Eakp$  recursively updated on each device $\alpha$:
\begin{enumerate}[leftmargin=0.475cm]
\item At initialization $\sfk=-1$, we set $\bfx^{\alpha\sf(-1)} = \bfx^{\alpha\sf(0)}$ and 
\begin{equation}\label{eq::Fainit}
	\begin{aligned}
	&\quad\; F^{\alpha(-1)} = \Ealpha\big(\bfx^{\alpha\sf(-1)}| \bfx^{\sf(-1)}\big),\\
	&\overline{F}^{\alpha(-1)} = F^{\alpha(-1)},\;\; E^{\alpha(0)}\!=\!F^{\alpha(-1)}.
	\end{aligned}
\end{equation}
\item For $\sfk\geq 0$, $\Fak$, $\lFak$, $\Eakp$ are updated with
\begin{equation}\label{eq::Fak}
\Fak = \Eak + \Delta \Ealpha\big(\bfxk|\bfxkm\big),
\end{equation}
\begin{equation}\label{eq::lFak}
\lFak = (1-\eta)\cdot\lFakm + \eta\cdot\Fak,
\end{equation}
\begin{equation}\label{eq::Eak}
\!\!\!\!\Eakp \!=\! \Ealpha\!\big(\bfxakp|\bfxk\big) + \Fak\! - \Ealpha\!\big(\bfxak|\bfxk\big).
\end{equation}
Note that $\eta\in(0,\, 1]$ in \cref{eq::lFak}.
\end{enumerate}
With  local metrics $\Fak$, $\lFak$, $\Eakp$ in \cref{eq::Fak,eq::lFak,eq::Eak,eq::Fainit},  the adaptive restart scheme on each device $\alpha$ independently executes the following procedure:
\begin{enumerate}[leftmargin=0.475cm]
\item Solve \cref{eq::update_amm} to update $\bfxakp$;
\item If $\Eakp > \lFak$, update $\bfxakp$ again with \cref{eq::update_mm}.
\end{enumerate}
In spite of no central device to evaluate the objective value $F(\bfxk)$, such a local adaptive restart scheme actually suffices for provable convergence to first-order critical points; see
\select{\cite[App. C.3]{fan2023daba}}{App. \hyperref[proof::amm]{C.3} for complete analysis}.
Moreover, \cref{eq::Fainit,eq::Eak,eq::Fak,eq::lFak,eq::DEalpha} suggest that  $\Fak$, $\lFak$, $\Eakp$ can be updated if each device $\alpha$ can communicate with its neighbors. Thus, the adaptive restart scheme is well-suited  for decentralized bundle adjustment.

By applying Nesterov's acceleration and adaptive restart while maintaining decentralization, we have achieved empirical speedup without losing theoretical guarantees. The resulting improvements are evaluated in \cref{section::experiment::efficiency}. 

\section{$\amm$: Putting it all Together}

We have presented a novel reprojection error in \cref{section::error}, decoupled optimization variables to reduce decentralized bundle adjustment to independent optimization subproblems in \cref{section::mm}, and improved the convergence with Nesterov's acceleration and adaptive restart in \cref{section::acceleration}. All of these result in  $\amm$, i.e., Decentralized and Accelerated Bundle Adjustment; see \cref{algorithm::amm}. As previously discussed, $\amm$ is fully decentralized and requires limited peer-to-peer communication. Furthermore, $\amm$ is guaranteed to converge to first-order critical points under mild conditions as the following proposition states.

\begin{algorithm}[t]
	\caption{The $\amm$ Method}
	\label{algorithm::amm}
	\begin{algorithmic}[1]
		\State\textbf{Input}: An initial iterate $\bfx^{(0)}\in $, $\xi > 0$, $0<\eta\leq 1$.
		\State\textbf{Output}: A sequence of iterates $\{\bfxk\}$.\vspace{0.1em} 
		\For{each device $\alpha\in\calS$}
		\State $\bfx^{\alpha(-1)}\leftarrow \bfx^{\alpha(0)}$ and $s^{\alpha(0)}\leftarrow 1$
		\State initialize $F^{\alpha(-1)}$, $\overline{F}^{\alpha(-1)}$, $E^{\alpha(0)}$  with \cref{eq::Fainit} 
		\EndFor
		\For{$\sfk\gets 0,\,1,\,2,\,\cdots$}
		\For{ each device $\alpha\in\calS$}
        \State{\color{gray}//  Nesterov's acceleration}\vspace{0.15em}
		\State  $\sakp\leftarrow\frac{\sqrt{4{\sak}^2+1}+1}{2}$ and $\gammaak \leftarrow \frac{\sak-1}{\sakp}$
		\State update $\bflxak$ with \cref{eq::nesterov_x} on device $\alpha$\vspace{0.1em}
        \State{\color{gray} // Inter-device communication}
		\State retrieve $\bfxbk,\,\bflxbk$ from neighboring devices $\beta$
        \State{\color{gray} // Majorization}
		\State construct $\Ealphak$, $\lEalphak\!$ with \cref{eq::Ealpha}
        \State{\color{gray} // Minimization}
		\State solve $\bfxakp\leftarrow\arg\min_{\bfxa} \lEalphak$ \label{line::amm::update_amm}
        \State{\color{gray}// Adaptive restart}
        \State update $\Fak$, $\lFak\!$, $\Eakp$ with \cref{eq::Fak,eq::lFak,eq::Eak}
		\If{$\Eakp>\lFak$} \label{line::amm::adaptive_restart::start}
  		\State solve $\bfxakp\leftarrow\arg\min_{\bfxa} \Ealphak$ \label{line::amm::update_mm}
		\State update $\Eakp$ with \cref{eq::Eak}
		\EndIf
		\EndFor
		\EndFor
	\end{algorithmic}
\end{algorithm}

\begin{proposition}\label{proposition::amm}
	If $\xi>0$ and $0<\eta\leq 1$, the  sequence of iterates $\{\bfxk\}$ from $\amm$ (\cref{algorithm::amm}) converges to first-order critical points under \select{Assumptions 1 to 4 in \cite[App. B]{fan2023daba}}{\cref{assumption::loss,assumption::bounded_intr,assumption::local_opt,assumption:nonzero} in App. \hyperref[section::app::assumptions]{B}}.
\end{proposition}

\begin{proof}
Please refer to \select{\cite[App. C.3]{fan2023daba}}{App. \hyperref[proof::amm]{C.3}}.
\end{proof}

\begin{table*}[t]
    \caption{
    \textbf{Mean reprojection errors} with the \textbf{Trivial loss} and \textbf{Huber loss} on the datasets in \cref{table::large_dataset}.
    Decentralized methods $\dr$ \cite{eriksson2016consensus}, $\admm$ \cite{zhang2017dist}, $\amm$ (ours) are run for 1000 iterations with 4, 8, 16, 32 devices.
    Centralized methods $\ceres$ \cite{ceres-solver} and $\deeplm$ \cite{huang2021deeplm} are run for 40 iterations with single device as reference ($\deeplm$ does not support Huber loss).
    On each dataset (row), any decentralized method with best result is \textbf{bold}, and outperforming $\ceres$ and $\deeplm$ is {\color{red} red}.
    \textbf{$\amm$ (ours) achieves lowest reprojection error between decentralized methods and mostly outperforms centralized methods}.
    }
    \label{table::accuracy}
    \centering
    \setlength{\tabcolsep}{0.55em}

    \label{table::comp_trivial}
	\begin{tabular}{ccccccccccccccccc}
        \toprule
		\multicolumn{2}{c}{}& \multicolumn{15}{c}{Mean Reprojection Error with the Trivial Loss}\\
		\cmidrule(lr){3-17}
		\multicolumn{2}{c}{Dataset} &\multirow{2}{*}{Init} & \multirow{2}{*}{$\ceres$} & \multirow{2}{*}{$\deeplm$} &  \multicolumn{3}{c}{4 Devices} & \multicolumn{3}{c}{8 Devices} & \multicolumn{3}{c}{16 Devices} & \multicolumn{3}{c}{32 Devices}\\
		\cmidrule(lr){6-8} \cmidrule(lr){9-11} \cmidrule(lr){12-14} \cmidrule(lr){15-17}
		\multicolumn{2}{c}{} & & & & {\,\scriptsize $\dr$\,} & {\scriptsize $\admm$} & {\scriptsize $\amm$} & {\,\scriptsize $\dr$\,} & {\scriptsize $\admm$} & {\scriptsize $\amm$} & {\,\scriptsize $\dr$\,} & {\scriptsize $\admm$}  & {\scriptsize $\amm$}  & {\,\scriptsize $\dr$\,} & {\scriptsize $\admm$} &  {\scriptsize $\amm$}\\
		\cmidrule(lr){1-2} \cmidrule(lr){3-17}
		{\multirow{3}{*}{\rotatebox[origin=c]{90}{\scriptsize BAL \cite{agarwal2010bundle}}}} 
		& {\scriptsize\sf Ladybug} & {10.48} & {0.707} &  {0.710} &{0.837} & {\color{red}0.698} & \textbf{\color{red}0.690} & {0.846} & {\color{red}0.703} & \textbf{\color{red}0.690} & {0.850} & {0.711} & \textbf{\color{red}0.690} & {0.859} & {0.723} & \textbf{\color{red}0.690} \\
		& {\scriptsize\sf Venice} & {26.33} & {0.468} &  {0.466} &{0.515} & {0.516} & \textbf{\color{red}0.465} & {0.515} & {0.525} & \textbf{\color{red}0.465} & {0.516} & {0.532} & \textbf{\color{red}0.466} & {0.517} & {0.544} & \textbf{0.473} \\
		& {\scriptsize\sf Final} & {12.57} & {0.855} &  {0.848} &{2.017} & {0.876} & \textbf{\color{red}0.828} & {2.042} & {0.903} & \textbf{\color{red}0.828} & {2.053} & {0.908} & \textbf{\color{red}0.829} & {2.123} & {0.921} & \textbf{\color{red}0.833} \\
        \cmidrule(lr){1-2} \cmidrule(lr){3-17}
		{\multirow{6}{*}{\rotatebox[origin=c]{90}{\scriptsize 1DSfM \cite{wilson2014robust}}}} 
		& {\scriptsize\sf Gen. Markt} & {8.181} & {4.024} &  {4.028} &{4.835} & {4.390} & \textbf{\color{red}3.977} & {4.870} & {4.616} & \textbf{\color{red}3.986} & {4.900} & {4.646} & \textbf{\color{red}3.988} & {4.942} & {4.688} & \textbf{\color{red}4.017} \\
		& {\scriptsize\sf Piccadily} & {11.26} & {4.379} &  {4.424} &{6.040} & {4.634} & \textbf{\color{red}4.218} & {6.073} & {4.761} & \textbf{\color{red}4.217} & {6.148} & {4.850} & \textbf{\color{red}4.249} & {6.193} & {5.027} & \textbf{\color{red}4.312} \\
		& {\scriptsize\sf R. Forum} & {6.407} & {1.211} &  {1.177} &{2.194} & {1.381} & \textbf{\color{red}1.152} & {2.211} & {1.446} & \textbf{\color{red}1.161} & {2.232} & {1.516} & \textbf{\color{red}1.171} & {2.251} & {1.585} & \textbf{\color{red}1.168} \\
		& {\scriptsize\sf Trafalgar} & {10.77} & {3.702} &  {3.886} &{5.065} & {3.994} & \textbf{\color{red}3.604} & {5.263} & {4.111} & \textbf{\color{red}3.601} & {5.169} & {4.185} & \textbf{\color{red}3.624} & {5.193} & {4.395} & \textbf{\color{red}3.657} \\
		& {\scriptsize\sf U. Square} & {10.56} & {3.992} &  {3.977} &{5.376} & {4.233} & \textbf{\color{red}3.893} & {5.426} & {4.401} & \textbf{\color{red}3.888} & {5.485} & {4.519} & \textbf{\color{red}3.925} & {5.538} & {4.619} & \textbf{\color{red}3.975} \\
		& {\scriptsize\sf V. Cathedral} & {9.506} & {1.957} &  {1.959} &{3.119} & {2.068} & \textbf{\color{red}1.777} & {3.139} & {2.021} & \textbf{\color{red}1.770} & {3.273} & {2.057} & \textbf{\color{red}1.831} & {3.328} & {2.216} & \textbf{\color{red}1.834} \\
		\bottomrule
	\end{tabular}
    
    \vfill
    \vspace{2mm}

    \begin{tabular}{c cccccccccccccccc}
        \toprule
		\multicolumn{2}{c}{}& \multicolumn{15}{c}{Mean Reprojection Error with the Huber Loss}\\
		\cmidrule(lr){3-17}
		\multicolumn{2}{c}{Dataset} &\multirow{2}{*}{Init} & \multirow{2}{*}{$\ceres$} & \multirow{2}{*}{$\deeplm$} &  \multicolumn{3}{c}{4 Devices} & \multicolumn{3}{c}{8 Devices} & \multicolumn{3}{c}{16 Devices} & \multicolumn{3}{c}{32 Devices}\\
		\cmidrule(lr){6-8} \cmidrule(lr){9-11} \cmidrule(lr){12-14} \cmidrule(lr){15-17}
		\multicolumn{2}{c}{} & & & & {\,\scriptsize $\dr$\,} & {\scriptsize $\admm$} & {\scriptsize $\amm$} & {\,\scriptsize $\dr$\,} & {\scriptsize $\admm$} & {\scriptsize $\amm$} & {\,\scriptsize $\dr$\,} & {\scriptsize $\admm$}  & {\scriptsize $\amm$}  & {\,\scriptsize $\dr$\,} & {\scriptsize $\admm$} &  {\scriptsize $\amm$}\\
		\cmidrule(lr){1-2} \cmidrule(lr){3-17}
		{\multirow{3}{*}{\rotatebox[origin=c]{90}{\scriptsize BAL \cite{agarwal2010bundle}}}} 
		& {\scriptsize\sf Ladybug} & {3.267} & {0.704}  & - & {0.834} & {\color{red}0.698} & \textbf{\color{red}0.690} & {0.841} & {\color{red}0.703} & \textbf{\color{red}0.690} & {0.844} & {0.712} & \textbf{\color{red}0.690} & {0.848} & {0.723} & \textbf{\color{red}0.690} \\
		& {\scriptsize\sf Venice} & {4.750} & {0.468}  & - & {0.515} & {0.516} & \textbf{\color{red}0.465} & {0.515} & {0.524} & \textbf{\color{red}0.465} & {0.516} & {0.530} & \textbf{\color{red}0.465} & {0.517} & {0.544} & \textbf{0.473} \\
		& {\scriptsize\sf Final} & {7.573} & {0.815}  & - & {1.995} & {0.859} & \textbf{\color{red}0.796} & {2.020} & {0.882} & \textbf{\color{red}0.795} & {2.032} & {0.886} & \textbf{\color{red}0.796} & {2.101} & {0.898} & \textbf{\color{red}0.815} \\
		\cmidrule(lr){1-2} \cmidrule(lr){3-17}
		{\multirow{6}{*}{\rotatebox[origin=c]{90}{\scriptsize 1DSfM \cite{wilson2014robust}}}} 
		& {\scriptsize\sf Gen. Markt} & {8.031} & {3.942}  & - & {4.777} & {4.320} & \textbf{\color{red}3.893} & {4.812} & {4.525} & \textbf{\color{red}3.914} & {4.842} & {4.561} & \textbf{\color{red}3.925} & {4.886} & {4.598} & \textbf{\color{red}3.942} \\
		& {\scriptsize\sf Piccadily} & {10.34} & {4.136}  & - & {5.739} & {4.334} & \textbf{\color{red}3.951} & {5.776} & {4.441} & \textbf{\color{red}3.963} & {5.851} & {4.564} & \textbf{\color{red}3.961} & {5.892} & {4.689} & \textbf{\color{red}3.978} \\
		& {\scriptsize\sf R. Forum} & {6.247} & {1.136}  & - & {2.149} & {1.296} & \textbf{\color{red}1.120} & {2.163} & {1.406} & \textbf{\color{red}1.127} & {2.187} & {1.471} & \textbf{\color{red}1.133} & {2.204} & {1.516} & \textbf{\color{red}1.133} \\
		& {\scriptsize\sf Trafalgar} & {9.900} & {3.544}  & - & {4.816} & {3.722} & \textbf{\color{red}3.363} & {4.939} & {3.876} & \textbf{\color{red}3.377} & {5.000} & {3.952} & \textbf{\color{red}3.401} & {4.944} & {4.149} & \textbf{\color{red}3.438} \\
		& {\scriptsize\sf U. Square} & {10.18} & {3.987}  & - & {5.242} & {4.096} & \textbf{\color{red}3.755} & {5.291} & {4.261} & \textbf{\color{red}3.763} & {5.348} & {4.366} & \textbf{\color{red}3.763} & {5.400} & {4.470} & \textbf{\color{red}3.818} \\
		& {\scriptsize\sf V. Cathedral} & {9.085} & {1.748}  & - & {3.008} & {1.928} & \textbf{\color{red}1.709} & {3.027} & {1.923} & \textbf{\color{red}1.681} & {3.156} & {1.956} & \textbf{\color{red}1.740} & {3.206} & {2.088} & \textbf{\color{red}1.734} \\
		\bottomrule
	\end{tabular}
	\vspace{-1em}
\end{table*}

Compared to $\amm$, decentralized methods \cite{eriksson2016consensus,zhang2017dist} have much stricter  assumptions for provable convergence, making them sensitive to parameter tuning in practice. Even though \cref{algorithm::amm} requires $\bfxakp$ to solve $\Ealphak$ or $\lEalphak$ for convergence guarantees, $\amm$ still empirically achieves good performances with approximate solutions, and thus, consumes fewer computational resources. Furthermore, as is shown  in the next section, $\amm$ is much faster to converge to more accurate solutions than the consensus methods \cite{eriksson2016consensus,zhang2017dist}, which suggests that  $\amm$ is preferable for large-scale decentralized bundle adjustment.

\section{Evaluation}\label{section::experiment}

\begin{table}[t]
	\centering
	\setlength{\tabcolsep}{0.3em}
	\caption{Largest Bundle adjustment datasets of more than 700 cameras in BAL \cite{agarwal2010bundle} and 1DSfM \cite{wilson2014robust}.}\label{table::large_dataset}
	\begin{tabular}{P{0.015\textwidth}  P{0.075\textwidth} P{0.07\textwidth} P{0.075\textwidth}P{0.095\textwidth}}
		\toprule			
		\multicolumn{2}{c}{Dataset} & \# Cameras & \# Points &\# Observations  \\
		\cmidrule(lr){1-2} \cmidrule(lr){3-5}
 		{\multirow{3}{*}{\rotatebox[origin=c]{90}{\scriptsize BAL \cite{agarwal2010bundle}}}} & {\scriptsize\sf Ladybug} & 1723 & 156502 & 678718  \\
		&{\scriptsize\sf Venice} & 1778 & 993923 & 5001946  \\
		& {\scriptsize\sf Final} & 13682 & 4456117 & 28987644 \\
		\cmidrule(lr){1-2} \cmidrule(lr){3-5}
		{\multirow{6}{*}{\rotatebox[origin=c]{90}{\scriptsize 1DSfM \cite{wilson2014robust}}}} & {\scriptsize\sf Gen. Markt} & 706 & 93672 & 364029  \\
		& {\scriptsize\sf Piccadilly} & 2289 & 209504 & 999878  \\
		& {\scriptsize\sf R. Forum} & 1063 & 265047 &  1292756  \\
		& {\scriptsize\sf Trafalgar} & 5032 & 388956 & 1826071  \\
		& {\scriptsize\sf U. Square} & 796 & 46066 & 230811  \\
		& {\scriptsize\sf V. Cathedral} & 836 & 265553 & 1333280 \\
		\bottomrule
	\end{tabular}
    \vspace{-0.5em}
\end{table}

We perform extensive benchmarks on the BAL \cite{agarwal2010bundle} and 1DSfM \cite{wilson2014robust} datasets and compare our method $\amm$ (\cref{algorithm::amm}) with centralized (single device) methods $\ceres$\cite{ceres-solver} and $\deeplm$\cite{huang2021deeplm}, and decentralized methods $\dr$\cite{eriksson2016consensus} and $\admm$ \cite{zhang2017dist}. For all decentralized methods, the datasets are partitioned by evenly distributing measurements to each device if there are more than 700 cameras. Otherwise, the partitioning is determined according to the number of cameras per device. $\dr$, $\admm$ and $\amm$ are implemented in CUDA,  $\ceres$ in C++, and $\deeplm$ in CUDA \& Python. All experiments are conducted on a computer with 80 Intel Xeon 2.2GHz CPU cores and 8 Nvidia V100 GPUs, and OpenMPI is used for inter-device communication. We compare all methods across metrics on accuracy, efficiency, memory usage and communication load.

\vspace{-0.2em}
\subsection{Accuracy}\label{section::experiment::accuracy}
\vspace{-0.1em}

We evaluate the mean reprojection errors of all methods. 
The centralized methods are run for 40 iterations and decentralized methods for 1000 iterations to ensure sufficient accuracy.
In the default solver options, $\ceres$ and $\deeplm$ have a maximum of 50 and 20 iterations, respectively. However, we found that these centralized methods usually exhibit limited improvement after 30 iterations, and the optimization time per iteration tends to increase significantly beyond this point. 
Thus, we run $\ceres$ and $\deeplm$ for 40 iterations to balance accuracy and time. 
We report in \cref{table::accuracy} the mean reprojection errors with the trivial loss and Huber loss on the nine largest datasets of more than 700 cameras in BAL \& 1DSfM (see \cref{table::large_dataset}), while the results for the remaining datasets of less than 700 cameras are in
\select{\cite[App. A.1]{fan2023daba}}{App. \hyperref[section::app::accuracy]{A.1}}.
The  3D reconstruction results of $\amm$ (ours) with 8 devices and the Huber loss on some of these datasets are shown  in \cref{fig::results}.

As shown in \cref{table::accuracy} and 
\select{\cite[App. A.1]{fan2023daba}}{App. \hyperref[section::app::accuracy]{A.1}},
$\amm$  yields the most accurate results compared to the other decentralized methods by a large margin across the board---all 20 datasets, both types of losses, and all numbers of devices. Additionally, $\amm$, while being fully decentralized, outperforms centralized methods $\ceres$ and $\deeplm$ on most of the datasets as summarized in \cref{table::stats}.
Unsurprisingly, decentralized methods converge to lower accuracy as the number of devices increases, since more devices lead to more relaxed upper bounds for $\amm$ and greater consensus errors for $\dr$ and $\admm$. However, we also find that larger problems are less impacted by the increasing number of devices. This validates that decentralized methods like $\amm$ are suitable for large-scale bundle adjustment.

\begin{figure*}[t]
	\centering
	\begin{tabular}{cccc}
		\hspace{-0.5em}\subfloat[][4 devices]{\includegraphics[trim =0mm 0mm 0mm 0mm,width=0.2425\textwidth]{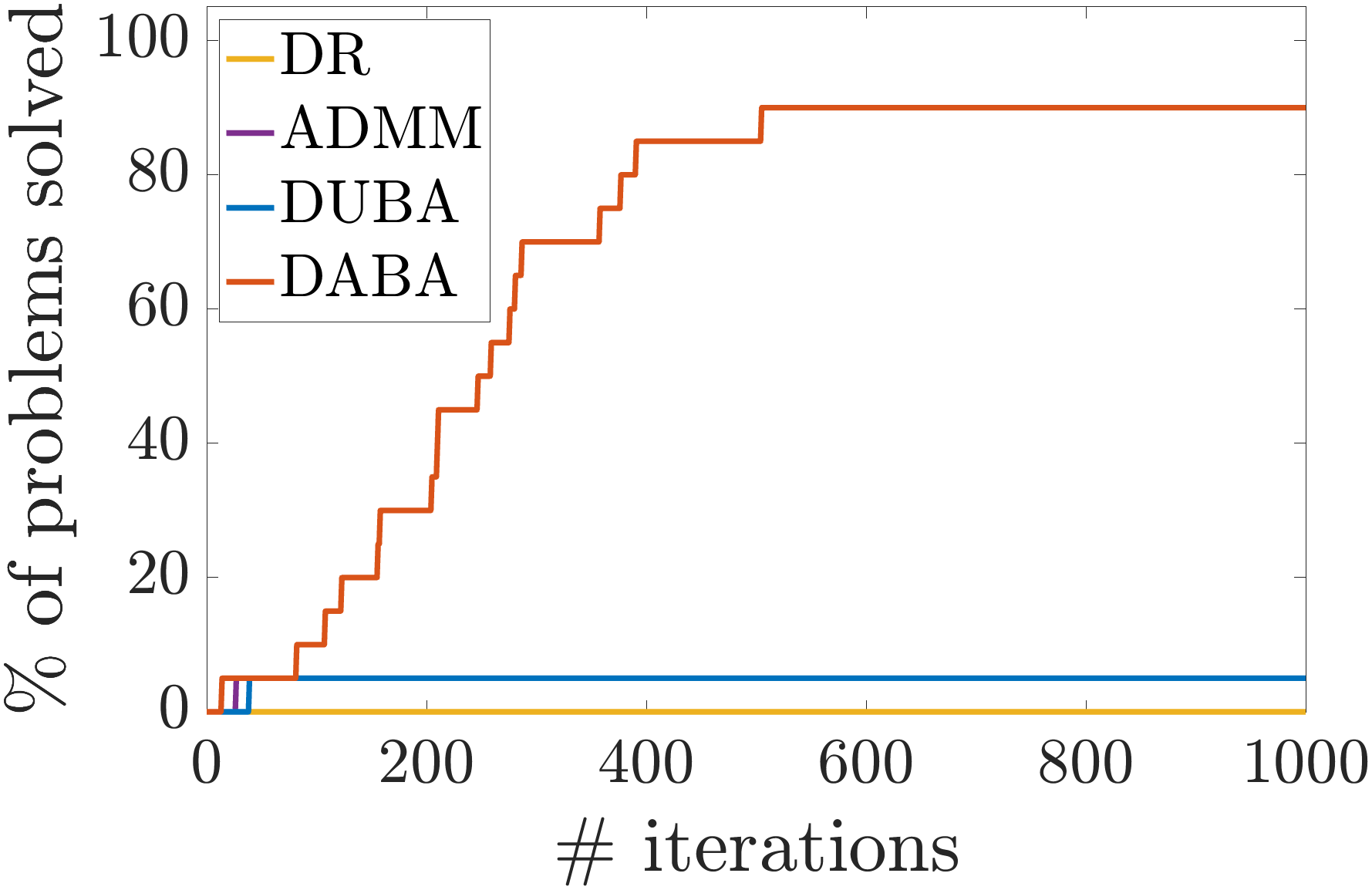}} &
		\hspace{-0.6em}\subfloat[][8 devices]{\includegraphics[trim =0mm 0mm 0mm 0mm,width=0.2425\textwidth]{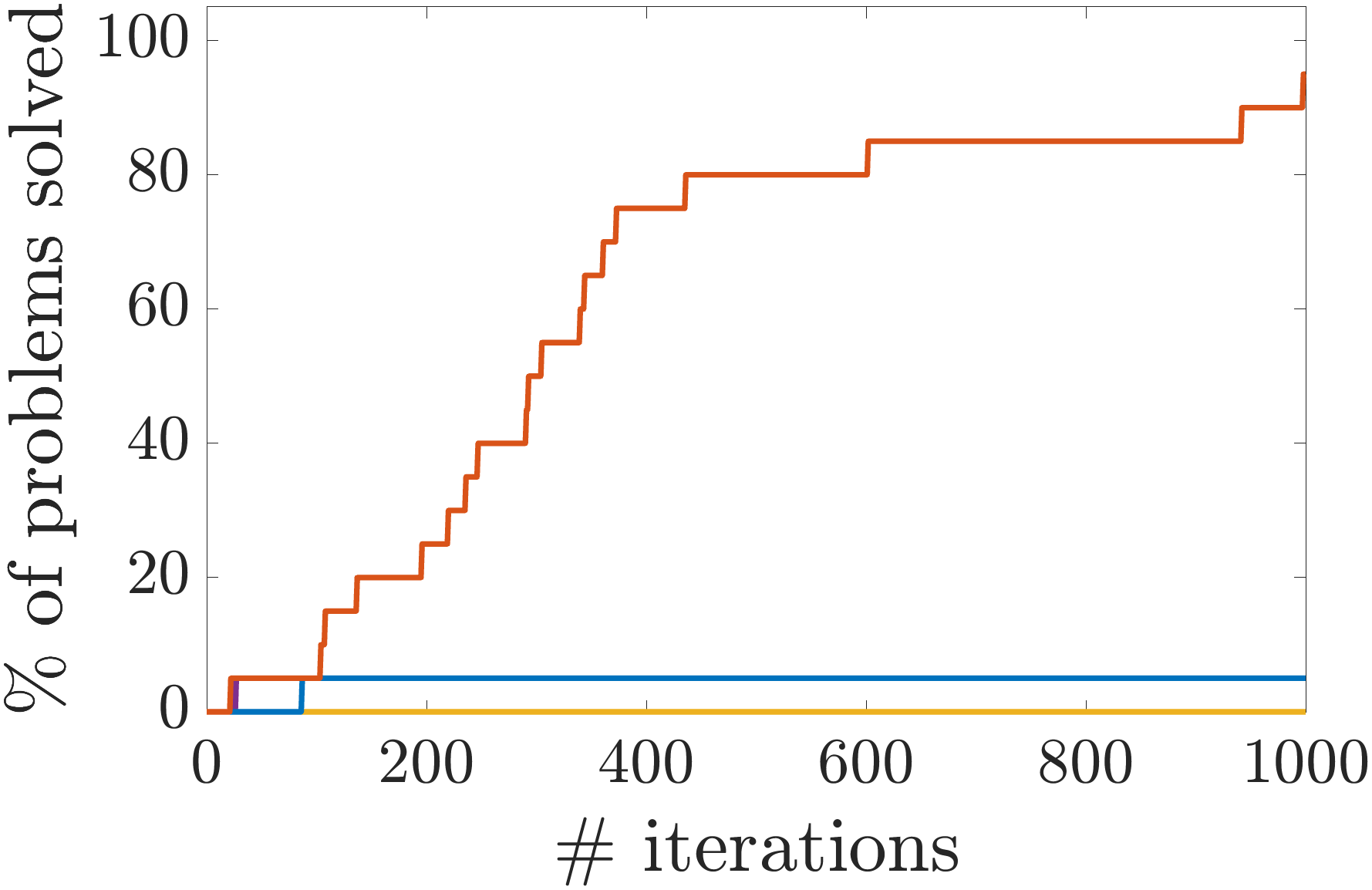}} &
		\hspace{-0.6em}\subfloat[][16 devices]{\includegraphics[trim =0mm 0mm 0mm 0mm,width=0.2425\textwidth]{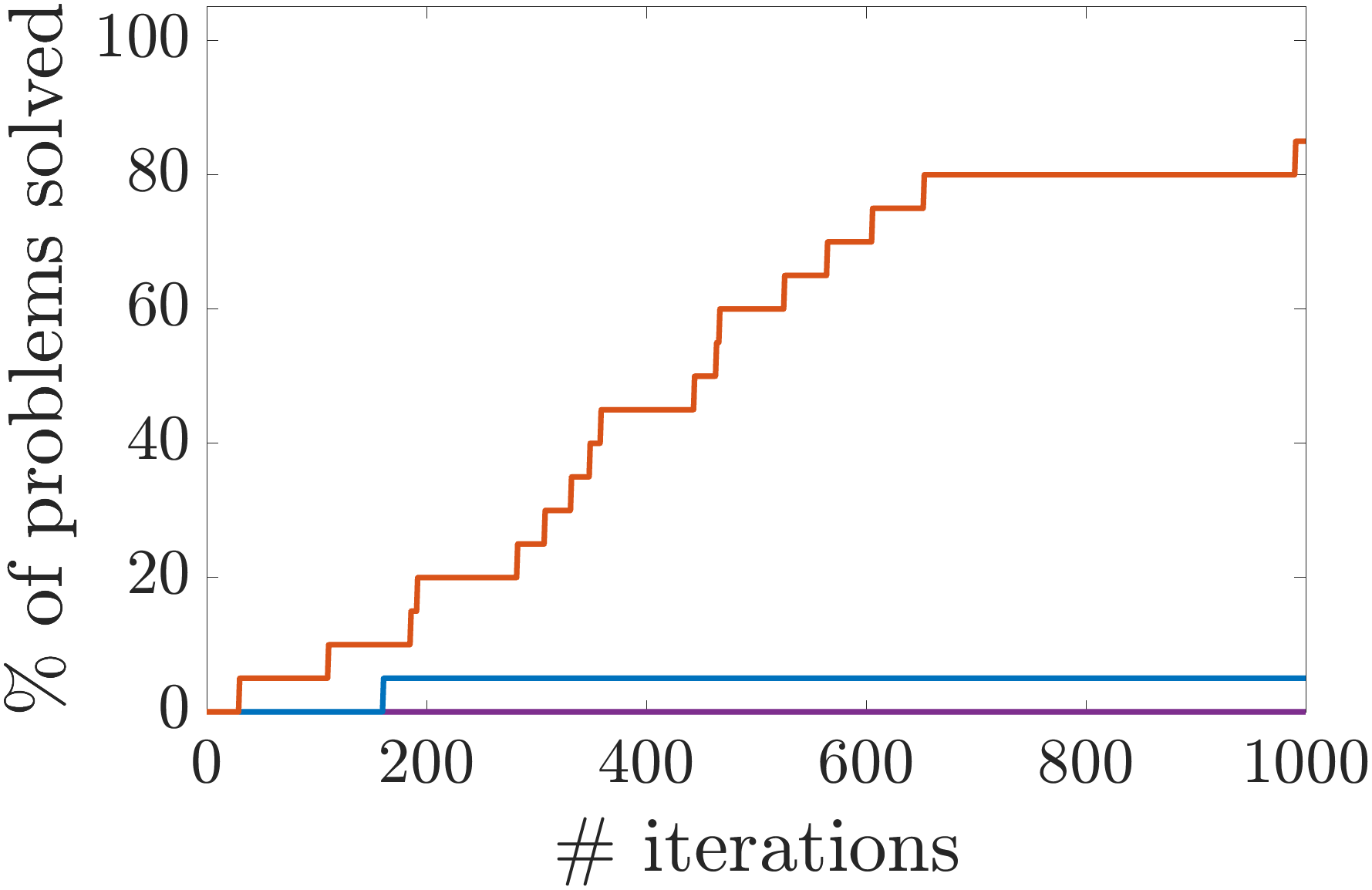}}&
		\hspace{-0.6em}\subfloat[][32 devices]{\includegraphics[trim =0mm 0mm 0mm 0mm,width=0.2425\textwidth]{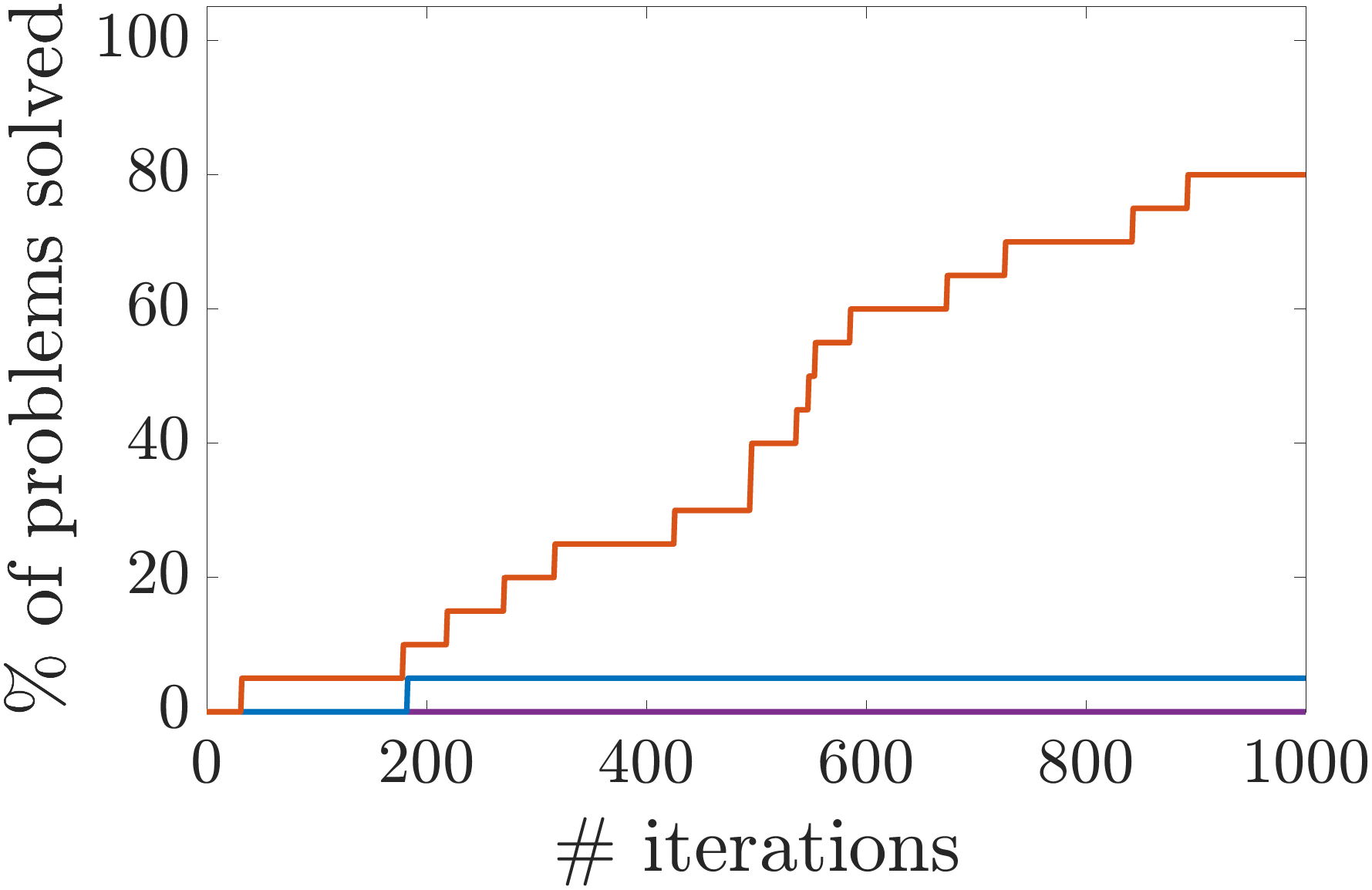}}
        \vspace{0.25em}
  \\
		\hspace{-0.5em}\subfloat[][4 devices]{\includegraphics[trim =0mm 0mm 0mm 0mm,width=0.2425\textwidth]{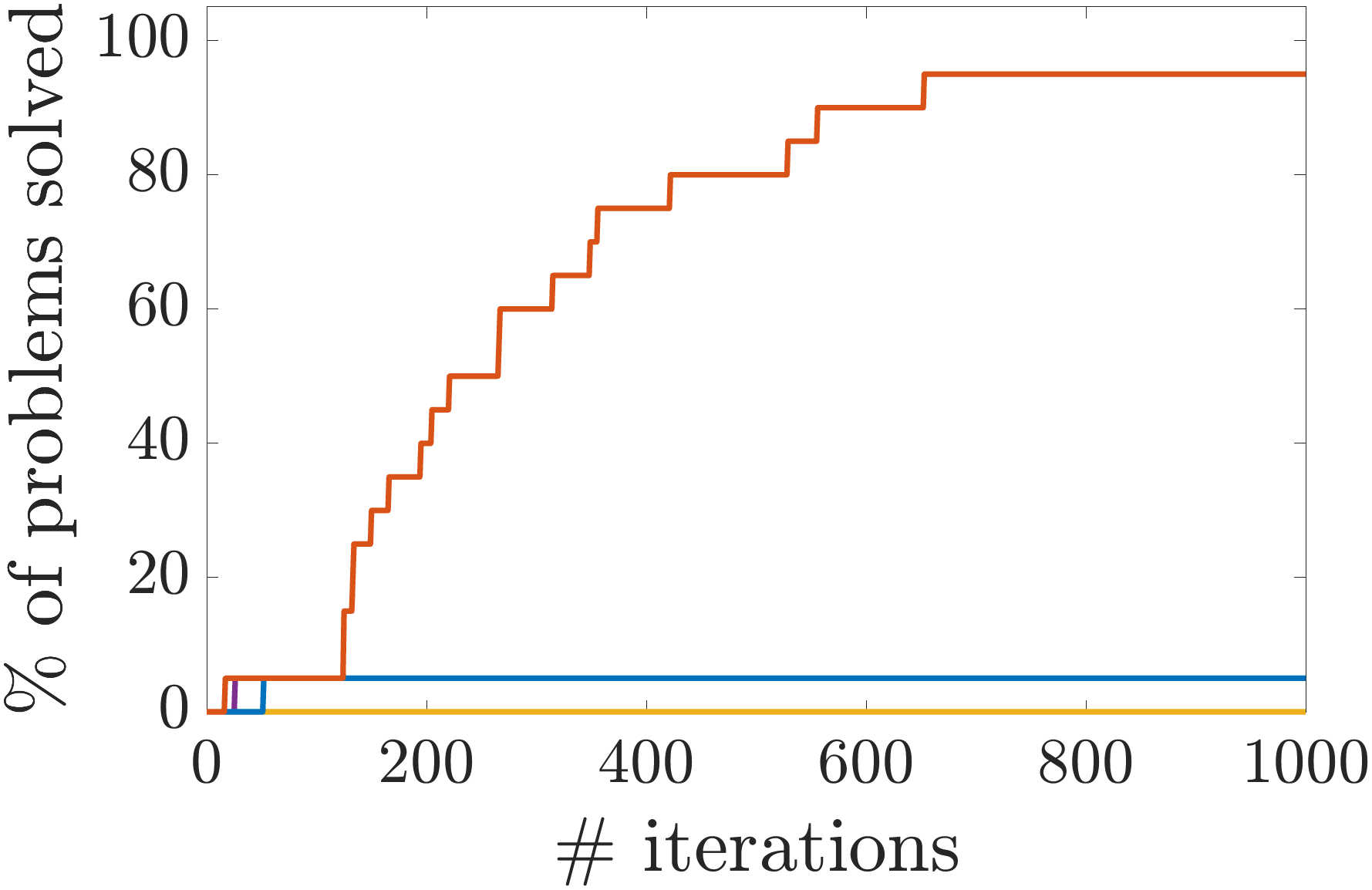}} &
		\hspace{-0.6em}\subfloat[][8 devices]{\includegraphics[trim =0mm 0mm 0mm 0mm,width=0.2425\textwidth]{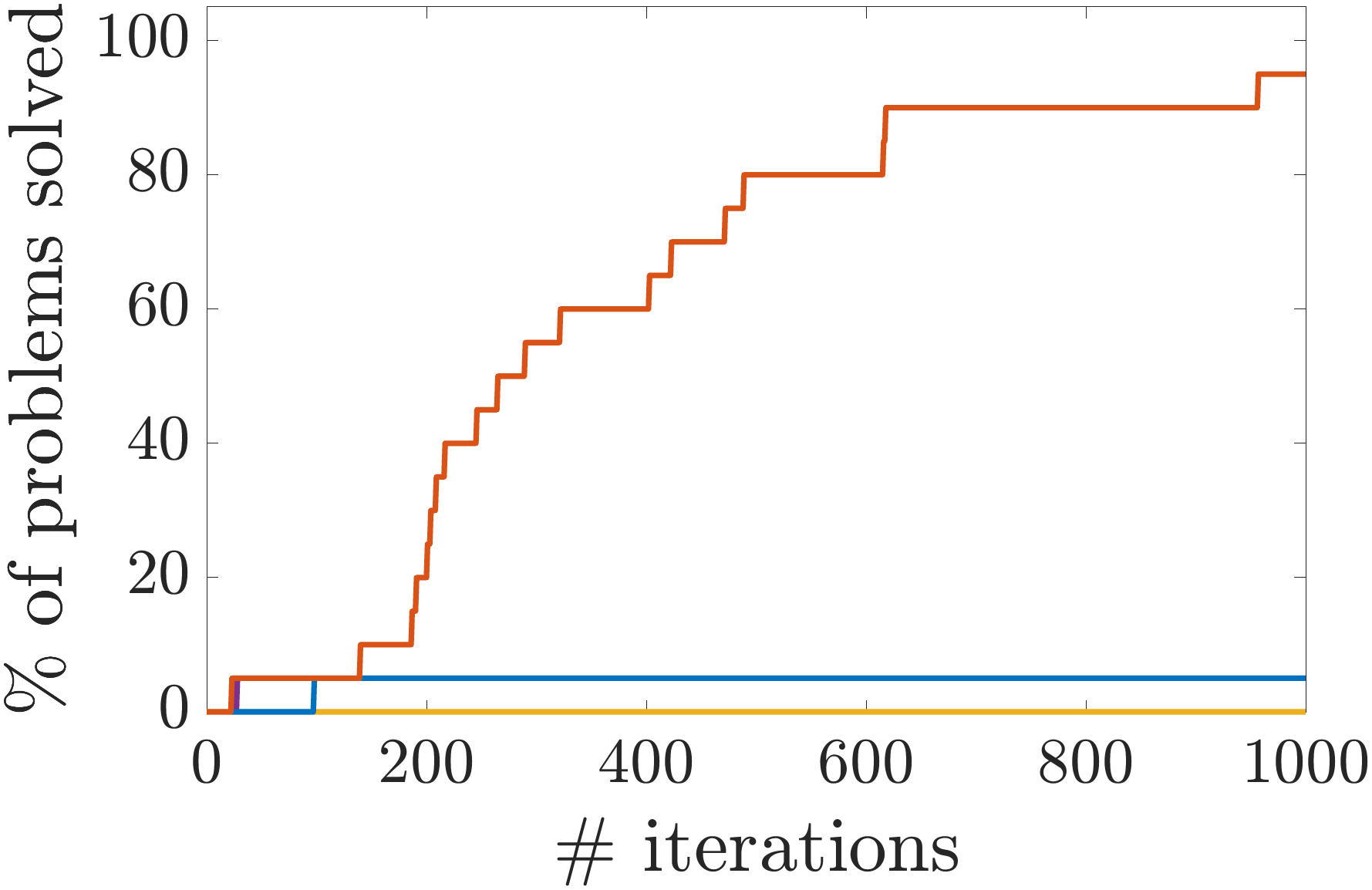}} &
		\hspace{-0.6em}\subfloat[][16 devices]{\includegraphics[trim =0mm 0mm 0mm 0mm,width=0.2425\textwidth]{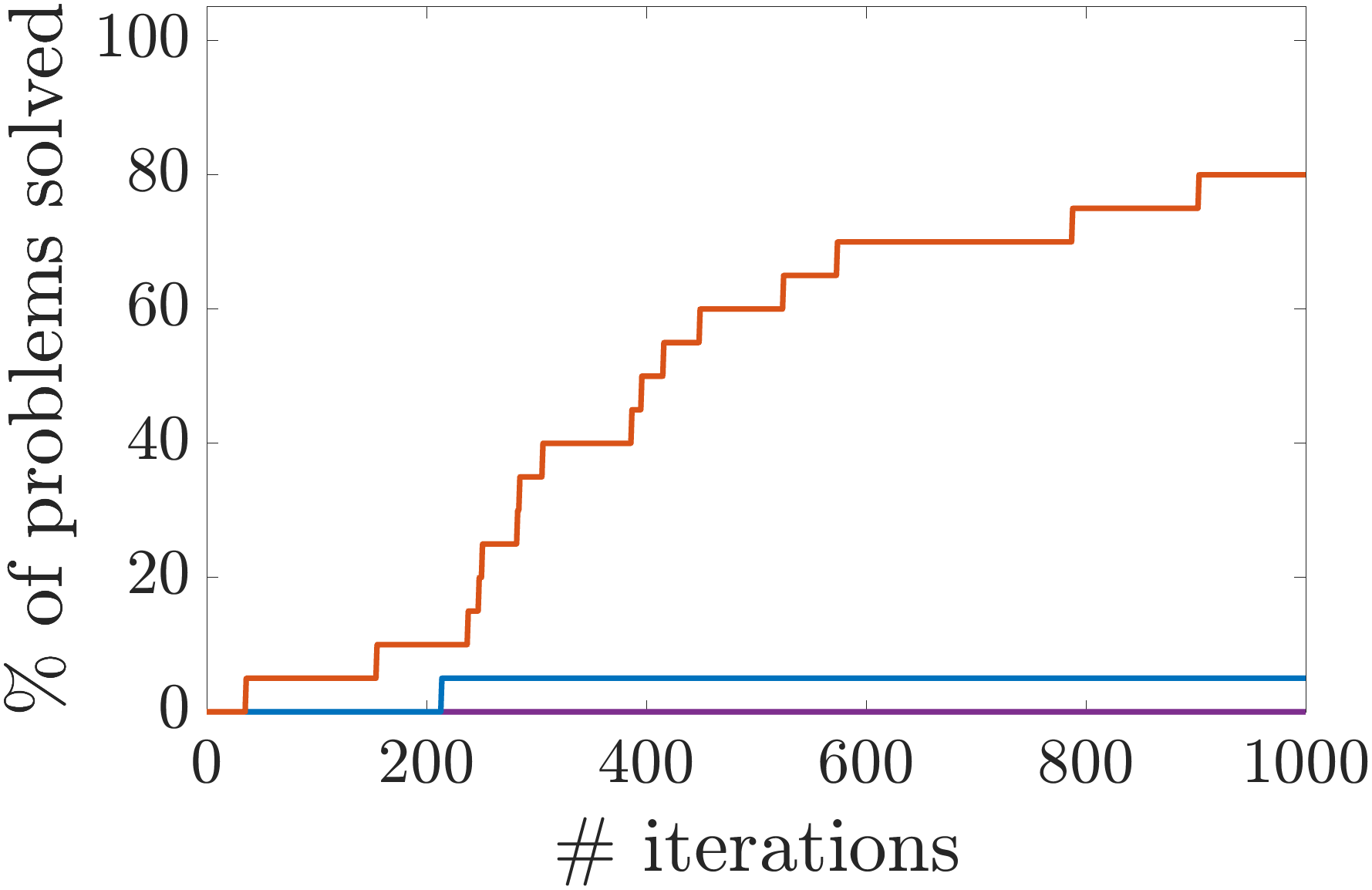}}&
		\hspace{-0.6em}\subfloat[][32 devices]{\includegraphics[trim =0mm 0mm 0mm 0mm,width=0.2425\textwidth]{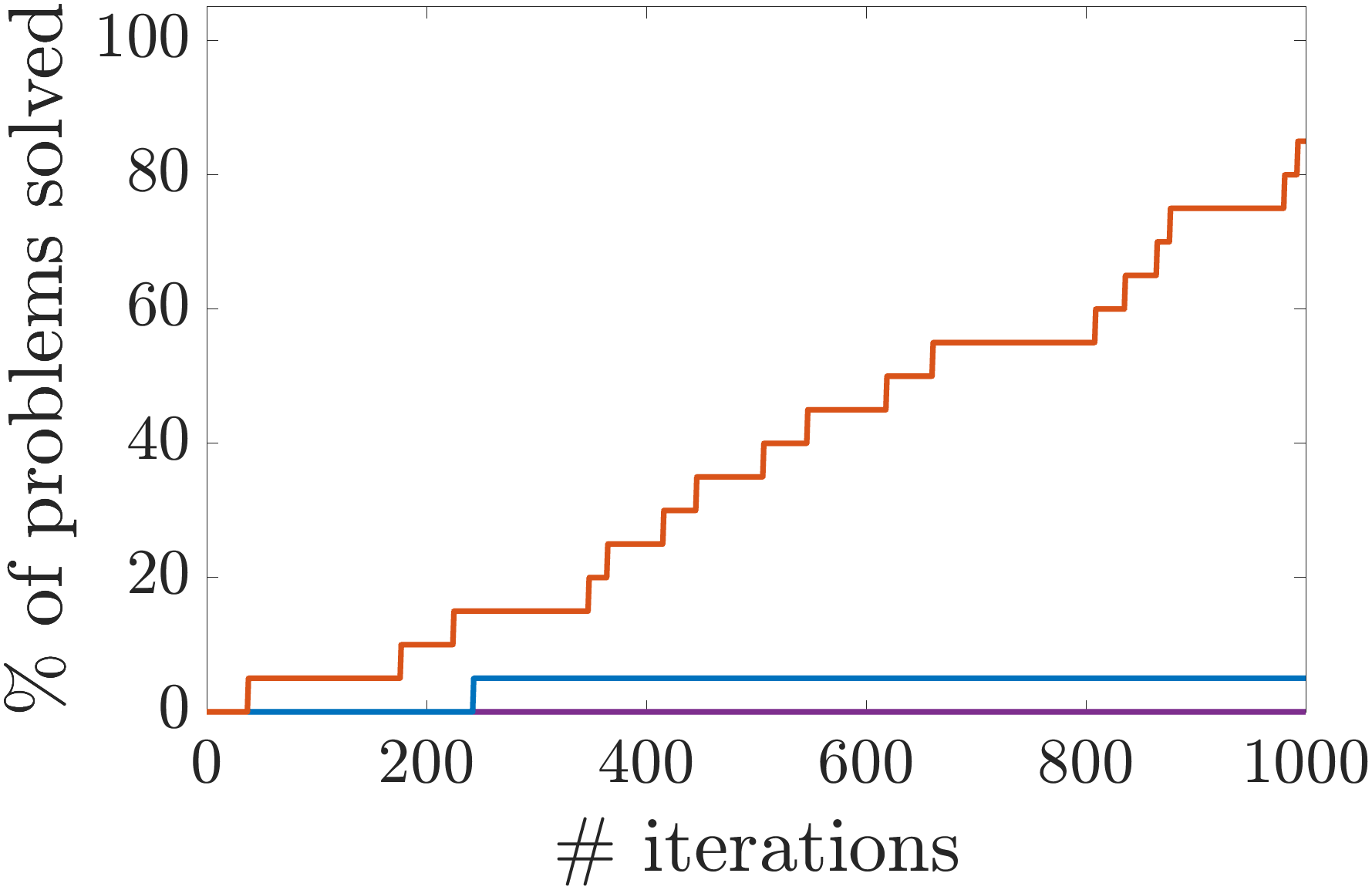}}
	\end{tabular}
	\caption{\textbf{Performance profiles} with respect to the number of iterations for $\amm$ (ours) compared to other decentralized methods on the 20 datasets in BAL and 1DSfM with (a)-(d) the trivial loss and (e)-(h) the Huber loss. Reference objective values $F_{\text{ref}}$ is from $\ceres$ \cite{ceres-solver} and suboptimality tolerance $\Delta=1\times10^{-4}$.}\label{fig::perf}
 \vspace{-1em}
	\end{figure*}
\begin{table}[t]
    \centering
    \renewcommand{\arraystretch}{1.0}
    \caption{The number of datasets out of All  in BAL \& 1DSfM (20 datasets) and Largest in \cref{table::large_dataset} (9 datasets) where $\amm$ is more accurate than $\ceres$ \cite{ceres-solver} and $\deeplm$ \cite{huang2021deeplm}.}
    \setlength{\tabcolsep}{0.65em}
    \begin{tabular}{cccccc}
    \toprule
    Datasets & Loss & 4 Devices & 8 Devices & 16 Devices & 32 Devices        \\
    \midrule
    \multirow{2}{*}{All} 
    & Trivial &  18/20 & 18/20 & 15/20  & 14/20\\
    & Huber & 19/20 & 19/20 & 16/20 & 17/20\\
    \midrule
    \multirow{2}{*}{Largest} 
     & Trivial & 9/9 & 9/9 & 9/9 & 8/9\\
     & Huber & 9/9 & 9/9 & 9/9 & 8/9\\
    \bottomrule 
    \end{tabular}
    \label{table::stats}
    \vspace{-0.25em}
\end{table}

$\dr$, $\admm$ and $\amm$ use Levenberg–Marquardt (LM) algorithm to solve subproblems, e.g., \cref{eq::update_amm,eq::update_mm}. We found $\dr$ and $\admm$ need up to 20 inner LM steps to converge to acceptable accuracy, whereas $\amm$ achieves better accuracy with only one successful inner LM step for all setups. This suggests that $\amm$ is more robust to approximated solutions of subproblems and more time efficient due to fewer inner LM steps, which is evaluated in the next section.

\subsection{Efficiency}\label{section::experiment::efficiency}

First, we evaluate the time speedup of $\amm$ (ours) with 4 and 8 devices against  centralized methods $\ceres$ \cite{ceres-solver} and $\deeplm$ \cite{huang2021deeplm} on the largest datasets of more than 700 cameras in \cref{table::large_dataset}. Each device uses one GPU and we consider together both the computation and communication time. The time speedup of $\amm$ for a problem $p$ is $\frac{T_{\Delta}(p)}{T_{\text{\sf\tiny DABA}}(p)}$, where $T_{\text{\sf\tiny DABA}}(p)$ is the time that $\amm$ takes to reach the reference objective value $F_{\text{ref}}$ while $T_{\Delta}(p)$  is the time that $\ceres$ or $\deeplm$ takes to reach a target objective value $F_{\Delta}(p)>F_{\text{ref}}$:
\begingroup
\begin{equation}\label{eq::Fdelta}
	F_{\Delta}(p)\triangleq F_{\text{ref}} + \Delta\cdot
	(F_{\text{init}} - F_{\text{ref}}) 
	\vspace{-0.25em}
\end{equation}
\endgroup
where $F_{\text{init}}$ is the initial objective value and $\Delta\in(0,\,1)$ is the suboptimality tolerance. Here, we choose $F_{\text{ref}}$ to be the smallest objective value separately achieved by $\ceres$ and $\deeplm$ for 40 iterations and $\Delta=2.5\times 10^{-4}$   since $\ceres$ and $\deeplm$  have a slow convergence around $F_{\text{ref}}$. In contrast, $\amm$ has to exactly attain $F_{\text{ref}}< F_\Delta(p)$. This means when measuring time, our method is subjected to a stricter convergence criteria. Nevertheless, \cref{table::speedup} indicates that $\amm$ is significantly faster in all setups, where $\amm$ is 23.9$\sim$588.3x faster than $\ceres$ and 2.6$\sim$174.6x faster than $\deeplm$ for the trivial loss, and 69.7$\sim$953.7x faster than $\ceres$ for the Huber loss. In summary, while $\amm$ might take several hundred iterations to reach accuracies comparable to $\ceres$ and $\deeplm$, our multi-GPU implementation makes $\amm$ empirically far more time-efficient for large-scale bundle adjustment problems.

Next, we analyze the efficiency of decentralized methods in terms of numbers of iterations. In addition to $\dr$ \cite{eriksson2016consensus}, $\admm$ \cite{zhang2017dist}, $\amm$ (ours), we also implement $\mm$ (Decentralized and Unaccelerated Bundle Adjustment) to ablate our method without acceleration and adaptive restart. The only difference between $\mm$ and $\amm$ is that $\bfxakp$ in $\mm$ is always yielded by \cref{eq::update_mm}, whereas  $\bfxakp$ in $\amm$ results from \cref{eq::update_amm} unless  adaptive restart  is triggered. Here, we compute the performance profiles \cite{dolan2002benchmarking} of all decentralized methods with respect to the number of iterations on all datasets. Given an optimizer and a problem set $\calP$, the performance profile at iteration $\sfk$ refers to the percentage of problems $p\in\calP$ solved, i.e. the optimizer attains the target objective value $F_{\Delta}(p)$ in \cref{eq::Fdelta} with the reference objective value $F_{\text{ref}}$ from $\ceres$ and suboptimality tolerance $\Delta=1\times10^{-4}$. The performance profiles with the trivial loss and the Huber loss on 4, 8, 16, 32 devices are shown in \cref{fig::perf}.
$\amm$ solves more problems than other methods and takes fewer iterations to reach more accurate results. $\amm$ is also much faster than $\mm$, demonstrating that Nesterov's acceleration and adaptive restart significantly improve performance.

\begin{table}[t]
	\centering
	\setlength{\tabcolsep}{0.1em}
    \renewcommand{\arraystretch}{1.0}
    \caption{
    \textbf{Time speedup} of $\amm$ with 4 and 8 devices over $\ceres$ \cite{ceres-solver} and $\deeplm$ \cite{huang2021deeplm} on datasets of more than 700 cameras.
    \textbf{$\amm$ achieves speedups on all ranging from 23.9$\sim$953.7x over $\ceres$ and 2.6$\sim$174.6x over $\deeplm$}.
    }
    \label{table::speedup}
    \begin{tabular}{P{0.02\textwidth} P{0.08\textwidth} P{0.055\textwidth} P{0.055\textwidth} P{0.055\textwidth} P{0.055\textwidth} P{0.07\textwidth} P{0.07\textwidth}}
		\toprule
		\multicolumn{2}{c}{} & \multicolumn{6}{c}{Time Speedup}\\
		 \cmidrule(lr){3-8}
		\multicolumn{2}{c}{\multirow{3}{*}{Dataset}} &  \multicolumn{4}{c}{Trivial Loss} &  \multicolumn{2}{c}{Huber Loss}\\
		\cmidrule(lr){3-6} \cmidrule(lr){7-8}
		\multicolumn{2}{c}{}  & \multicolumn{2}{c}{4 Devices} & \multicolumn{2}{c}{8 Devices} & 4 Devices & 8 Devices\\
		\cmidrule(lr){3-4} \cmidrule(lr){5-6} \cmidrule(lr){7-8}
		\multicolumn{2}{c}{}   & {\scriptsize $\ceres$} & {\scriptsize $\deeplm$} & {\scriptsize $\ceres$} & {\scriptsize $\deeplm$} & \multicolumn{2}{c}{\scriptsize $\ceres$}  \\
		\cmidrule(lr){1-2} \cmidrule(lr){3-8}
		{\multirow{3}{*}{\rotatebox[origin=c]{90}{\scriptsize BAL \cite{agarwal2010bundle}}}}
		& {\scriptsize\sf Ladybug}     & 575.1  & 174.6  & 588.3 & 165.1 & 811.3 & 953.7 \\
		& {\scriptsize\sf Venice}        & 23.9    & 2.6      & 43.9    &   4.8   & 89.2   & 148.4 \\
		& {\scriptsize\sf Final}            & 185.0  & 22.7    & 258.2  & 32.8 &  173.2   & 250.6 \\
		\cmidrule(lr){1-2} \cmidrule(lr){3-8}
		{\multirow{6}{*}{\rotatebox[origin=c]{90}{\scriptsize 1DSfM\cite{wilson2014robust}}}} 
		& {\scriptsize\sf Gen. Markt}   & 99.7     & 24.6   & 65.4    & 16.2 & 102.0  & 65.0 \\
		& {\scriptsize\sf Piccadilly}      & 117.2    & 18.1    & 170.5  & 23.8 & 139.3  & 155.2 \\
		& {\scriptsize\sf R. Forum}       & 138.3   &  11.0    & 181.8  & 11.2 &  69.7    & 97.1 \\
		& {\scriptsize\sf Trafalgar}       & 148.2   &  29.7   & 239.6 & 41.7 & 174.1   & 228.0  \\
		& {\scriptsize\sf U. Square}     & 59.2     &   9.7    & 49.9    & 8.6 & 162.3    & 117.1 \\
		& {\scriptsize\sf V. Cathedral} & 308.8  & 32.6   & 477.9  & 50.9 & 176.6   & 419.8 \\
		\bottomrule
	\end{tabular}
\vspace{-0.15em}
\end{table}
\begin{figure*}[t]
	\centering
	\begin{tabular}{cccc}
		\hspace{-0.5em}\subfloat[][{\sf\small Venice}]{\includegraphics[trim =0mm 0mm 0mm 0mm,width=0.2425\textwidth]{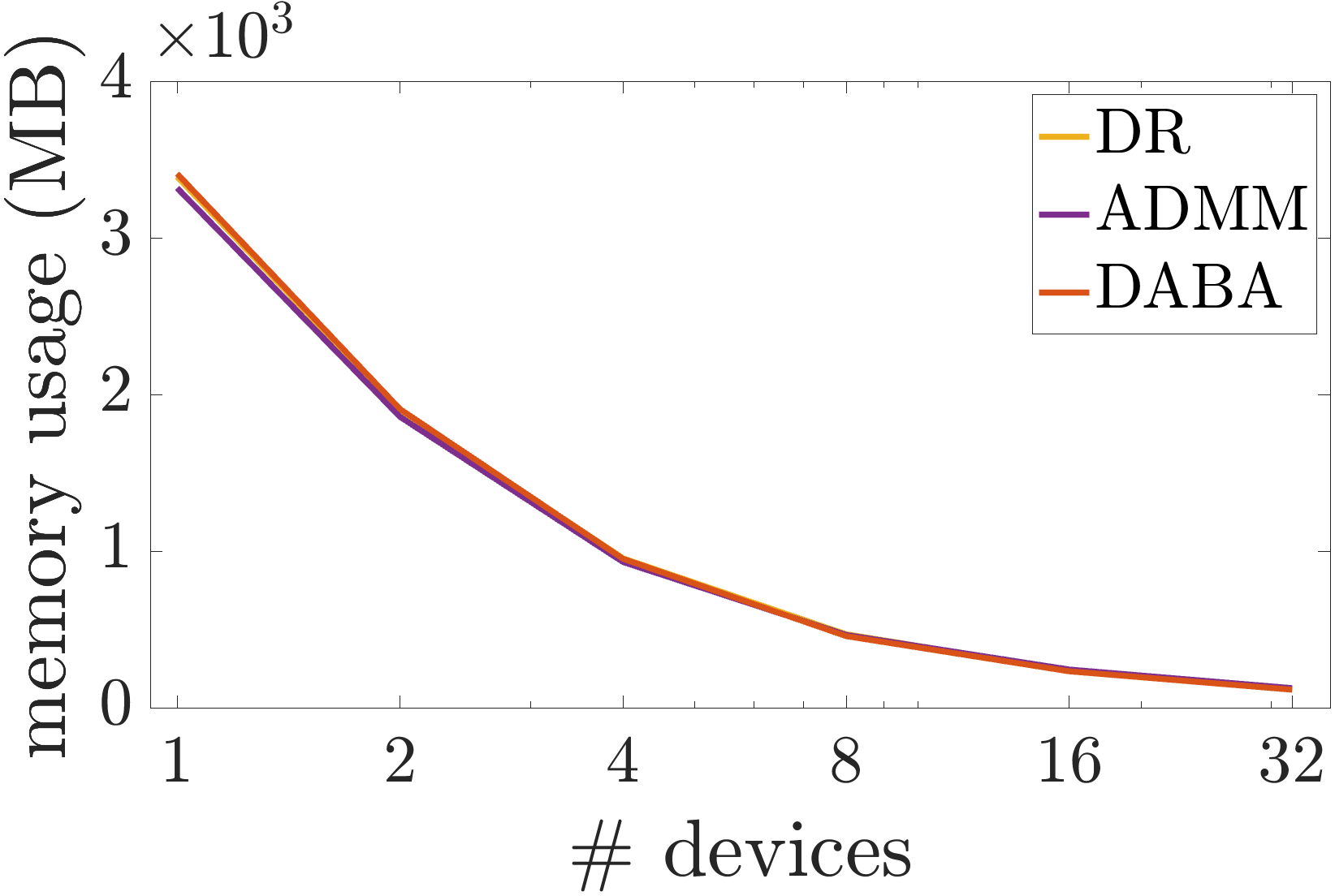}} &
		\hspace{-0.6em}\subfloat[][{\sf\small Final}]{\includegraphics[trim =0mm 0mm 0mm 0mm,width=0.2425\textwidth]{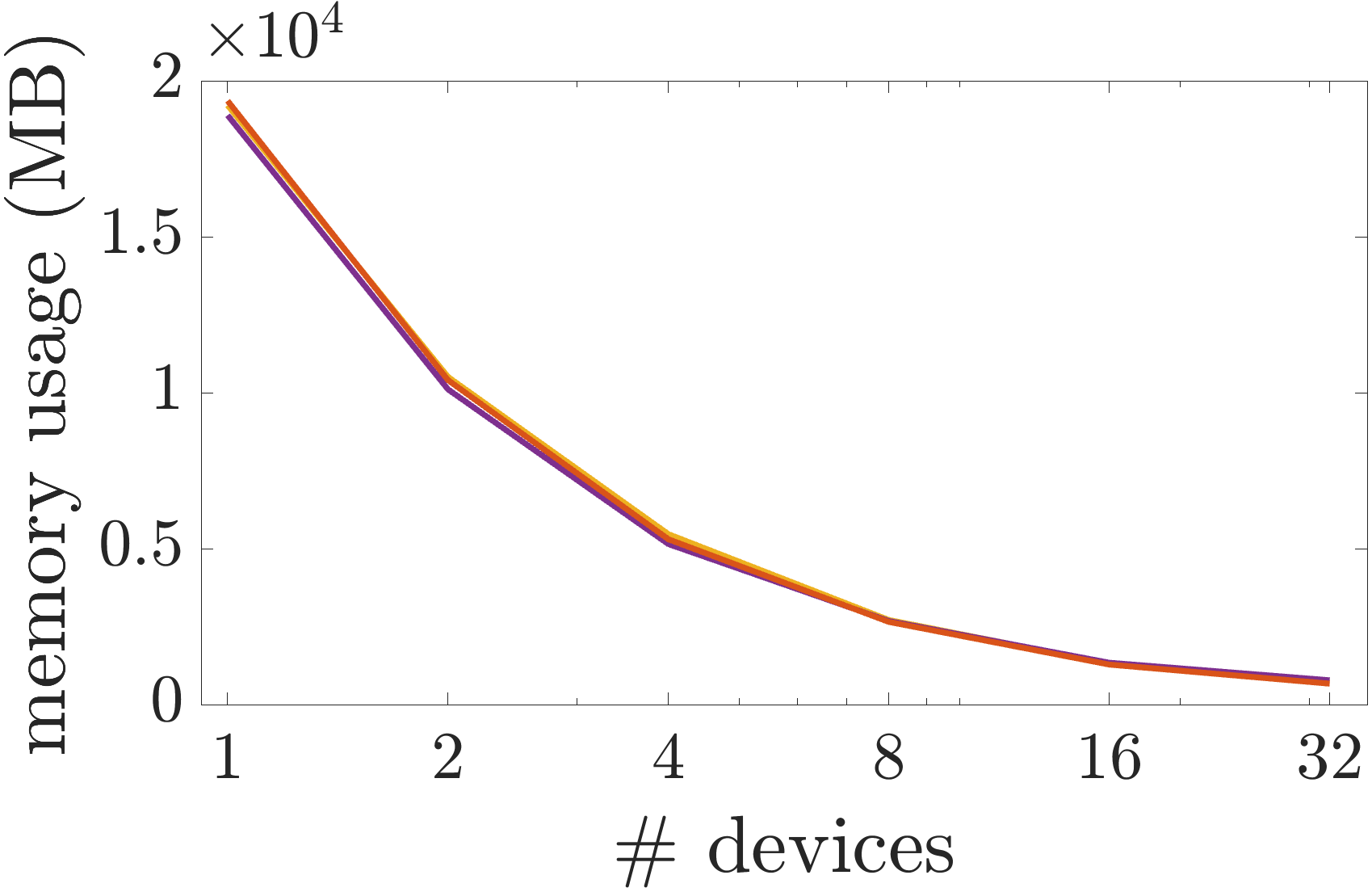}} &
		\hspace{-0.6em}\subfloat[][{\sf\small Piccadilly}]{\includegraphics[trim =0mm 0mm 0mm 0mm,width=0.2425\textwidth]{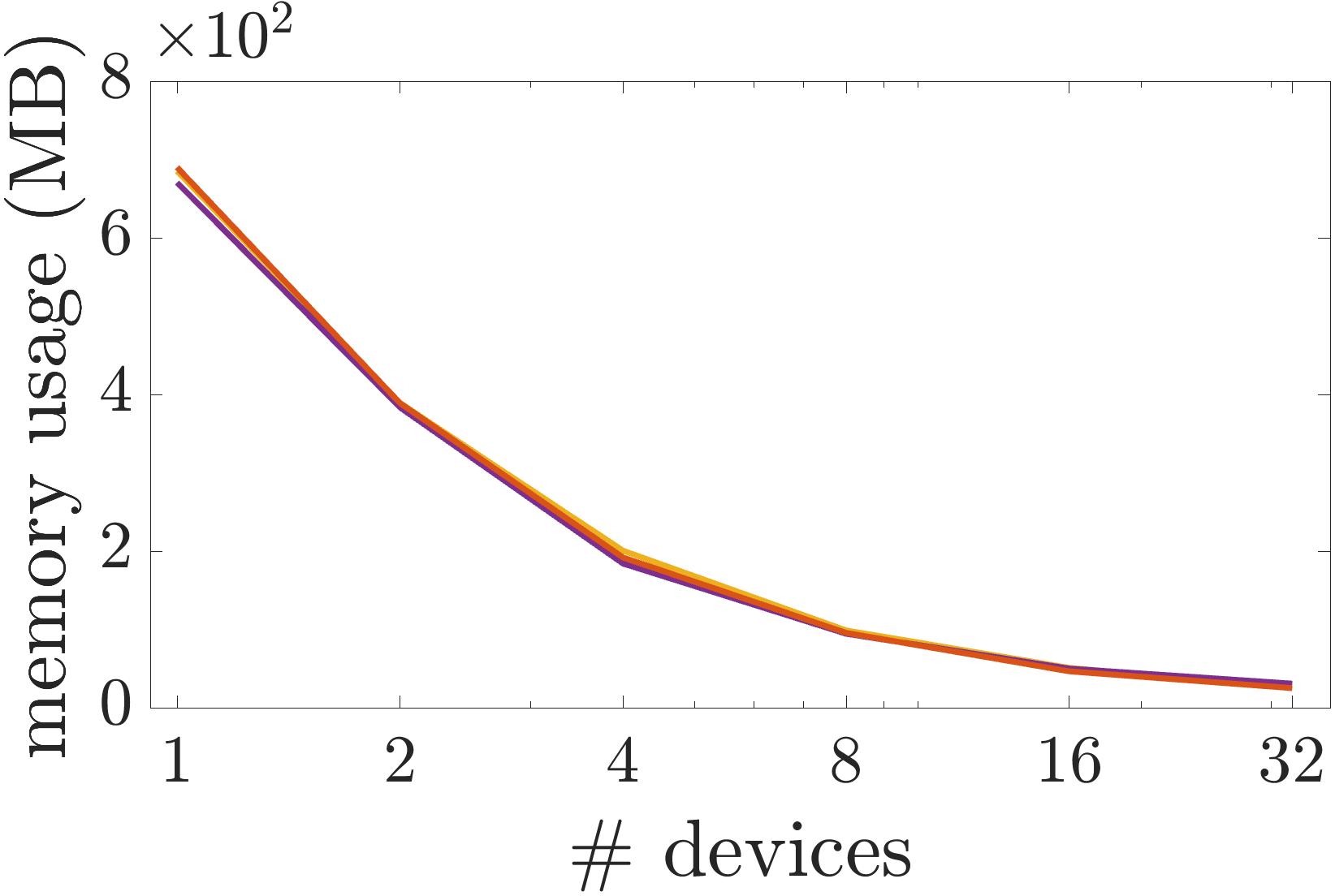}}&
		\hspace{-0.6em}\subfloat[][{\sf\small Trafalgar}]{\includegraphics[trim =0mm 0mm 0mm 0mm,width=0.2425\textwidth]{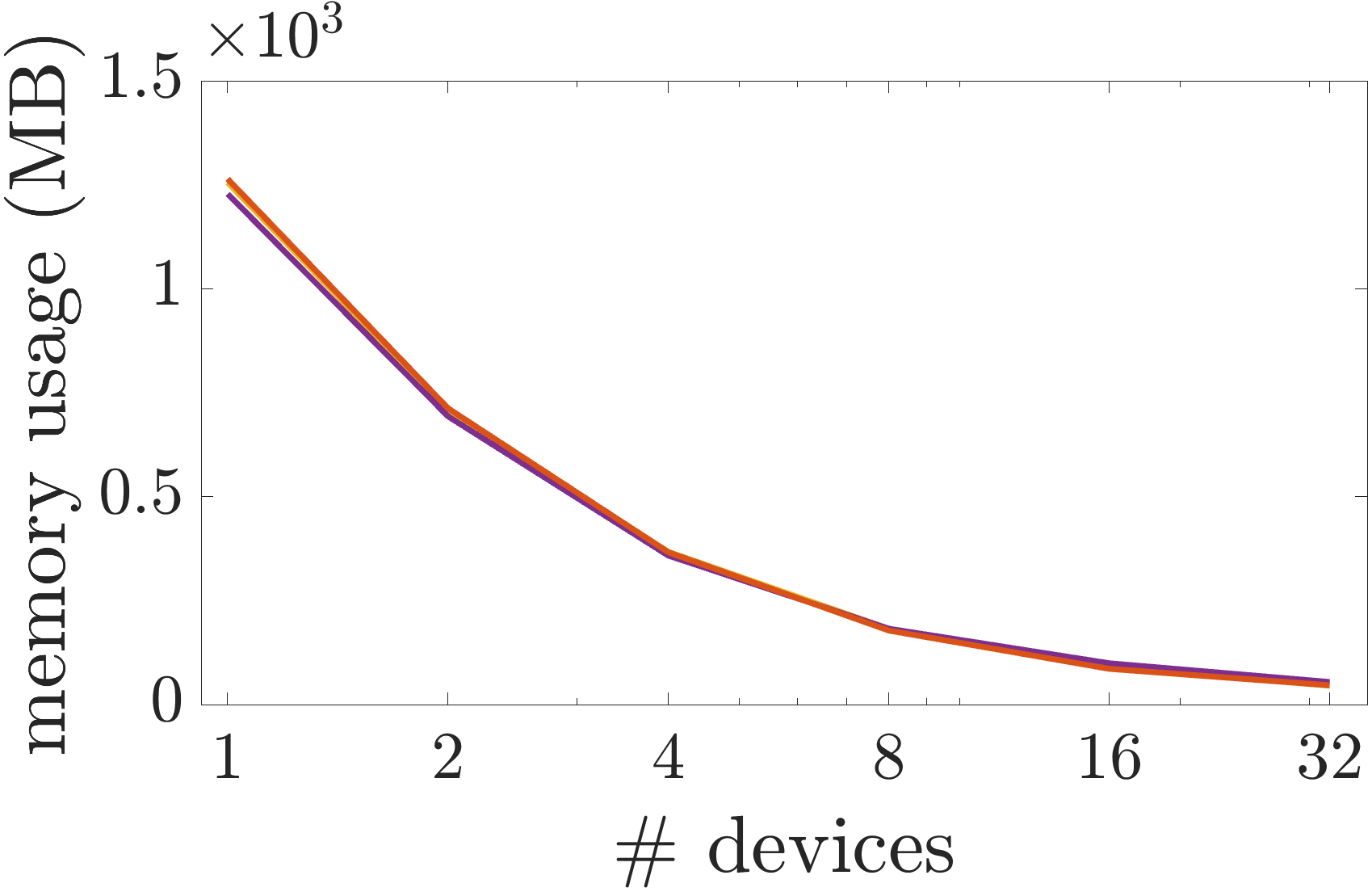}}
	\end{tabular}
	\begin{tabular}{cccc}
		\hspace{-0.5em}\subfloat[][{\sf\small Venice}]{\includegraphics[trim =0mm 0mm 0mm 0mm,width=0.2425\textwidth]{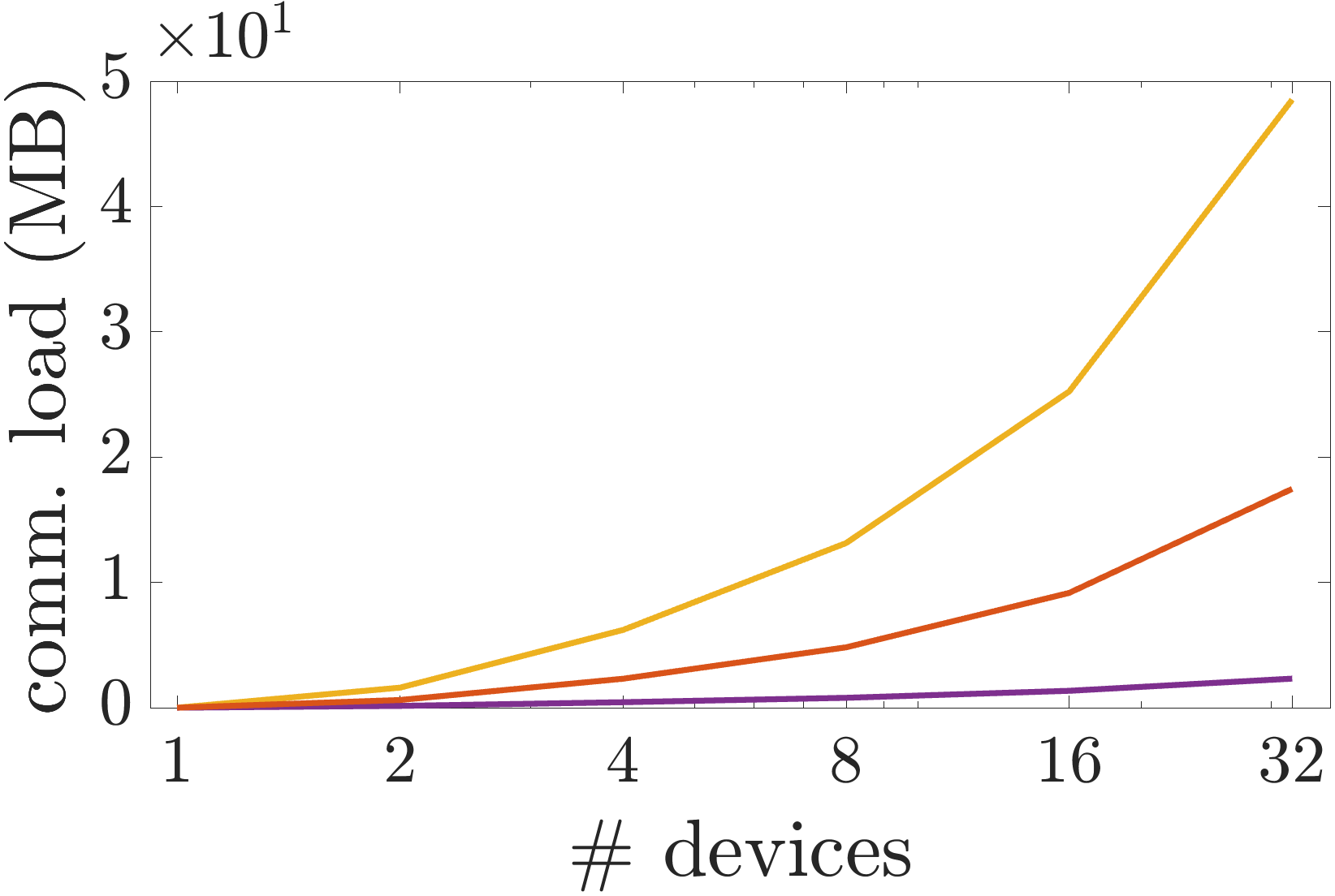}} &
		\hspace{-0.6em}\subfloat[][{\sf\small Final}]{\includegraphics[trim =0mm 0mm 0mm 0mm,width=0.2425\textwidth]{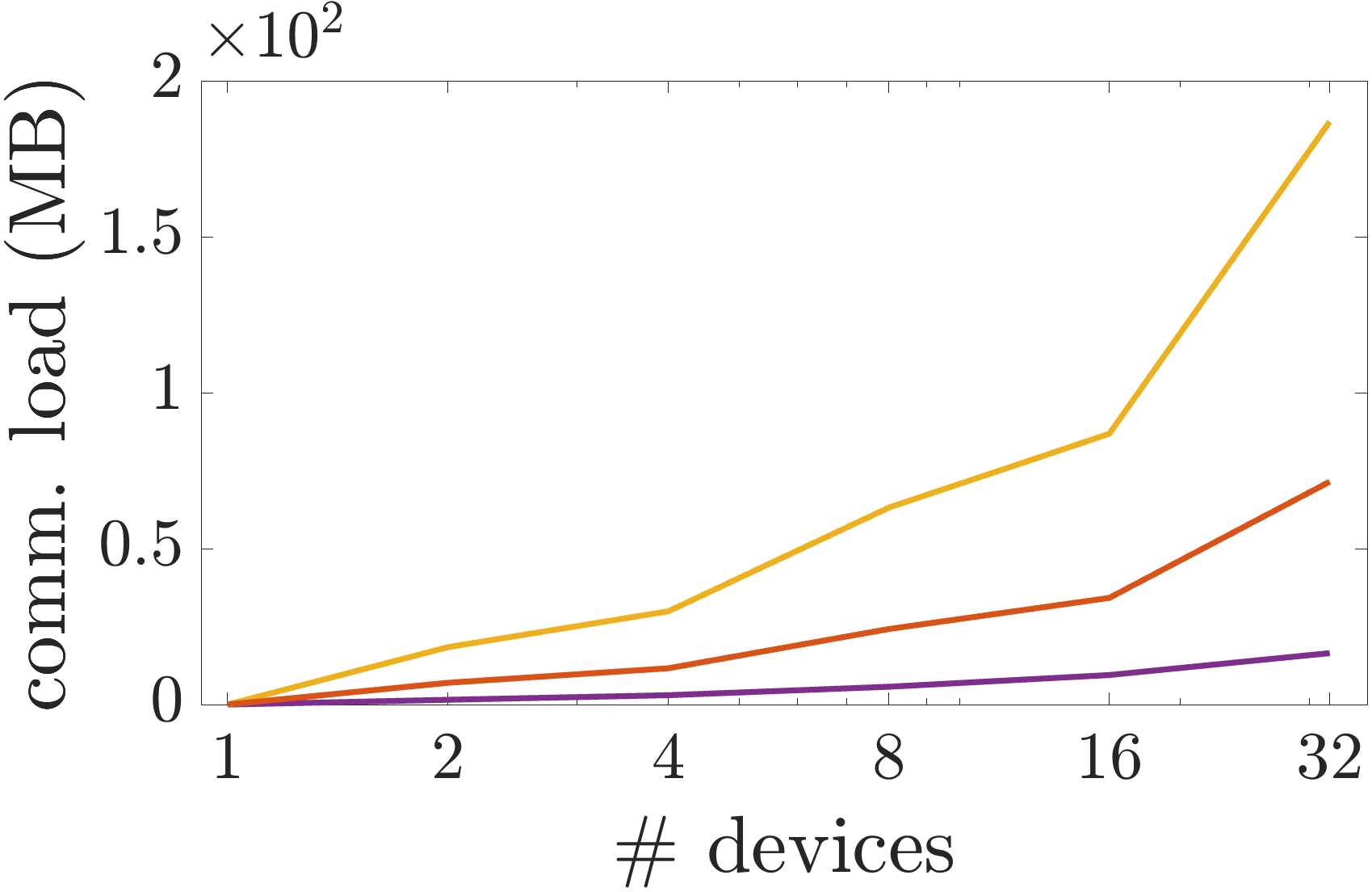}} &
		\hspace{-0.6em}\subfloat[][{\sf\small Piccadilly}]{\includegraphics[trim =0mm 0mm 0mm 0mm,width=0.2425\textwidth]{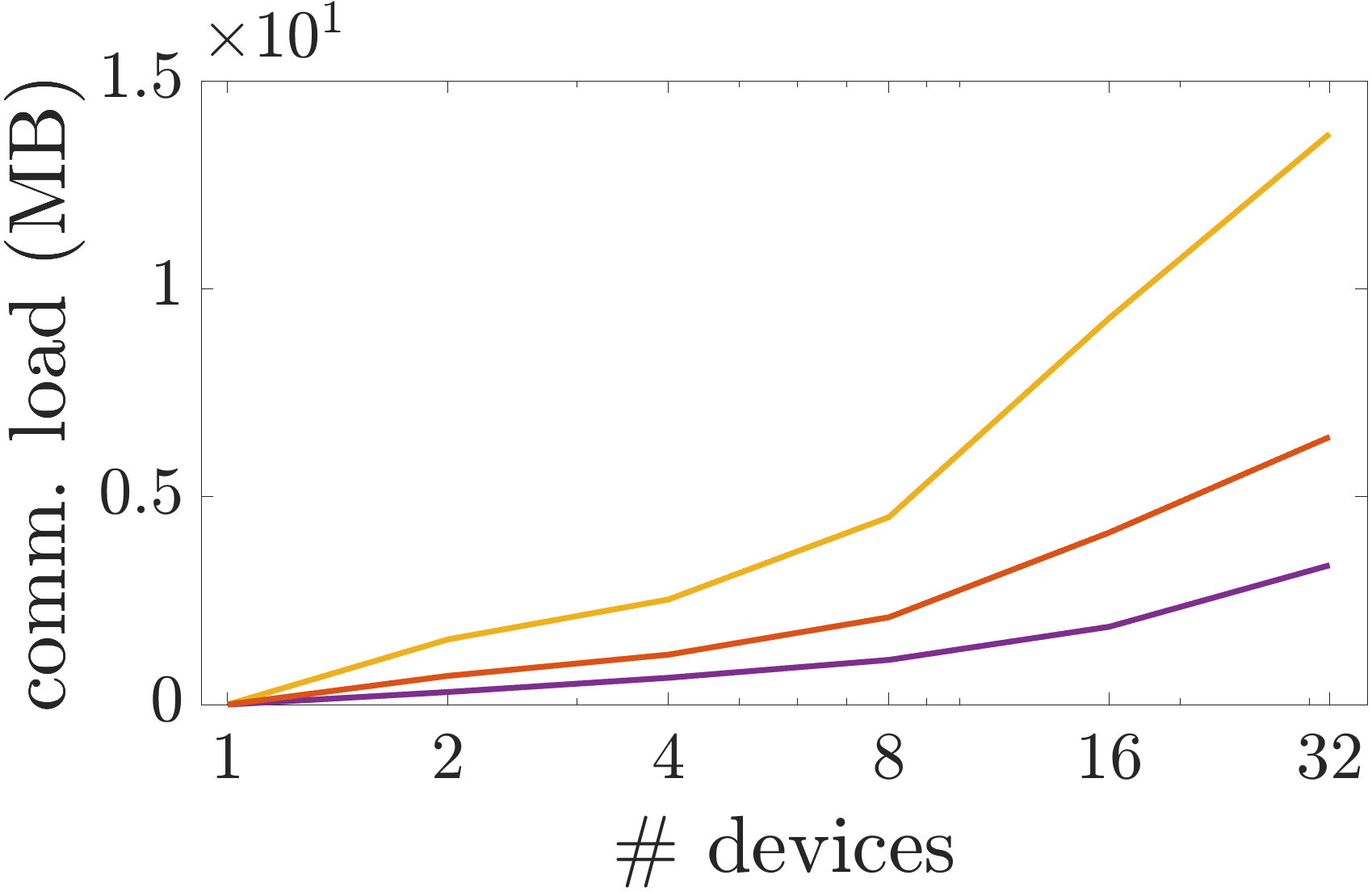}}&
		\hspace{-0.6em}\subfloat[][{\sf\small Trafalgar}]{\includegraphics[trim =0mm 0mm 0mm 0mm,width=0.2425\textwidth]{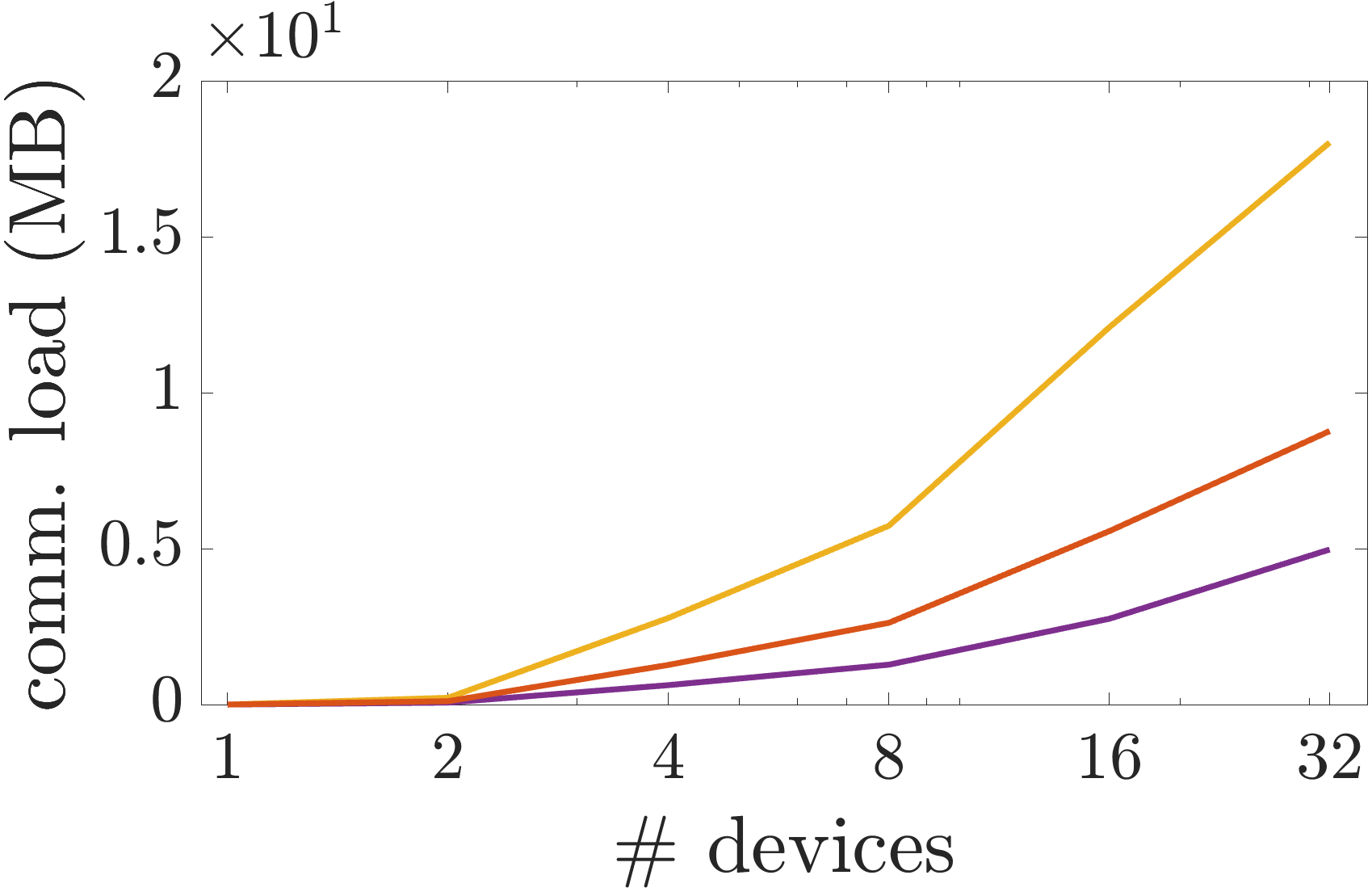}}
	\end{tabular}
	\caption{\textbf{Max memory usage per device} (a)-(d) and \textbf{total communication load per iteration} (e)-(h) for $\amm$ (ours) compared to decentralized methods on four largest datasets, {\sf\small Venice} and {\sf\small Final} in BAL, and {\sf\small Piccadilly} and {\sf\small Trafalgar} in 1DSfM.}\label{fig::mem_and_comm}
	\vspace{-1.5em}
\end{figure*}

\vspace{-0.1em}
\subsection{Memory \& Communication}

We evaluate the maximum memory usage per device and the total communication load per iteration for decentralized methods with 1, 2, 4, 8, 16, 32 devices on the four largest datasets in \cref{table::large_dataset}, i.e., {\sf\small Venice}, {\sf\small Final} in BAL  and {\sf\small Piccadilly}, {\sf\small Trafalgar} in 1DSfM.
In \cref{fig::mem_and_comm}(a) to \ref{fig::mem_and_comm}(d), we observe that, as expected, the maximum memory usage per device decreases as the number of devices increases, enabling scaling to large bundle adjustment problems.  Furthermore, $\amm$ takes almost the same memory as the others while being more efficient and accurate. 
In \cref{fig::mem_and_comm}(e) to \ref{fig::mem_and_comm}(h), we observe that the total communication load generally increases with the number of devices, also expected, since more devices require more communication. 
Since $\admm$  communicates merely cameras while $\amm$ exchanges points as well as cameras, $\amm$ has more communication load than $\admm$.
Despite the  larger communication load, the superior  efficiency and accuracy of $\amm$ compared to the other decentralized methods (see \cref{section::experiment::efficiency,section::experiment::accuracy}) makes it desirable for arbitrarily large-scale bundle adjustment.

\section{Conclusion}\label{section::conclusion}

We have presented $\amm$, a decentralized and accelerated method for large-scale bundle adjustment with provable convergence to first-order critical points. The key insight of our method is to decouple optimization variables and reduce bundle adjustment to independent subproblems with majorization minimization. This is achieved through the complete analysis of a novel reprojection error. We also implement Nesterov's acceleration for empirical speedup and adaptive restart for theoretical guarantees while maintaining decentralization. Compared to decentralized baselines \cite{ortiz2020bundle,eriksson2016consensus,zhang2017dist,demmel2020distributed}, our method has less strict assumptions for provable convergence, no need for specific parameter tuning, and more robustness to approximate solutions of subproblems. On extensive benchmarks with public datasets, our method has similar memory and communication overhead but outperforms decentralized baselines \cite{eriksson2016consensus,zhang2017dist} with a large margin in terms of accuracy and efficiency. As a result of multi-GPU implementation, our method, albeit decentralized, is more accurate on large-scale datasets with a speedup of up to 953.7x and 174.6x over centralized baselines $\ceres$ \cite{ceres-solver} and $\deeplm$ \cite{huang2021deeplm}, respectively. In the future, we are planning to theoretically relax the local minimum conditions of $\bfxakp$ in \cref{eq::update_amm,eq::update_mm}; extend our method for  bundle adjustment with lines and planes; and implement our method on multi-robot large-scale 3D reconstruction. 

\section*{Acknowledgments}
TM was supported by the National Science Foundation under award 1837515.

{\small
\bibliographystyle{IEEEtran}
\bibliography{mybib}
}

\clearpage

\begin{appendices}
\def\thesubsectiondis{\thesection.\arabic{subsection}.} 
\crefname{section}{App.}{Apps.}
\numberwithin{equation}{section}

\onecolumn

\section{ Evaluation}
We provide the complete evaluation and comparison results of our method $\amm$ (\cref{algorithm::amm})  on all the 20 datasets in BAL \cite{agarwal2010bundle} and 1DSfM \cite{wilson2014robust}; see  \cref{table::dataset}.

\begin{table}[H]
	\centering
	\setlength{\tabcolsep}{0.3em}
	\caption{Bundle adjustment datasets in BAL \cite{agarwal2010bundle} and 1DSfM \cite{wilson2014robust}.}\label{table::dataset}
	\begin{tabular}{P{0.015\textwidth}  P{0.085\textwidth} P{0.07\textwidth} P{0.075\textwidth}P{0.095\textwidth}}
		\toprule			
		\multicolumn{2}{c}{Dataset} & \# Cameras & \# Points &\# Observations  \\
		\cmidrule(lr){1-2} \cmidrule(lr){3-5}
		{\multirow{5}{*}{\rotatebox[origin=c]{90}{\scriptsize BAL \cite{agarwal2010bundle}}}}
		& {\scriptsize\sf Trafalgar} & 257 & 65132 & 225911  \\
		& {\scriptsize\sf Ladybug} & 1723 & 156502 & 678718  \\
		& {\scriptsize\sf Dubrovnik} & 356 & 226730 & 1255268  \\
		& {\scriptsize\sf Venice} & 1778 & 993923 & 5001946  \\
		& {\scriptsize\sf Final} & 13682 & 4456117 & 28987644 \\
		\cmidrule(lr){1-2} \cmidrule(lr){3-5}
		{\multirow{15}{*}{\rotatebox[origin=c]{90}{\scriptsize 1DSfM \cite{wilson2014robust}}}}
		& {\scriptsize\sf Alamo} & 571 & 151085 & 891301  \\
		& {\scriptsize\sf Ellis Island} & 234 & 29164 & 130903  \\
		& {\scriptsize\sf Gen. Markt} & 706 & 93672 & 364029  \\
		& {\scriptsize\sf M. Metropolis} & 346 & 55679 & 255987  \\
		& {\scriptsize\sf M. N. Dame} & 459 & 158005 & 860116 \\
		& {\scriptsize\sf N. Dame} & 547 & 273590 & 1534747 \\
		& {\scriptsize\sf NYC Library} & 338 & 74249 & 303955  \\
		& {\scriptsize\sf P. del Popolo} & 335 & 37609 & 195016  \\
		& {\scriptsize\sf Piccadilly} & 2289 & 209504 & 999878  \\
		& {\scriptsize\sf R. Forum} & 1063 & 265047 &  1292756  \\
		& {\scriptsize\sf T. of London} & 483 & 151328 & 797022  \\
		& {\scriptsize\sf Trafalgar} & 5032 & 388956 & 1826071  \\
		& {\scriptsize\sf U. Square} & 796 & 46066 & 230811  \\
		& {\scriptsize\sf V. Cathedral} & 836 & 265553 & 1333280 \\
		& {\scriptsize\sf Y. Minster} & 422 & 152591 & 701989  \\
		\bottomrule
	\end{tabular}
	\vspace{-0.5em}
\end{table}
\subsection{Accuracy}\label{section::app::accuracy}
We report  in \cref{table::all_accuracy} the mean reprojection errors with the trivial loss and Huber loss on all the  20 datasets in BAL \cite{agarwal2010bundle} and 1DSfM \cite{wilson2014robust} (see  \cref{table::dataset}).

\begin{table}[p]
	\renewcommand{\arraystretch}{1.02}
	\setlength{\tabcolsep}{0.625em}
	\centering
	\caption{
		\textbf{Mean reprojection errors} with the \textbf{Trivial loss} and \textbf{Huber loss} on all the  20 datasets in BAL \cite{agarwal2010bundle} and 1DSfM \cite{wilson2014robust} (see  \cref{table::dataset}).
		Decentralized methods $\dr$ \cite{eriksson2016consensus}, $\admm$ \cite{zhang2017dist}, $\amm$ (ours) are run for 1000 iterations with 4, 8, 16, 32 devices.
		Centralized methods $\ceres$ \cite{ceres-solver} and $\deeplm$ \cite{huang2021deeplm} are run for 40 iterations with single device as reference ($\deeplm$ does not support Huber loss).
		On each dataset (row), any decentralized method with best result is \textbf{bold}, and outperforming $\ceres$ and $\deeplm$ is {\color{red} red}.
		\textbf{$\amm$ (ours) achieves lowest reprojection error between decentralized methods and mostly outperforms centralized methods}.
	}
	\label{table::all_accuracy}
	\centering
	\setlength{\tabcolsep}{0.55em}
	\renewcommand{\arraystretch}{1.0}
	\label{table::comp_trivial_all}
	\begin{tabular}{ccccccccccccccccc}
		\toprule
		\multicolumn{2}{c}{}& \multicolumn{15}{c}{Mean Reprojection Error with the Trivial Loss}\\
		\cmidrule(lr){3-17}
		\multicolumn{2}{c}{Dataset} &\multirow{2}{*}{Init} & \multirow{2}{*}{$\ceres$} & \multirow{2}{*}{$\deeplm$} &  \multicolumn{3}{c}{4 Devices} & \multicolumn{3}{c}{8 Devices} & \multicolumn{3}{c}{16 Devices} & \multicolumn{3}{c}{32 Devices}\\
		\cmidrule(lr){6-8} \cmidrule(lr){9-11} \cmidrule(lr){12-14} \cmidrule(lr){15-17}
		\multicolumn{2}{c}{} & & & & {\,\scriptsize $\dr$\,} & {\scriptsize $\admm$} & {\scriptsize $\amm$} & {\,\scriptsize $\dr$\,} & {\scriptsize $\admm$} & {\scriptsize $\amm$} & {\,\scriptsize $\dr$\,} & {\scriptsize $\admm$}  & {\scriptsize $\amm$}  & {\,\scriptsize $\dr$\,} & {\scriptsize $\admm$} &  {\scriptsize $\amm$}\\
		\cmidrule(lr){1-2} \cmidrule(lr){3-17}
		{\multirow{5}{*}{\rotatebox[origin=c]{90}{\scriptsize BAL \cite{agarwal2010bundle}}}} 
		& {\scriptsize\sf Trafalgar} & {1.527} & {0.452} &  {0.453} &{0.493} & {0.487} & \textbf{0.453} & {0.494} & {0.491} & \textbf{0.455} & {0.496} & {0.492} & \textbf{0.455} & {0.497} & {0.501} & \textbf{0.458} \\
		& {\scriptsize\sf Ladybug} & {10.48} & {0.707} &  {0.710} &{0.837} & {\color{red}0.698} & \textbf{\color{red}0.690} & {0.846} & {\color{red}0.703} & \textbf{\color{red}0.690} & {0.850} & {0.711} & \textbf{\color{red}0.690} & {0.859} & {0.723} & \textbf{\color{red}0.690} \\
		& {\scriptsize\sf Dubrovnik} & {3.765} & {0.423} &  {0.423} &{1.998} & {0.473} & \textbf{\color{red}0.423} & {2.038} & {0.484} & \textbf{\color{red}0.423} & {2.051} & {0.490} & \textbf{0.424} & {2.067} & {0.496} & \textbf{0.424} \\
		& {\scriptsize\sf Venice} & {26.33} & {0.468} &  {0.466} &{0.515} & {0.516} & \textbf{\color{red}0.465} & {0.515} & {0.525} & \textbf{\color{red}0.465} & {0.516} & {0.532} & \textbf{\color{red}0.466} & {0.517} & {0.544} & \textbf{0.473} \\
		& {\scriptsize\sf Final} & {12.57} & {0.855} &  {0.848} &{2.017} & {0.876} & \textbf{\color{red}0.828} & {2.042} & {0.903} & \textbf{\color{red}0.828} & {2.053} & {0.908} & \textbf{\color{red}0.829} & {2.123} & {0.921} & \textbf{\color{red}0.833} \\
        \cmidrule(lr){1-2} \cmidrule(lr){3-17}
		{\multirow{15}{*}{\rotatebox[origin=c]{90}{\scriptsize 1DSfM \cite{wilson2014robust}}}} 
		& {\scriptsize\sf Alamo} & {3.616} & {1.152} &  {1.203} &{1.540} & {1.347} & \textbf{\color{red}1.135} & {1.550} & {1.325} & \textbf{\color{red}1.136} & {1.575} & {1.438} & \textbf{\color{red}1.138} & {1.575} & {1.445} & \textbf{\color{red}1.137} \\
		& {\scriptsize\sf Ellis Island} & {11.98} & {5.061} &  {5.022} &{6.862} & {6.015} & \textbf{\color{red}5.016} & {6.926} & {6.469} & \textbf{\color{red}4.973} & {6.998} & {6.421} & \textbf{5.309} & {7.058} & {6.375} & \textbf{5.043} \\
		& {\scriptsize\sf Gen. Markt} & {8.181} & {4.024} &  {4.028} &{4.835} & {4.390} & \textbf{\color{red}3.977} & {4.870} & {4.616} & \textbf{\color{red}3.986} & {4.900} & {4.646} & \textbf{\color{red}3.988} & {4.942} & {4.688} & \textbf{\color{red}4.017} \\
		& {\scriptsize\sf M. Metro.} & {4.599} & {2.424} &  {2.378} &{2.784} & {2.522} & \textbf{\color{red}2.376} & {2.797} & {2.554} & \textbf{2.382} & {2.808} & {2.609} & \textbf{2.380} & {2.818} & {2.695} & \textbf{2.387} \\
		& {\scriptsize\sf M. N. Dame} & {7.335} & {3.704} &  {3.754} &{4.421} & {3.763} & \textbf{\color{red}3.697} & {4.452} & {3.802} & \textbf{\color{red}3.702} & {4.471} & {3.858} & \textbf{\color{red}3.699} & {4.492} & {3.963} & \textbf{\color{red}3.700} \\
		& {\scriptsize\sf N. Dame} & {9.442} & {3.312} &  {3.312} &{3.712} & {3.526} & \textbf{3.316} & {3.741} & {3.557} & \textbf{\color{red}3.312} & {3.762} & {3.618} & \textbf{\color{red}3.307} & {3.785} & {3.693} & \textbf{3.323} \\
		& {\scriptsize\sf NYC Library} & {5.140} & {1.887} &  {1.888} &{2.405} & {1.997} & \textbf{\color{red}1.863} & {2.412} & {2.019} & \textbf{\color{red}1.840} & {2.444} & {2.089} & \textbf{\color{red}1.849} & {2.461} & {2.132} & \textbf{\color{red}1.861} \\
		& {\scriptsize\sf P. del Popolo} & {6.552} & {2.438} &  {2.377} &{3.014} & {2.544} & \textbf{\color{red}2.336} & {3.037} & {2.616} & \textbf{\color{red}2.353} & {3.049} & {2.827} & \textbf{2.433} & {3.065} & {2.907} & \textbf{\color{red}2.372} \\
		& {\scriptsize\sf Piccadily} & {11.26} & {4.379} &  {4.424} &{6.040} & {4.634} & \textbf{\color{red}4.218} & {6.073} & {4.761} & \textbf{\color{red}4.217} & {6.148} & {4.850} & \textbf{\color{red}4.249} & {6.193} & {5.027} & \textbf{\color{red}4.312} \\
		& {\scriptsize\sf R. Forum} & {6.407} & {1.211} &  {1.177} &{2.194} & {1.381} & \textbf{\color{red}1.152} & {2.211} & {1.446} & \textbf{\color{red}1.161} & {2.232} & {1.516} & \textbf{\color{red}1.171} & {2.251} & {1.585} & \textbf{\color{red}1.168} \\
		& {\scriptsize\sf T. of London} & {4.399} & {0.668} &  {0.667} &{1.169} & {0.791} & \textbf{\color{red}0.656} & {1.184} & {0.789} & \textbf{\color{red}0.653} & {1.201} & {0.845} & \textbf{\color{red}0.659} & {1.226} & {0.889} & \textbf{\color{red}0.665} \\
		& {\scriptsize\sf Trafalgar} & {10.77} & {3.702} &  {3.886} &{5.065} & {3.994} & \textbf{\color{red}3.604} & {5.263} & {4.111} & \textbf{\color{red}3.601} & {5.169} & {4.185} & \textbf{\color{red}3.624} & {5.193} & {4.395} & \textbf{\color{red}3.657} \\
		& {\scriptsize\sf U. Square} & {10.56} & {3.992} &  {3.977} &{5.376} & {4.233} & \textbf{\color{red}3.893} & {5.426} & {4.401} & \textbf{\color{red}3.888} & {5.485} & {4.519} & \textbf{\color{red}3.925} & {5.538} & {4.619} & \textbf{\color{red}3.975} \\
		& {\scriptsize\sf V. Cathedral} & {9.506} & {1.957} &  {1.959} &{3.119} & {2.068} & \textbf{\color{red}1.777} & {3.139} & {2.021} & \textbf{\color{red}1.770} & {3.273} & {2.057} & \textbf{\color{red}1.831} & {3.328} & {2.216} & \textbf{\color{red}1.834} \\
		& {\scriptsize\sf Y. Minster} & {8.929} & {2.061} &  {1.986} &{3.242} & {2.171} & \textbf{\color{red}1.930} & {3.301} & {2.411} & \textbf{\color{red}1.923} & {3.352} & {2.242} & \textbf{\color{red}1.952} & {3.398} & {2.334} & \textbf{\color{red}1.933} \\
		\bottomrule
	\end{tabular}
	
	\vfill
	\vspace{5mm}
	
	\begin{tabular}{ccccccccccccccccc}
		\toprule
		\multicolumn{2}{c}{}& \multicolumn{15}{c}{Mean Reprojection Error with the Huber Loss}\\
		\cmidrule(lr){3-17}
		\multicolumn{2}{c}{Dataset} &\multirow{2}{*}{Init} & \multirow{2}{*}{$\ceres$} & \multirow{2}{*}{$\deeplm$} &  \multicolumn{3}{c}{4 Devices} & \multicolumn{3}{c}{8 Devices} & \multicolumn{3}{c}{16 Devices} & \multicolumn{3}{c}{32 Devices}\\
		\cmidrule(lr){6-8} \cmidrule(lr){9-11} \cmidrule(lr){12-14} \cmidrule(lr){15-17}
		\multicolumn{2}{c}{} & & & & {\,\scriptsize $\dr$\,} & {\scriptsize $\admm$} & {\scriptsize $\amm$} & {\,\scriptsize $\dr$\,} & {\scriptsize $\admm$} & {\scriptsize $\amm$} & {\,\scriptsize $\dr$\,} & {\scriptsize $\admm$}  & {\scriptsize $\amm$}  & {\,\scriptsize $\dr$\,} & {\scriptsize $\admm$} &  {\scriptsize $\amm$}\\
		\cmidrule(lr){1-2} \cmidrule(lr){3-17}
		{\multirow{5}{*}{\rotatebox[origin=c]{90}{\scriptsize BAL \cite{agarwal2010bundle}}}} 
		& {\scriptsize\sf Trafalgar} & {1.003} & {0.452}  & - & {0.493} & {0.487} & \textbf{0.454} & {0.494} & {0.491} & \textbf{0.455} & {0.496} & {0.492} & \textbf{0.455} & {0.497} & {0.501} & \textbf{0.457} \\
		& {\scriptsize\sf Ladybug} & {3.267} & {0.704}  & - & {0.834} & {\color{red}0.698} & \textbf{\color{red}0.690} & {0.841} & {\color{red}0.703} & \textbf{\color{red}0.690} & {0.844} & {0.712} & \textbf{\color{red}0.690} & {0.848} & {0.723} & \textbf{\color{red}0.690} \\
		& {\scriptsize\sf Dubrovnik} & {3.721} & {0.423}  & - & {1.996} & {0.473} & \textbf{\color{red}0.423} & {2.036} & {0.483} & \textbf{\color{red}0.423} & {2.050} & {0.490} & \textbf{0.424} & {2.065} & {0.496} & \textbf{0.424} \\
		& {\scriptsize\sf Venice} & {4.750} & {0.468}  & - & {0.515} & {0.516} & \textbf{\color{red}0.465} & {0.515} & {0.524} & \textbf{\color{red}0.465} & {0.516} & {0.530} & \textbf{\color{red}0.465} & {0.517} & {0.544} & \textbf{0.473} \\
		& {\scriptsize\sf Final} & {7.573} & {0.815}  & - & {1.995} & {0.859} & \textbf{\color{red}0.796} & {2.020} & {0.882} & \textbf{\color{red}0.795} & {2.032} & {0.886} & \textbf{\color{red}0.796} & {2.101} & {0.898} & \textbf{\color{red}0.815} \\
		\cmidrule(lr){1-2} \cmidrule(lr){3-17}
		{\multirow{15}{*}{\rotatebox[origin=c]{90}{\scriptsize 1DSfM \cite{wilson2014robust}}}} 
		& {\scriptsize\sf Alamo} & {3.417} & {1.127}  & - & {1.436} & {1.259} & \textbf{\color{red}1.050} & {1.446} & {1.238} & \textbf{\color{red}1.052} & {1.468} & {1.348} & \textbf{\color{red}1.052} & {1.471} & {1.355} & \textbf{\color{red}1.053} \\
		& {\scriptsize\sf Ellis Island} & {11.75} & {4.906}  & - & {6.751} & {5.823} & \textbf{\color{red}4.866} & {6.802} & {5.806} & \textbf{\color{red}4.832} & {6.881} & {6.107} & \textbf{5.191} & {6.939} & {6.152} & \textbf{\color{red}4.895} \\
		& {\scriptsize\sf Gen. Markt} & {8.031} & {3.942}  & - & {4.777} & {4.320} & \textbf{\color{red}3.893} & {4.812} & {4.525} & \textbf{\color{red}3.914} & {4.842} & {4.561} & \textbf{\color{red}3.925} & {4.886} & {4.598} & \textbf{\color{red}3.942} \\
		& {\scriptsize\sf M. Metro.} & {4.445} & {2.344}  & - & {2.711} & {2.444} & \textbf{\color{red}2.300} & {2.725} & {2.477} & \textbf{\color{red}2.311} & {2.734} & {2.534} & \textbf{\color{red}2.308} & {2.744} & {2.620} & \textbf{\color{red}2.315} \\
		& {\scriptsize\sf M. N. Dame} & {6.961} & {3.527}  & - & {4.216} & {3.572} & \textbf{\color{red}3.506} & {4.243} & {3.611} & \textbf{\color{red}3.510} & {4.261} & {3.666} & \textbf{\color{red}3.507} & {4.279} & {3.769} & \textbf{\color{red}3.508} \\
		& {\scriptsize\sf N. Dame} & {8.810} & {2.994}  & - & {3.395} & {3.209} & \textbf{\color{red}2.992} & {3.424} & {3.238} & \textbf{\color{red}2.984} & {3.446} & {3.301} & \textbf{\color{red}2.984} & {3.469} & {3.381} & \textbf{\color{red}2.985} \\
		& {\scriptsize\sf NYC Library} & {5.037} & {1.820}  & - & {2.354} & {1.943} & \textbf{\color{red}1.802} & {2.363} & {1.968} & \textbf{\color{red}1.795} & {2.393} & {2.036} & \textbf{\color{red}1.806} & {2.410} & {2.076} & \textbf{\color{red}1.820} \\
		& {\scriptsize\sf P. del Popolo} & {6.437} & {2.319}  & - & {2.942} & {2.485} & \textbf{\color{red}2.283} & {2.965} & {2.550} & \textbf{\color{red}2.297} & {2.974} & {2.763} & \textbf{2.383} & {2.990} & {2.815} & \textbf{\color{red}2.317} \\
		& {\scriptsize\sf Piccadily} & {10.34} & {4.136}  & - & {5.739} & {4.334} & \textbf{\color{red}3.951} & {5.776} & {4.441} & \textbf{\color{red}3.963} & {5.851} & {4.564} & \textbf{\color{red}3.961} & {5.892} & {4.689} & \textbf{\color{red}3.978} \\
		& {\scriptsize\sf R. Forum} & {6.247} & {1.136}  & - & {2.149} & {1.296} & \textbf{\color{red}1.120} & {2.163} & {1.406} & \textbf{\color{red}1.127} & {2.187} & {1.471} & \textbf{\color{red}1.133} & {2.204} & {1.516} & \textbf{\color{red}1.133} \\
		& {\scriptsize\sf T. of London} & {4.374} & {0.679}  & - & {1.158} & {0.785} & \textbf{\color{red}0.645} & {1.172} & {0.782} & \textbf{\color{red}0.642} & {1.189} & {0.838} & \textbf{\color{red}0.649} & {1.214} & {0.881} & \textbf{\color{red}0.655} \\
		& {\scriptsize\sf Trafalgar} & {9.900} & {3.544}  & - & {4.816} & {3.722} & \textbf{\color{red}3.363} & {4.939} & {3.876} & \textbf{\color{red}3.377} & {5.000} & {3.952} & \textbf{\color{red}3.401} & {4.944} & {4.149} & \textbf{\color{red}3.438} \\
		& {\scriptsize\sf U. Square} & {10.18} & {3.987}  & - & {5.242} & {4.096} & \textbf{\color{red}3.755} & {5.291} & {4.261} & \textbf{\color{red}3.763} & {5.348} & {4.366} & \textbf{\color{red}3.763} & {5.400} & {4.470} & \textbf{\color{red}3.818} \\
		& {\scriptsize\sf V. Cathedral} & {9.085} & {1.748}  & - & {3.008} & {1.928} & \textbf{\color{red}1.709} & {3.027} & {1.923} & \textbf{\color{red}1.681} & {3.156} & {1.956} & \textbf{\color{red}1.740} & {3.206} & {2.088} & \textbf{\color{red}1.734} \\
		& {\scriptsize\sf Y. Minster} & {8.586} & {1.910}  & - & {3.145} & {2.038} & \textbf{\color{red}1.876} & {3.203} & {2.272} & \textbf{\color{red}1.887} & {3.248} & {2.104} & \textbf{\color{red}1.875} & {3.294} & {2.153} & \textbf{\color{red}1.866} \\
		
		\bottomrule
	\end{tabular}
\end{table}	

\subsection{Efficiency}
We report the optimization time of our method $\amm$ (\cref{algorithm::amm}) with 4 and 8 devices and centralized  methods $\ceres$\cite{ceres-solver} and $\deeplm$\cite{huang2021deeplm} to attain the reference/target mean reprojection errors $F_{\text{ref}}$ and $F_\Delta(p)$ on the largest datasets of more than 700 cameras in \cref{table::large_dataset}. Each device uses one GPU and we consider together both the computation and communication time. We choose $F_{\text{ref}}$ to be the smallest objective value separately achieved by $\ceres$ and $\deeplm$ for 40 iterations and $\Delta=2\times 10^{-4}$ to compute $F_{\Delta}(p)$ in \cref{eq::Fdelta}. The results are in \cref{table:ceres_h,table:ceres_t,table:lm_t}.

\begin{table}
	\centering
		\renewcommand{\arraystretch}{1.0}
	\caption{
		 \textbf{Optimization time} with the \textbf{trivial loss}  of our method $\amm$ (\cref{algorithm::amm}) and centralized  method $\ceres$\cite{ceres-solver} to attain the reference/target mean reprojection errors $F_{\text{ref}}$ and $F_\Delta(p)$ on the largest datasets of more than 700 cameras in \cref{table::large_dataset}. The reference reprojection errors are from $\ceres$ and $\Delta=2.5\times10^{-4}$. 
	}\label{table:ceres_t}
	\begin{tabular}{cccccccccc}
		\toprule
		 & & \multicolumn{2}{c}{Error} &  \multicolumn{6}{c}{Time (seconds)}\\
		\cmidrule(lr){3-4} \cmidrule(lr){5-10}
		 \multicolumn{2}{c}{Dataset} & \multirow{2}{*}{$F_{\Delta}(p)$} & \multirow{2}{*}{$F_{\text{ref}}$} &  \multicolumn{2}{c}{$\ceres$}  &  \multicolumn{2}{c}{$\amm$ with 4 Devices} &  \multicolumn{2}{c}{$\amm$ with 8 Devices} \\ 
		 \cmidrule(lr){5-6} \cmidrule(lr){7-8} \cmidrule(lr){9-10} 
		 & & & & $T_\Delta(p)$ & $T_{\text{ref}}$ & $T_\Delta(p)$ & $T_{\text{ref}}$ & $T_\Delta(p)$ & $T_{\text{ref}}$ \\
		\cmidrule(lr){1-2} \cmidrule(lr){3-4} \cmidrule(lr){5-10}
		{\multirow{3}{*}{\rotatebox[origin=c]{90}{\scriptsize BAL \cite{agarwal2010bundle}}}}
		& {\scriptsize\sf Ladybug}  & {0.710} &  {0.707} & $4.72 \times 10^{1}$ & $6.78 \times 10^{1}$ & $7.10 \times 10^{-2}$ & $8.22 \times 10^{-2}$ & $6.93 \times 10^{-2}$ & $8.03 \times 10^{-2}$ \\
		& {\scriptsize\sf Venice}  & {0.474} &  {0.468} & $2.92 \times 10^{2}$ & $1.47 \times 10^{3}$ & $4.70 \times 10^{0}$ & $1.22 \times 10^{1}$ & $2.81 \times 10^{0}$ & $6.65 \times 10^{0}$ \\
		& {\scriptsize\sf Final}  & {0.858} &  {0.855} & $2.94 \times 10^{3}$ & $3.30 \times 10^{3}$ & $1.48 \times 10^{1}$ & $1.59 \times 10^{1}$ & $1.07 \times 10^{1}$ & $1.14 \times 10^{1}$ \\
		\cmidrule(lr){1-2} \cmidrule(lr){3-4} \cmidrule(lr){5-10}
		{\multirow{6}{*}{\rotatebox[origin=c]{90}{\scriptsize 1DSfM \cite{wilson2014robust}}}}
		& {\scriptsize\sf Gen. Markt}  & {4.025} &  {4.024} & $5.41 \times 10^{1}$ & $5.41 \times 10^{1}$ & $5.36 \times 10^{-1}$ & $5.40 \times 10^{-1}$ & $8.17 \times 10^{-1}$ & $8.27 \times 10^{-1}$ \\
		& {\scriptsize\sf Piccadily}  & {4.380} &  {4.379} & $2.16 \times 10^{2}$ & $2.16 \times 10^{2}$ & $1.83 \times 10^{0}$ & $1.84 \times 10^{0}$ & $1.26 \times 10^{0}$ & $1.27 \times 10^{0}$ \\
		& {\scriptsize\sf R. Forum}  & {1.212} &  {1.211} & $2.87 \times 10^{2}$ & $3.02 \times 10^{2}$ & $2.05 \times 10^{0}$ & $2.08 \times 10^{0}$ & $1.55 \times 10^{0}$ & $1.58 \times 10^{0}$ \\
		& {\scriptsize\sf Trafalgar}  & {3.703} &  {3.702} & $5.13 \times 10^{2}$ & $5.13 \times 10^{2}$ & $3.41 \times 10^{0}$ & $3.46 \times 10^{0}$ & $2.11 \times 10^{0}$ & $2.14 \times 10^{0}$ \\
		& {\scriptsize\sf U. Square}  & {3.994} &  {3.992} & $5.34 \times 10^{1}$ & $5.34 \times 10^{1}$ & $8.90 \times 10^{-1}$ & $9.01 \times 10^{-1}$ & $1.06 \times 10^{0}$ & $1.07 \times 10^{0}$ \\
		& {\scriptsize\sf V. Cathedral}  & {1.959} &  {1.957} & $2.96 \times 10^{2}$ & $3.16 \times 10^{2}$ & $9.50 \times 10^{-1}$ & $9.60 \times 10^{-1}$ & $6.15 \times 10^{-1}$ & $6.20 \times 10^{-1}$ \\
		\bottomrule
	\end{tabular}
\end{table}

\begin{table}
	\centering
	\renewcommand{\arraystretch}{1.0}
	\caption{
		\textbf{Optimization time} with the \textbf{trivial loss}  of our method $\amm$ (\cref{algorithm::amm}) and centralized  method $\deeplm$\cite{huang2021deeplm} to attain the reference/target mean reprojection errors $F_{\text{ref}}$ and $F_\Delta(p)$ on the largest datasets of more than 700 cameras in \cref{table::large_dataset}. The reference reprojection errors are from $\deeplm$ and $\Delta=2.5\times10^{-4}$. 
	}\label{table:lm_t}
	\begin{tabular}{cccccccccc}
		\toprule
		& & \multicolumn{2}{c}{Error} &  \multicolumn{6}{c}{Time (seconds)}\\
		\cmidrule(lr){3-4} \cmidrule(lr){5-10}
		\multicolumn{2}{c}{Dataset} & \multirow{2}{*}{$F_{\Delta}(p)$} & \multirow{2}{*}{$F_{\text{ref}}$} &  \multicolumn{2}{c}{$\deeplm$}  &  \multicolumn{2}{c}{$\amm$ with 4 Devices} &  \multicolumn{2}{c}{$\amm$ with 8 Devices} \\ 
		\cmidrule(lr){5-6} \cmidrule(lr){7-8} \cmidrule(lr){9-10} 
		& & & & $T_\Delta(p)$ & $T_{\text{ref}}$ & $T_\Delta(p)$ & $T_{\text{ref}}$ & $T_\Delta(p)$ & $T_{\text{ref}}$ \\
		\cmidrule(lr){1-2} \cmidrule(lr){3-4} \cmidrule(lr){5-10}
		{\multirow{3}{*}{\rotatebox[origin=c]{90}{\scriptsize BAL \cite{agarwal2010bundle}}}}
		& {\scriptsize\sf Ladybug}  & {0.712} &  {0.710} & $1.14 \times 10^{1}$ & $1.76 \times 10^{1}$ & $5.47 \times 10^{-2}$ & $6.56 \times 10^{-2}$ & $6.20 \times 10^{-2}$ & $6.93 \times 10^{-2}$ \\
		& {\scriptsize\sf Venice}  & {0.473} &  {0.466} & $4.07 \times 10^{1}$ & $1.42 \times 10^{2}$ & $6.15 \times 10^{0}$ & $1.56 \times 10^{1}$ & $3.61 \times 10^{0}$ & $8.50 \times 10^{0}$ \\
		& {\scriptsize\sf Final}  & {0.851} &  {0.848} & $4.53 \times 10^{2}$ & $5.21 \times 10^{2}$ & $1.83 \times 10^{1}$ & $2.00 \times 10^{1}$ & $1.27 \times 10^{1}$ & $1.38 \times 10^{1}$ \\
		\cmidrule(lr){1-2} \cmidrule(lr){3-4} \cmidrule(lr){5-10}
		{\multirow{6}{*}{\rotatebox[origin=c]{90}{\scriptsize 1DSfM \cite{wilson2014robust}}}}
		& {\scriptsize\sf Gen. Markt}  & {4.029} &  {4.028} & $1.27 \times 10^{1}$ & $1.27 \times 10^{1}$ & $5.16 \times 10^{-1}$ & $5.19 \times 10^{-1}$ & $7.74 \times 10^{-1}$ & $7.87 \times 10^{-1}$ \\
		& {\scriptsize\sf Piccadily}  & {4.425} &  {4.424} & $2.53 \times 10^{1}$ & $2.60 \times 10^{1}$ & $1.38 \times 10^{0}$ & $1.40 \times 10^{0}$ & $1.06 \times 10^{0}$ & $1.06 \times 10^{0}$ \\
		& {\scriptsize\sf R. Forum}  & {1.178} &  {1.177} & $3.39 \times 10^{1}$ & $3.48 \times 10^{1}$ & $2.97 \times 10^{0}$ & $3.07 \times 10^{0}$ & $2.91 \times 10^{0}$ & $3.02 \times 10^{0}$ \\
		& {\scriptsize\sf Trafalgar}  & {3.888} &  {3.886} & $4.07 \times 10^{1}$ & $4.07 \times 10^{1}$ & $1.35 \times 10^{0}$ & $1.37 \times 10^{0}$ & $9.67 \times 10^{-1}$ & $9.76 \times 10^{-1}$ \\
		& {\scriptsize\sf U. Square}  & {3.978} &  {3.977} & $1.02 \times 10^{1}$ & $1.04 \times 10^{1}$ & $1.03 \times 10^{0}$ & $1.05 \times 10^{0}$ & $1.16 \times 10^{0}$ & $1.18 \times 10^{0}$ \\
		& {\scriptsize\sf V. Cathedral}  & {1.961} &  {1.959} & $3.10 \times 10^{1}$ & $3.19 \times 10^{1}$ & $9.40 \times 10^{-1}$ & $9.50 \times 10^{-1}$ & $6.02 \times 10^{-1}$ & $6.09 \times 10^{-1}$ \\
		\bottomrule
	\end{tabular}
\end{table}

\begin{table}
	\centering
	\renewcommand{\arraystretch}{1.0}
	\caption{
		\textbf{Optimization time} with the \textbf{Huber loss}  of our method $\amm$ (\cref{algorithm::amm}) and centralized  method $\ceres$\cite{ceres-solver} to attain the reference/target mean reprojection errors $F_{\text{ref}}$ and $F_\Delta(p)$ on the largest datasets of more than 700 cameras in \cref{table::large_dataset}. The reference reprojection errors are from $\ceres$ and $\Delta=2.5\times10^{-4}$. 
	}\label{table:ceres_h}
	\begin{tabular}{cccccccccc}
		\toprule
		& & \multicolumn{2}{c}{Error} &  \multicolumn{6}{c}{Time (seconds)}\\
		\cmidrule(lr){3-4} \cmidrule(lr){5-10}
		\multicolumn{2}{c}{Dataset} & \multirow{2}{*}{$F_{\Delta}(p)$} & \multirow{2}{*}{$F_{\text{ref}}$} &  \multicolumn{2}{c}{$\ceres$}  &  \multicolumn{2}{c}{$\amm$ with 4 Devices} &  \multicolumn{2}{c}{$\amm$ with 8 Devices} \\ 
		\cmidrule(lr){5-6} \cmidrule(lr){7-8} \cmidrule(lr){9-10} 
		& & & & $T_\Delta(p)$ & $T_{\text{ref}}$ & $T_\Delta(p)$ & $T_{\text{ref}}$ & $T_\Delta(p)$ & $T_{\text{ref}}$ \\
		\cmidrule(lr){1-2} \cmidrule(lr){3-4} \cmidrule(lr){5-10}
		{\multirow{3}{*}{\rotatebox[origin=c]{90}{\scriptsize BAL \cite{agarwal2010bundle}}}}
		& {\scriptsize\sf Ladybug}  & {0.705} &  {0.704} & $7.57 \times 10^{1}$ & $1.67 \times 10^{2}$ & $8.77 \times 10^{-2}$ & $9.33 \times 10^{-2}$ & $7.93 \times 10^{-2}$ & $7.93 \times 10^{-2}$ \\
		& {\scriptsize\sf Venice}  & {0.469} &  {0.468} & $1.07 \times 10^{3}$ & $1.31 \times 10^{3}$ & $1.05 \times 10^{1}$ & $1.20 \times 10^{1}$ & $6.41 \times 10^{0}$ & $7.19 \times 10^{0}$ \\
		& {\scriptsize\sf Final}  & {0.817} &  {0.815} & $6.59 \times 10^{3}$ & $7.69 \times 10^{3}$ & $3.46 \times 10^{1}$ & $3.80 \times 10^{1}$ & $2.40 \times 10^{1}$ & $2.63 \times 10^{1}$ \\
		\cmidrule(lr){1-2} \cmidrule(lr){3-4} \cmidrule(lr){5-10}
		{\multirow{6}{*}{\rotatebox[origin=c]{90}{\scriptsize 1DSfM \cite{wilson2014robust}}}}
		& {\scriptsize\sf Gen. Markt}  & {3.943} &  {3.942} & $5.96 \times 10^{1}$ & $5.96 \times 10^{1}$ & $5.81 \times 10^{-1}$ & $5.84 \times 10^{-1}$ & $9.08 \times 10^{-1}$ & $9.18 \times 10^{-1}$ \\
		& {\scriptsize\sf Piccadily}  & {4.138} &  {4.136} & $1.80 \times 10^{2}$ & $1.80 \times 10^{2}$ & $1.29 \times 10^{0}$ & $1.30 \times 10^{0}$ & $1.15 \times 10^{0}$ & $1.16 \times 10^{0}$ \\
		& {\scriptsize\sf R. Forum}  & {1.137} &  {1.136} & $3.27 \times 10^{2}$ & $3.32 \times 10^{2}$ & $4.39 \times 10^{0}$ & $4.69 \times 10^{0}$ & $3.25 \times 10^{0}$ & $3.37 \times 10^{0}$ \\
		& {\scriptsize\sf Trafalgar}  & {3.546} &  {3.544} & $3.74 \times 10^{2}$ & $3.74 \times 10^{2}$ & $2.13 \times 10^{0}$ & $2.15 \times 10^{0}$ & $1.63 \times 10^{0}$ & $1.64 \times 10^{0}$ \\
		& {\scriptsize\sf U. Square}  & {3.988} &  {3.987} & $7.20 \times 10^{1}$ & $7.20 \times 10^{1}$ & $4.40 \times 10^{-1}$ & $4.44 \times 10^{-1}$ & $6.12 \times 10^{-1}$ & $6.15 \times 10^{-1}$ \\
		& {\scriptsize\sf V. Cathedral}  & {1.750} &  {1.748} & $5.15 \times 10^{2}$ & $5.15 \times 10^{2}$ & $2.86 \times 10^{0}$ & $2.91 \times 10^{0}$ & $1.21 \times 10^{0}$ & $1.23 \times 10^{0}$ \\
		\bottomrule
	\end{tabular}
\end{table}

\section{assumptions}\label{section::app::assumptions}
We summarize  \cref{assumption::loss,assumption::local_opt,assumption:nonzero,assumption::bounded_intr} made in this paper, where  \cref{assumption::loss} applies to a broad class of robust loss functions like Huber and Welsch \cite{fan2021mm_full}; \cref{assumption:nonzero} requires that camera and point not coincide;  and \cref{assumption::bounded_intr,assumption::local_opt} are common in the  convergence analysis of optimization. Note that \cite{eriksson2016consensus,zhang2017dist} hold similar but stricter assumptions for decentralized bundle adjustment.

\begin{assumption}\label{assumption::loss}
	The robust loss function $\rho(\cdot):\reals^+\rightarrow\reals$ in \cref{eq::Fobj} has the following properties:
	\begin{enumerate}[(a)]
		\item $\rho(s)\geq 0$ and $\rho(0)=0$;\\[-0.7em]
		\item $\rho(s)$ is  differentiable;\label{assumption::loss::cont}\\[-0.7em]
		\item $\rho(s)$ is a concave function;\label{assumption::loss::concave}\\[-0.7em]
		\item $0\leq\nabla\rho(s)\leq 1$ for any $s\in\reals^+$ and $\nabla\rho(0)=1$ where $\nabla\rho(s)$ is the first-order derivative of $\rho(s)$;\label{assumption::loss::drho}\\[-0.7em]
		\item $\rho(s)$ has a Lipschitz continuous gradient. \label{assumption::loss::lipschitz}
	\end{enumerate}
\end{assumption}

\begin{assumption}\label{assumption:nonzero}
	There exists $\epsilon>0$ such that $\|\bfl_j-\bft_i\|>\epsilon$ for any  reprojection pair $(i,\,j)\in\calE$. 
\end{assumption}

\begin{assumption}\label{assumption::bounded_intr}
	The camera intrinsics $\bfd_i\in\Real{3}$ is bounded.
\end{assumption}

\begin{assumption}\label{assumption::local_opt}
	$\bfxakp$ is a local minimum to $\Ealphak$ and $\lEalphak$; see \cref{eq::update_amm,eq::update_mm}.
\end{assumption}

\newpage
\section{Proofs of Propositions}\label{section::app::propositions}
\subsection{Proof of \cref{proposition::majorize}}\label{proof::majorize}

First, we analyze and upper-bound the squared norm $\|\bfe_{ij}\|^2$ of reprojection errors. Recall that the reprojection error $\bfe_{ij}$ in \cref{eq::error} is derived by minimizing $\|\bfp_{ij} -\lambda_{ij}\cdot\bfR_i^\top(\bfl_j - \bft_i)\|^2$. Then, as a result of  \cref{eq::error_opt,eq::lambdaij,eq::error}, it can be shown that
\begin{equation}
\nonumber
	\|\bfe_{ij}\|^2 = \min_{\lambda_{ij}\in\reals} \left\|\bfp_{ij} -\lambda_{ij}\cdot\bfR_i^\top(\bfl_j - \bft_i)\right\|^2
\end{equation}
under the assumption of $\|\bfR_i^\top(\bfl_j-\bfl_i)\|=\|\bfl_j-\bft_i\|\neq 0$. Since $\bfR\in\SOthree$ and $\bfR\bfR^\top=\bfI$, the equation above is equivalent to
\begin{equation}\label{eq::enorm_opt}
	\|\bfe_{ij}\|^2 = \min_{\lambda_{ij}\in\reals} \big\|\bfR_i\bfp_{ij} -\lambda_{ij}\cdot(\bfl_j - \bft_i)\big\|^2.
\end{equation}
With $\lambda_{ij}=\lambdaijk$ in \cref{eq::gamma},  the right-hand side of \cref{eq::enorm_opt} is upper-bounded:
\begin{equation}\label{eq::enorm_ubnd0}
\begin{aligned}
	 \|\bfe_{ij}\|^2 
	 &=\min_{\lambda_{ij}\in\reals} \left\|\bfR_i\bfp_{ij} -\lambda_{ij}\cdot(\bfl_j - \bft_i)\right\|^2\\
	 &\leq \left\|\bfR_i\bfp_{ij} -\lambdaijk\cdot(\bfl_j - \bft_i)\right\|^2\\
	 &= \left\|\bfR_i\bfp_{ij} + \lambdaijk\cdot \bft_i - \lambdaijk\cdot\bfl_j\right\|^2
\end{aligned}
\end{equation}
where the equality ``='' holds if $\bfc_i=\bfcik$ and $\bfl_j=\bfljk$. For any $\bfx_i$ and $\bfy_j\in\Real{n}$, note that
\begin{equation}\label{eq::inequality}
\|\bfx_i-\bfy_j\|^2 =\min_{\bfg_{ij}\in\Real{n}} 2\|\bfx_i-\bfg_{ij}\|^2 + 2\|\bfy_j-\bfg_{ij}\|^2
\end{equation}
where the right-hand side has a unique solution at 
\begin{equation}
\bfg_{ij}=\half\bfx_i + \half\bfy_j.
\end{equation} 
Substituting  $\bfx_i = \bfR_i\bfp_{ij} + \lambdaijk\cdot \bft_i$ and $\bfy_j=\lambdaijk\cdot\bfl_j$ into \cref{eq::inequality} results in
\begin{equation}
\nonumber
 \left\|\bfR_i\bfp_{ij} + \lambdaijk\cdot \bft_i - \lambdaijk\cdot\bfl_j\right\|^2 =
 \min_{\bfg_{ij}\in\Real{3}} 
 2 \left\|\bfR_i\bfp_{ij} + \lambdaijk\cdot \bft_i-\bfg_{ij}\right\|^2 +
 2 \left\|\lambdaijk\cdot\bfl_j-\bfg_{ij}\right\|^2.
\end{equation}
If we let $\bfg_{ij}=\bfgijk$ in \cref{eq::g}, the equation above results in
\begin{equation}\label{eq::enorm_ubnd1}
	\left\|\bfR_i\bfp_{ij} + \lambdaijk\cdot \bft_i - \lambdaijk\cdot\bfl_j \right\|^2 \leq 
	2 \left\|\bfR_i\bfp_{ij} + \lambdaijk\cdot \bft_i-\bfgijk \right\|^2 +
	2 \left\|\lambdaijk\cdot\bfl_j-\bfgijk \right\|^2.
\end{equation}
Here, we might  upper-bound $\|\bfe_{ij}\|^2$ while decoupling camera extrinsics/intrinsics $\bfc_i=(\bfR_i,\,\bft_i,\,\bfd_i)\in \SEthree\times\Real{3}$ and point positions $\bfl_j\in\Real{3}$: 
\begin{equation}\label{eq::enorm_ubnd2}
	\begin{aligned}
	 \|\bfe_{ij}\|^2  
	 &\leq \left\|\bfR_i\bfp_{ij} + \lambdaijk\cdot \bft_i - \lambdaijk\cdot\bfl_j \right\|^2 \\
	 &\leq  2 \left\| \bfR_i\bfp_{ij} + \lambdaijk\cdot \bft_i-\bfgijk \right\|^2 + 2 \left\|\lambdaijk\cdot\bfl_j-\bfgijk \right\|^2
	\end{aligned}
\end{equation}
where the first inequality is from \cref{eq::enorm_ubnd0} and the second inequality is from \cref{eq::enorm_ubnd1}.

Next, note that the robust loss function $\rho(\cdot):\reals^+\rightarrow\reals$ is assumed to be differentiable, concave and nondecreasing; see \cref{assumption::loss}. From \cref{eq::Fij}, the concavity of robust loss function  immediately yields
\begin{equation}
	\nonumber
	\Fij \leq \frac{1}{2}\nabla\rho(\|\bfeijk\|)\cdot\left(\|\bfe_{ij}\|^2-\|\bfeijk\|^2\right)+ \frac{1}{2}\rho(\|\bfeijk\|^2)
\end{equation}
that holds for any $\bfe_{ij}$ and $\bfeijk$. The equation above is equivalent to
\begin{equation}\label{eq::loss_ubnd0}
	\Fij \leq \frac{1}{2}\wijk\cdot\|\bfe_{ij}\|^2+ \aijk
\end{equation}
where $\aijk$ and $\wijk$ are given by \cref{eq::w,eq::a}, respectively. Also, since the robust loss function $\rho(\cdot)$ is nondecreasing, we have $\wijk=\nabla\rho(\|\bfeijk\|^2)\geq 0$. Then, applying \cref{eq::enorm_ubnd2} onto the right-hand side of \cref{eq::loss_ubnd0} results in
\begin{equation}\label{eq::FPQ0}
	\begin{aligned}
	\Fij & \leq \wijk \cdot \left\|\bfR_i\bfp_{ij} + \lambdaijk\cdot \bft_i-\bfgijk \right\|^2 + \wijk\cdot \left\|\lambdaijk\cdot\bfl_j-\bfgijk \right\|^2 + \aijk \\
		  & =  \underbrace{\wijk\cdot \left\|\bfR_i\bfp_{ij} + \lambdaijk\cdot \bft_i-\bfgijk \right\|^2 + \half\aijk}_{P_{ij}(\bfc_i|\bfxk) } + \underbrace{\wijk\cdot \left\|\lambdaijk\cdot\bfl_j-\bfgijk \right\|^2 + \half\aijk}_{Q_{ij}(\bfl_j | \bfxk)} \\
		  & = \Pijk + \Qijk
	\end{aligned}
\end{equation}
where $\Pijk$ and $\Qijk$ are from \cref{eq::P,eq::Q}, respectively. Furthermore, substituting $\bfc_i=\bfcik$ and $\bfl_j=\bfljk$ into $\Fij$, $\Pijk$ and $\Qijk$ yields
\begin{equation}\label{eq::FPQ1}
	F\big(\bfcik,\,\bfljk\big) = P\big(\bfcik|\bfxk\big) + Q\big(\bfljk|\bfxk\big).
\end{equation}
Then, as a result of \cref{eq::FPQ0,eq::FPQ1}, we conclude that
\begin{equation}
\Fij \leq \Pijk + \Qijk
\end{equation}
and the equality ``='' holds at $\bfc_i=\bfcik$ and $\bfl_j=\bfljk$. This completes the proof of \cref{proposition::majorize}.
\subsection{Proof of \cref{proposition::surrogate}}\label{proof::surrogate}
The proof is straightforward from  \cref{proposition::majorize} and \cref{eq::Ealpha}.

\subsection{Proof of \cref{proposition::amm}}\label{proof::amm}
We start the proof by defining the notation for Euclidean and Riemannian gradients \cite{absil2009optimization}. Given a matrix manifold $\calM\subset\Real{m\times n}$ and a function $H(\cdot):\calM\subset\Real{m\times n} \rightarrow \Real{r}$, $\nabla H(\bfx)$ and $\grad H(\bfx)$ represents the Euclidean and Riemannian gradient, respectively; and $\nabla_{\bfx_i} H(\bfx)$ and $\grad_{\bfx_i} H(\bfx)$ represents the Euclidean and Riemannian gradient with respect to $\bfx_i\subset\bfx$, respectively.  With the notion of Riemannian gradient,  the convergence of $\amm$ to first-order critical points is equivalent to  
\begin{equation}
\grad F(\bfxk)\rightarrow\zero.
\end{equation}   
The rest of this proof is organized as follow.  We first prove $F(\bfxk)\rightarrow\Finf$ in App. \hyperref[section::app::prop3::F]{C.3.1}, then $\|\bfxkp-\bfxk\|\rightarrow 0$ and $\|\bfxkp-\bflxk\|\rightarrow 0$  in App. \hyperref[section::app::prop3::x]{C.3.2}, and at last $\grad F(\bfxk) \rightarrow \zero$ in App. \hyperref[section::app::prop3::grad]{C.3.3}. When analyzing the convergence, we also introduce \cref{lemma::adaptive,lemma::lipschitzF,lemma::lipschitzE,lemma::lipschitz_sum_prod,lemma::lipschitz_sum_prod_k} whose proofs are left in \cref{section::app::lemmas}.

\vspace{0.5em}
\begin{enumerate}[before=\itshape,left=0pt]
\item Proof of $F(\bfxk)\rightarrow\Finf$ \label{section::app::prop3::F}
\end{enumerate}
\vspace{0.2em}

For notational simplicity, we introduce $\lFk$ that is recursively defined by:
\begin{equation}\label{eq::lFk0}
	\lFk \triangleq \begin{cases}
		F(\bfx^{(0)}), & \sfk=-1,\\
		(1-\eta)\cdot \lFkm + \eta\cdot F(\bfxk), & \sfk\geq 0 
	\end{cases}
\end{equation} 
where $\eta\in(0,\,1]$ is the same as that in \cref{eq::lFak}. The equation above indicates that $\lFk$ is  an exponential averaging of the objective value $F(\bfxk)$ for $\sfk \geq 0$: 
\begin{equation}\label{eq::lFk1}
	\overline{F}^{(\sfk)} = (1-\eta)\cdot F(\bfx^{(0)}) + \eta \cdot \sum_{\sfn=0}^{\sfk} (1-\eta)^{\sfk-\sfn} \cdot F(\bfx^{(\sfn)})
\end{equation}
Furthermore, we have the following proposition about local per device adaptive restart metrics $\Fak,\, \lFak,\,\Eakp$ in \cref{eq::Fainit,eq::Fak,eq::lFak,eq::Eak}.

\begin{lemma}\label{lemma::adaptive}
	The  adaptive restart metrics $\Fak,\, \lFak,\,\Eakp$ in \cref{eq::Fainit,eq::Fak,eq::lFak,eq::Eak} satisfy the following properties:
	\begin{enumerate}[(a),left=0pt]
		\item $\sum_{\alpha\in\calS}\Fak = F(\bfxk)$;  \label{lemma::adaptive::Fsum}
		\vspace{0.35em}
		\item $\sum_{\alpha\in\calS} \lFak=\lFk$;  \label{lemma::adaptive::lFsum}
		\vspace{0.35em}
		\item $\Fak \leq \Eak$. \label{lemma::adaptive::FakEak}
	\end{enumerate} 
\end{lemma}
\begin{proof}
	Please refer to App. \hyperref[section::app::lemma::adaptive]{D.1}.
\end{proof}

From \cref{eq::lFk0,eq::lFk1}, it is straightforward to conclude  that $F(\bfxk)\rightarrow \Finf$ if and only if $\lFk \rightarrow \Finf$.  Moreover, \cref{eq::Fobj,eq::lFk1} indicate that $F(\bfxk)$ and $\lFk$ are bounded below, i.e., $F(\bfxk)$ and $\lFk\geq 0$. As a result of monotone convergence theorem, a sequence converges if nonincreasing and bounded below. Then, the convergence of $F(\bfxk)$ and $\lFk$ is established if we can prove that $\lFk$ is nonincreasing. In addition, Lemma \ref{lemma::adaptive}\ref{lemma::adaptive::lFsum} indicates that  $\lFakp \leq \lFak$ for each device $\alpha$ is sufficient to yield $\lFkp\leq\lFk$. Therefore, the convergence of  $F(\bfxk)$ and $\lFk$ is reduced to  $\lFakp \leq \lFak$, which can be achieved through the proof of $\Fakp\leq \lFakp \leq\lFak$  by induction as the following.

\begin{enumerate} [1.]
\item For $\sfk=-1$, \cref{eq::Fainit} indicates that $\bfx^{\alpha(-1)} = \bfx^{\alpha(0)}$ and
\begin{equation}
	F^{\alpha(-1)} = \Ealpha(\bfx^{\alpha(-1)}|\iterate{\bfx}{}{-1}),\; \overline{F}^{\alpha(-1)} = F^{\alpha(-1)},\; E^{\alpha(0)} = F^{\alpha(-1)}.
\end{equation} 
With the equation above, $\bfx^{\alpha(-1)} = \bfx^{\alpha(0)}$ and $\iterate{\bfx}{}{-1}=\iterate{\bfx}{}{0}$, \cref{eq::Fak,eq::lFak} further result in
\begin{equation}\label{eq::FakplFakplFak0}
F^{\alpha(-1)}=\overline{F}^{\alpha(-1)}=\overline{F}^{\alpha(0)}=F^{\alpha(0)}.
\end{equation}
\item For $\sfk\geq 0$, we assume that $\Fak\leq\lFak\leq \lFakm$ holds. Then,  $\bfxakp$ in $\amm$ results from either line~\ref{line::amm::update_amm} or line~\ref{line::amm::update_mm} of \cref{algorithm::amm}. This means there are two possibilities as the following.
\begin{itemize}
	\item $\bfxakp$ is from line~\ref{line::amm::update_amm} of \cref{algorithm::amm}, or equivalently, \cref{eq::update_amm}.  Then, the adaptive restart scheme is not triggered and line~\ref{line::amm::adaptive_restart::start} of \cref{algorithm::amm} suggests that
	\begin{equation}\label{eq::EaklFak0}
		\Eakp \leq \lFak.		
	\end{equation}
	\item $\bfxakp$ is from line~\ref{line::amm::update_mm} of \cref{algorithm::amm}, or equivalently, \cref{eq::update_mm}. Then, the adaptive restart scheme is triggered. Recall from \cref{assumption::local_opt} that $\bfxakp$ is a local minimum to $\Ealphak$.  This suggests that
	\begin{equation}\label{eq::EalFa}
		E^\alpha(\bfxakp|\bfxk) - E^\alpha(\bfxak|\bfxk)\leq 0.
	\end{equation}
	Applying \cref{eq::EalFa} on  \cref{eq::Eak}, we obtain
	\begin{equation}
		\Eakp \leq \Fak.
	\end{equation}
	Since we assume that $\Fak\leq\lFak$,  the equation above  results in
	\begin{equation}\label{eq::EaklFak1}
		\Eakp \leq \lFak.
	\end{equation}
\end{itemize}
Then, no matter whether the adaptive restart scheme is triggered or not,  it can be shown that
\begin{equation}\label{eq::FakplFk1}
\Fakp\leq\Eakp\leq\lFak
\end{equation}
where the first inequality is from \cref{lemma::adaptive}\ref{lemma::adaptive::FakEak} and  the second inequality is from \cref{eq::EaklFak0,eq::EaklFak1}. Furthermore, recall from \cref{eq::lFak} that $\lFakp$ is a convex combination of $\lFak$ and $\Fakp$. Then, \cref{eq::FakplFk1}  yields
\begin{equation}\label{eq::FakplFakplFak1}
	 \Fakp \leq \lFakp \leq \lFak.
\end{equation}
\item From \cref{eq::FakplFakplFak0,eq::FakplFakplFak1}, we conclude that $\Fakp \leq \lFakp \leq \lFak$ holds for any $\sfk\geq -1$.
\end{enumerate}
Therefore, we have proved by induction that $\lFak$ is nonincreasing, which, as analyzed before, is sufficient to yield not only $\lFk\rightarrow\Finf$ but also $F(\bfxk)\rightarrow\Finf$. This completes the proof of $F(\bfxk)\rightarrow\Finf$.

\vspace{0.5em}
\begin{enumerate}[before=\itshape,left=0pt]
\setcounter{enumi}{1}
\item Proof of $\|\bfxkp-\bfxk\|\rightarrow 0$  and $\|\bfxkp - \bflxk\|\rightarrow 0$ \label{section::app::prop3::x}
\end{enumerate}
\vspace{0.2em}

First, we prove $\|\bfxkp-\bfxk\|\rightarrow 0$. For any $\sfk\geq 0$, we conclude from \cref{eq::Fak,eq::DEalpha} that 
\begin{equation}\label{eq::FakpEakp1}
	\Fakp - \Eakp =  -\frac{1}{2}\xi \|\bfxakp-\bfxak\|^2 + \half\sum_{(i,\,j)\in\calE''_\alpha} \Big( F_{ij}\big(\bfc_i^{(\sfk+1)},\, \bfl_j^{(\sfk+1)}\big) - P_{ij}\big(\bfc_i^{(\sfk+1)}|\bfxk\big) - Q_{ij}\big(\bfl_j^{(\sfk+1)}|\bfxk\big)\Big).
\end{equation}
Recall from \cref{proposition::majorize} that 
\begin{equation}
	F_{ij}\big(\bfc_i^{(\sfk+1)},\, \bfl_j^{(\sfk+1)}\big) \leq P_{ij}\big(\bfc_i^{(\sfk+1)}|\bfxk\big) + Q_{ij}\big(\bfl_j^{(\sfk+1)}|\bfxk\big).
\end{equation}
Then, the right-hand side of \cref{eq::FakpEakp1} can be upper-bounded:
\begin{equation}\label{eq::FakpEakp0}
	\Fakp - \Eakp \leq  -\frac{1}{2}\xi \|\bfxak-\bfxakm\|^2.
\end{equation}
In addition,  we have proved in \cref{eq::FakplFk1}  that 
\begin{equation}\label{eq::EakplFak0}
	\Eakp - \lFak \leq 0
\end{equation}
for any $\sfk\geq 0$. As a result of \cref{eq::FakpEakp0,eq::EakplFak0}, we obtain
\begin{equation}\label{eq::FakplFakp0}
	\Fakp - \lFak \leq  -\frac{1}{2}\xi \|\bfxakp-\bfxak\|^2.
\end{equation}
Then, summing both sides of the equation above over all the devices $\alpha\in\calS$ results in
\begin{equation}\label{eq::FklFk0}
F(\bfxkp) - \lFk = \sum_{\alpha\in\calS} \Fakp - \sum_{\alpha\in\calS}\lFakp \leq - \frac{1}{2}\xi \sum_{\alpha\in\calS}\|\bfxakp-\bfxak\|^2=-\frac{1}{2}\xi\|\bfxkp-\bfxk\|^2
\end{equation}
where the first equality is from Lemmas \ref{lemma::adaptive}\ref{lemma::adaptive::Fsum} and \ref{lemma::adaptive}\ref{lemma::adaptive::lFsum}. Moreover, as a result of \cref{eq::lFk0}, we obtain
\begin{equation}\label{eq::lFakplFak0}
	\lFkp -\lFk = \eta\cdot\big(F(\bfxkp) - \lFk\big).
\end{equation}
Then, applying  \cref{eq::FklFk0} to upper-bound the right-hand side of \cref{eq::lFakplFak0} yields
\begin{equation}\label{eq::lFkplFk0}
\lFkp - \lFk  \leq -\frac{1}{2}\eta\xi\|\bfxkp-\bfxk\|^2 \leq 0.
\end{equation}
In \hyperref[section::app::prop3::F]{App. C.3.1}, we have proved that  $\lFk$ converges, which suggests that
\begin{equation}\label{eq::lFkplFk1}
\lim_{\sfk\rightarrow\infty}  \lFkp-\lFk =0.
\end{equation}
As a result of \cref{eq::lFakplFak0,eq::lFkplFk1}, we conclude 
\begin{equation}
0=\lim_{\sfk\rightarrow\infty} \lFkp-\lFk \leq \lim_{\sfk\rightarrow\infty} -\frac{1}{2}\eta\xi\|\bfxkp-\bfxk\|^2\leq 0.
\end{equation}
With $\xi>0$ and $0<\eta\leq 1$, the equation above yields
\begin{equation}\label{eq::xkpxk}
\|\bfxkp-\bfxk\|\rightarrow 0.
\end{equation}

Next, we prove $\|\bfxkp - \bflxk\|\rightarrow 0$. From \cref{eq::nesterov_scalar}, note that $\sak\geq 1$ and
\begin{equation}\label{eq::gamma_bnd}
	\gammaak = \frac{2\sak-2}{\sqrt{4{\sak}^2+1} + 1}\leq \frac{\sak-1}{\sak}\in[0,\,1).
\end{equation}
With \cref{eq::xkpxk,eq::gamma_bnd}, it can be shown that
\begin{equation}\label{eq::lxakxak}
\bfxak + \gammaak\big(\bfxak-\bfxakm\big) \rightarrow \bfxak.
\end{equation} 
This results in
\begin{equation}
	\bfRik + \gammaak\big(\bfRik-\bfRikm\big) \rightarrow \bfRik.
\end{equation}
Since $\SVDOPlus(\cdot)$ in \cref{eq::proj_rot3d} is continuous around $\bfRik\in\SOthree$ \cite[Sec. 3.4]{levinson2020analysis}, the equation above suggests that
\begin{equation}\label{eq::lRikRik}
\SVDOPlus\left(\bfRik + \gammaak\big(\bfRik-\bfRikm\big)\right) \rightarrow \bfRik.
\end{equation}
In addition, \cref{eq::lxakxak} also results in 
\begin{equation}\label{eq::ltiktik}
	\bftik + \gammaak\big(\bftik-\bftikm\big) \rightarrow \bftik,
\end{equation}
\begin{equation}\label{eq::ldikdik}
	\bfdik + \gammaak\big(\bfdik-\bfdikm\big) \rightarrow \bfdik,
\end{equation}
\begin{equation}\label{eq::lljkljk}
	\bfljk + \gammaak\big(\bfljk-\bfljkm\big) \rightarrow \bfljk.
\end{equation}
From \cref{eq::lRikRik,eq::ltiktik,eq::ldikdik,eq::lljkljk},  \cref{eq::nesterov_x} immediately yields
\begin{equation}\label{eq::xklxk}
	\|\bfxk-\bflxk\|\rightarrow 0.
\end{equation}
Furthermore, with \cref{eq::xklxk,eq::xkpxk} and
\begin{equation}
	\|\bfxkp - \bflxk\| \leq \|\bfxkp - \bfxk\|  + \|\bfxk - \bflxk\|, 
\end{equation}
we  obtain
\begin{equation}
	\|\bfxkp - \bflxk\| \rightarrow 0.
\end{equation}
 This completes the proof of $\|\bfxkp-\bfxk\|\rightarrow 0$ and $\|\bfxkp - \bflxk\|\rightarrow 0$.

\vspace{0.5em}
\begin{enumerate}[before=\itshape,left=0pt]
	\setcounter{enumi}{2}
	\item Proof of $\grad F(\bfxk) \rightarrow \zero$ \label{section::app::prop3::grad}
\end{enumerate}
\vspace{0.2em}

For a function $H(\cdot):\calM\rightarrow \Real{r}$ on matrix manifold $\calM$,  Riemannian gradient $\grad H(\bfx)$ and  Euclidean $\nabla H(\bfx)$ are related in the form of
\begin{equation}
\grad H(\bfx) = \ProjGrad_{\bfx} \big(\nabla H(\bfx)\big)
\end{equation}
where  $\ProjGrad_\bfx(\cdot)$  is a linear operator associated with $\bfx$ and projects Euclidean gradients to the Riemannian tangent space at $\bfx$. In terms of camera extrinsics/intrinsics $(\bfR_i,\bft_i,\bfd_i)\in \SEthree\times\Real{3}$ and point positions $\bfl_j\in\Real{3}$,  such a linear operator $\ProjGrad_\bfx(\cdot)$  is defined by
\begin{subequations}\label{eq::proj_grad}
\begin{equation}\label{eq::grad_R}
	\grad_{\bfR_i} H(\bfx) =\half\nabla_{\bfR_i} H(\bfx) -\half \bfR_i\nabla_{\bfR_i} H(\bfx)^\top\bfR_i,
\end{equation}
\begin{equation}\label{eq::grad_t}
	\grad_{\bft_i} H(\bfx) = \nabla_{\bft_i} H(\bfx),
\end{equation}
\begin{equation}\label{eq::grad_d}
	\grad_{\bfd_i} H(\bfx) = \nabla_{\bfd_i} H(\bfx),
\end{equation}
\begin{equation}\label{eq::grad_l}
	\grad_{\bfl_j} H(\bfx) = \nabla_{\bfl_j} H(\bfx).
\end{equation}
\end{subequations}
Similar to \cite{fan2021mm_full,eriksson2016consensus,zhang2017dist},  we  need the Lipschitz-like continuity of Riemannian gradients  to guarantee the convergence to first-order critical points. Here, we introduce the following two lemmas about the sum and product of Lipschitz continuous functions.

\begin{lemma} \label{lemma::lipschitz_sum_prod}
	Suppose that $G(\cdot)$ and $H(\cdot):\Real{m\times n} \rightarrow \reals$ are Lipschitz continuous functions. Then, we have the following results:
	\begin{enumerate}[(a),left=0pt]
		\item $G(\bfx) + H(\bfx)$ is Lipschitz continuous. \label{lemma::lipschitz_sum}
		\item $G(\bfx) \cdot H(\bfx)$ is bounded and Lipschitz continuous if $G(\bfx)$ and $H(\bfx)$ are bounded. \label{lemma::lipschitz_prod}
	\end{enumerate}
\end{lemma}
\begin{proof}
	Please refer to App. \hyperref[section::app::lemma::lipschitz_sum_prod]{D.2}.
\end{proof}

\begin{lemma}\label{lemma::lipschitz_sum_prod_k}
	Suppose that  $\{\bfxk\}$ is a sequence, $G(\cdot)$ and $H(\cdot):\Real{m\times n} \rightarrow \reals$ are functions,  and there exists a constant $L>0$ such that
	\begin{equation}
		\big\|G(\bfxkp) - G(\bfxk)\big\| \leq L\big\|\bfxkp - \bfxk \big\|,
	\end{equation} 
	\begin{equation}
		\big\|H(\bfxkp) - H(\bfxk)\big\| \leq L\big\|\bfxkp - \bfxk \big\|
	\end{equation} 
	for any $\sfk\geq 0$. Then, we have the following results:
	\begin{enumerate}[(a),left=0pt]
		\item There exists a constant $L'>0$ such that  \label{lemma::lipschitz_sum_k}
		\begin{equation}
			\big|\big(G(\bfxkp) + H(\bfxkp)\big) -  \big(G(\bfxk) + H(\bfxk)\big) \big| \leq L' \big\|\bfxkp - \bfxk\big\|
		\end{equation}
		for any $\sfk\geq 0$.
		\item $G(\bfxk) \cdot H(\bfxk)$ and $G(\bfxkp) \cdot H(\bfxkp)$ are bounded, and there exits a constant $L''>0$ such that
		\begin{equation}
			\big|G(\bfxkp) \cdot H(\bfxkp) -  G(\bfxk) \cdot H(\bfxk) \big| \leq L'' \big\|\bfxkp - \bfxk\big\|
		\end{equation}
		for any $\sfk\geq 0$ if $G(\bfxk)$, $G(\bfxkp)$, $H(\bfxk)$ and $H(\bfxkp)$ are bounded.
	\end{enumerate}
\end{lemma}
\begin{proof}
	Please refer to App. \hyperref[section::app::lemma::lipschitz_sum_prod_k]{D.3}.
\end{proof}

Note that  $\bfR_i\in\SOthree$ is on a compacted manifold. With \cref{lemma::lipschitz_sum_prod,lemma::lipschitz_sum_prod_k},  \cref{eq::proj_grad} suggests that there exists a constant $L'>0$ such that
\begin{equation}
\| \grad H(\bfxkp) - \grad H(\bfxk) \| \leq L'\|\bfxkp - \bfxk\|
\end{equation}
if Euclidean gradients  $\nabla H(\bfxkp)$ and   $\nabla H(\bfxk)$ are bounded and there exists a constant $L>0$ such that 
\begin{equation}
	\|\nabla H(\bfxkp) - \nabla H(\bfxk)\| \leq L \|\bfxkp - \bfxk\|.
\end{equation}
As a matter of fact,  $F(\bfx)$, $E(\bfx|\bfxk)$, $E(\bfx|\bflxk)$ indeed have bounded Euclidean gradients satisfying the  equation above at $\bfxk$, $\bflxk$ $\bfxkp$.
\begin{lemma} \label{lemma::lipschitzF}
$ F(\bfx)$ has bounded and Lipschitz continuous Euclidean gradients $\nabla F(\bfx)$ under \cref{assumption::bounded_intr,assumption::loss,assumption:nonzero}.
\end{lemma}
\begin{proof}
	Please refer to App. \hyperref[section::app::lemma::lipschitzF]{D.4}.
\end{proof} 

\begin{lemma} \label{lemma::lipschitzE}
	Suppose $\{\bfxk\}$ and $\{\bflxk\}$ are sequences  resulting from \cref{algorithm::amm}. Then $\nabla E(\bfxk|\bfxk)$, $\nabla E(\bfxkp|\bfxk)$, $\nabla E(\bflxk|\bflxk)$ and $\nabla E(\bfxkp|\bflxk)$ are bounded, and there exists a constant $L>0$ such that
	\begin{equation}\label{eq::lipschitz_E}
		\|\nabla E(\bfxkp|\bfxk) - \nabla E(\bfxk|\bfxk)\|\leq L\|\bfxkp-\bfxk\|
	\end{equation}
	and
	\begin{equation}\label{eq::lipschitz_lE}
		\|\nabla E(\bfxkp|\bflxk) - \nabla E(\bflxk|\bflxk)\|\leq L\|\bfxkp-\bflxk\|
	\end{equation}
	 for any $\sfk\geq 0$ under \cref{assumption::bounded_intr,assumption::loss,assumption:nonzero}.
\end{lemma}
\begin{proof}
	Please refer to App. \hyperref[section::app::lemma::lipschitzE]{D.5}.
\end{proof} 

Then, as discussed before, \cref{lemma::lipschitzF,lemma::lipschitzE} suggest that there exists a constant $L'>0$ such that 
\begin{equation}\label{eq::lipschitz_Fk}
	\|\grad F(\bfxkp) - \grad F(\bfxk)\| \leq L'\|\bfxkp-\bfxk\|,
\end{equation}
\begin{equation}\label{eq::lipschitz_Ek}
	\|\grad E(\bfxkp|\bfxk) - \grad E(\bfxk|\bfxk)\| \leq L'\|\bfxkp-\bfxk\|,
\end{equation}
\begin{equation}\label{eq::lipschitz_lEk}
	\|\grad E(\bfxkp|\bflxk) - \grad E(\bflxk|\bflxk)\| \leq L'\|\bfxkp-\bflxk\|
\end{equation}
for any $\sfk\geq 0$. With \cref{eq::lipschitz_lEk,eq::lipschitz_Fk,eq::lipschitz_Ek}, the proof of $\grad F(\bfxk) \rightarrow \zero$ is as the following. 

As a result of \cref{eq::proj_grad}, it is tedious but straightforward to show from \cref{eq::Fij,eq::P,eq::Q,eq::error} that $\Fij$ and $\Pijk+\Qijk$ have the same Riemannian gradient at $\bfc_i=\bfcik$ and $\bfl_j=\bfljk$:
\begin{subequations}\label{eq::gradFij}
\begin{equation}
	\grad_{\bfc_i} F_{ij}\big(\bfcik,\bfljk\big) = \grad P_{ij}\big(\bfcik|\bfxk\big),
\end{equation}
\begin{equation}
	\grad_{\bfl_j} F_{ij}\big(\bfcik,\bfljk\big) = \grad Q_{ij}\big(\bfljk|\bfxk\big).
\end{equation}
\end{subequations}
The equation above further suggests  that $F(\bfx)$ and $E(\bfx|\bfxk)$ in \cref{eq::Fobj,eq::Ealpha} have the same Riemannian gradient at $\bfx=\bfxk$:
\begin{equation}\label{eq::gradFgradE0}
\grad F(\bfxk) = \grad E(\bfxk|\bfxk).
\end{equation}
Recall from \cref{eq::Esum}  that $E(\bfx|\bfxk)=\sum_{\alpha\in\calS}\Ealphak$  that decouples  $\bfxa$ on different devices. Then, \cref{eq::gradFgradE0} yields
\begin{equation}\label{eq::gradFagradEa0}
\grad_{\bfxa} F(\bfxk) = \grad \Ealpha(\bfxak|\bfxk).
\end{equation}
Note that \cref{algorithm::amm} suggests that $\bfxkp$ results from \cref{eq::update_mm} or \cref{eq::update_amm}. Since $\bfxakp$ is a local minimum, we obtain either
\begin{equation}\label{eq::gradEa0}
	\grad \Ealpha(\bfxakp|\bfxk) = \zero.
\end{equation}
or
\begin{equation}\label{eq::gradEa1}
	\grad \Ealpha(\bfxakp|\bflxk) = \zero
\end{equation}
In summary, there are two possibilities:
\begin{itemize}
\item If $\bfxakp$ results from \cref{eq::update_mm}, then
\begin{equation}
	\begin{aligned}
   &\big\|\grad_{\bfxa} F(\bfxkp)\big\| \\
= &\big\|\grad_{\bfxa} F(\bfxkp) - \grad \Ealpha(\bfxakp|\bfxk)\big\| \\
= &\big\|\grad_{\bfxa} F(\bfxkp)-\grad_{\bfxa} F(\bfxk) +  \grad \Ealpha(\bfxak|\bfxk) - \grad \Ealpha(\bfxakp|\bfxk)\big\| \\
\leq & \big\|\grad_{\bfxa} F(\bfxkp)-\grad_{\bfxa} F(\bfxk)\big\| + \big\| \grad \Ealpha(\bfxakp|\bfxk) - \grad \Ealpha(\bfxak|\bfxk) \big\|
	\end{aligned}
\end{equation}
where the first inequality is from \cref{eq::gradEa0}, the second equality is from $\grad_{\bfxa} F\big(\bfxk\big) =  \grad \Ealpha\big(\bfxak|\bfxk\big)$ in \cref{eq::gradFagradEa0}, and the last inequality is from triangle inequality. The equation above further suggests that
\begin{equation}\label{eq::bnd_gradFa0}
	\begin{aligned}
		&\big\|\grad_{\bfxa} F(\bfxkp)\big\|  \\
\leq & \big\|\grad_{\bfxa} F(\bfxkp)-\grad_{\bfxa} F(\bfxk)\big\| + \big\| \grad \Ealpha(\bfxakp|\bfxk) - \grad \Ealpha(\bfxak|\bfxk) \big\| \\
\leq &\big\| \grad F(\bfxkp) - \grad F(\bfxk) \big\| + \big\|\grad E(\bfxkp|\bfxk) - \grad E(\bfxk|\bfxk)\big\| \\
\leq & 2L' \big\|\bfxkp - \bfxk\big\|
	\end{aligned}
\end{equation}
where the last inequality is from \cref{eq::lipschitz_Ek,eq::lipschitz_Fk}.
\item If  $\bfxakp$ results from \cref{eq::update_amm}, then following a similar procedure, we obtain
\begin{equation}\label{eq::bnd_gradFa1}
\big\|\grad_{\bfxa} F(\bfxkp)\big\| \leq 2L'\big\|\bfxkp - \bflxk\big\|.
\end{equation}
\end{itemize}
As a result of \cref{eq::bnd_gradFa0,eq::bnd_gradFa1}, it is straightforward to show that
\begin{equation}
\big\|\grad_{\bfxa} F(\bfxkp)\big\| \leq 2L' \big\|\bfxkp - \bfxk\big\| + 2L'\big\|\bfxkp - \bflxk\big\|.
\end{equation}
Recall from  App. \hyperref[section::app::prop3::x]{C.3.2} that $\|\bfxkp - \bfxk\big\|\rightarrow 0$ and $\|\bfxkp - \bflxk\big\|\rightarrow 0$. Then, the equation above results in 
\begin{equation}
\big\|\grad_{\bfxa} F(\bfxkp)\big\|  \rightarrow 0,
\end{equation}
from which we  conclude that $\grad F(\bfxk)\rightarrow \zero$. This completes the proof.

\section{Proofs of Lemmas}\label{section::app::lemmas}

\subsection{Proof of \cref{lemma::adaptive}}\label{section::app::lemma::adaptive}
\vspace{0.5em}
\begin{enumerate}[before=\itshape,left=0pt]
	\item Proof of \cref{lemma::adaptive}\ref{lemma::adaptive::Fsum}
\end{enumerate}
\vspace{0.2em}

\cref{lemma::adaptive}\ref{lemma::adaptive::Fsum} is proved by induction as the following.

\begin{enumerate} [1.]
	\item For $\sfk=-1$, \cref{eq::Fainit} indicates that
	\begin{equation}
	\bfx^{\alpha(-1)} = \bfx^{\alpha(0)},\; F^{\alpha(-1)} = \Ealpha(\bfx^{\alpha(-1)}|\iterate{\bfx}{}{-1}),\; \overline{F}^{\alpha(-1)} = F^{\alpha(-1)},\; E^{\alpha(0)} = F^{\alpha(-1)}.
	\end{equation}
Note that \cref{proposition::surrogate} indicates that
\begin{equation}\label{eq::FEk0}
	F(\bfxk) = E(\bfxk|\bfxk)
\end{equation} 
Therefore, we obtain
\begin{equation}
\sum_{\alpha\in\calS} F^{\alpha(-1)} = \sum_{\alpha\in\calS}\Ealpha(\bfx^{\alpha(-1)}|\iterate{\bfx}{}{-1}) = E(\iterate{\bfx}{}{-1}|\iterate{\bfx}{}{-1}) = F(\iterate{\bfx}{}{-1})
\end{equation}
where the second equality is from \cref{eq::Esum} and the third equality is from \cref{eq::FEk0}. Furthermore, with \cref{eq::FakplFakplFak0} and $\iterate{\bfx}{}{-1} = \iterate{\bfx}{}{0}$, the equation above results in
\begin{equation}\label{eq::Fsumk0}
\sum_{\alpha\in\calS}  F^{\alpha(0)} = \sum_{\alpha\in\calS} F^{\alpha(-1)} =F(\iterate{\bfx}{}{-1}) = F(\iterate{\bfx}{}{0}).
\end{equation}

\item For $\sfk\geq 0$, we assume that
\begin{equation}\label{eq::Fsumk}
	\sum_{\alpha\in\calS} \Fak = F(\bfxk).
\end{equation}
From \cref{eq::Ealpha,eq::DEalpha}, an algebraic manipulation indicates that
\begin{equation}\label{eq::FEDE0}
F(\bfx) = \sum_{(i,\,j)\in\calE} \Fij = \sum_{\alpha\in\calS} \Ealphak + \sum_{\alpha\in\calS} \Delta E^\alpha(\bfx|\bfxk).
\end{equation}
In addition, \cref{eq::Fak} indicates that
\begin{equation}
	\sum_{\alpha\in\calS} \Fakp = \sum_{\alpha\in\calS} \Eakp + \sum_{\alpha\in\calS} \Delta\Ealpha\big(\bfxakp|\bfxk\big) 
\end{equation}
Substituting \cref{eq::Eak} into the equation above to expand $\Eakp$ results
\begin{equation}\label{eq::FakpSum0}
	\sum_{\alpha\in\calS} \Fakp = \sum_{\alpha\in\calS} \Ealpha\big(\bfxakp|\bfxk\big) + \sum_{\alpha\in\calS} \Delta\Ealpha\big(\bfxakp|\bfxk\big)  + \sum_{\alpha\in\calS} \Fak - \sum_{\alpha\in\calS} \Ealpha\big(\bfxak|\bfxk\big).
\end{equation}
 Applying \cref{eq::FEDE0} on the right-hand side of \cref{eq::FakpSum0}, we further obtain
\begin{equation}
	\sum_{\alpha\in\calS} \Fakp = F(\bfxkp) + \sum_{\alpha\in\calS} \Fak - \sum_{\alpha\in\calS} \Ealpha\big(\bfxak|\bfxk\big)
\end{equation}
With \cref{eq::Fsumk,eq::Esum}, the equation above is equivalent to 
\begin{equation}\label{eq::Fsumk2}
	\sum_{\alpha\in\calS} \Fakp = F(\bfxkp) + F(\bfxk)-  E\big(\bfxk|\bfxk\big) = F(\bfxkp)
\end{equation}
where the last equality is from \cref{eq::FEk0}.
\item From \cref{eq::Fsumk0,eq::Fsumk2}, we conclude that $\sum_{\alpha\in\calS} \Fak = F(\bfxk)$ for any $\sfk\geq 0$. This completes the proof.
\end{enumerate}

\vspace{0.5em}
\begin{enumerate}[before=\itshape,left=0pt]
	\setcounter{enumi}{1}
	\item Proof of \cref{lemma::adaptive}\ref{lemma::adaptive::lFsum}
\end{enumerate}
\vspace{0.2em}

With \cref{eq::lFak,eq::FakplFakplFak0} , it is straightforward to show that
\begin{equation}\label{eq::lFak0}
	\lFak = (1-\eta)\cdot F^{\alpha(0)} + \eta \cdot \sum_{\sfn=0}^{\sfk} (1-\eta)^{\sfk-\sfn} \cdot F^{\alpha(\sfn)}
\end{equation}
for any $\sfk\geq  0$. Summing both sides of the equation above over $\alpha\in\calS$, we obtain
\begin{equation}
	\sum_{\alpha\in\calS} \lFak = (1-\eta)\cdot \sum_{\alpha\in\calS} F^{\alpha(0)} + \eta \cdot \sum_{\sfn=0}^{\sfk} (1-\eta)^{\sfk-\sfn} \cdot \sum_{\alpha\in\calS} F^{\alpha(\sfn)}
\end{equation}
Furthermore, \cref{lemma::adaptive}\ref{lemma::adaptive::Fsum} indicates that the equation above can be further simplified to
\begin{equation}\label{eq::lFaksum0}
	\sum_{\alpha\in\calS} \lFak = (1-\eta)\cdot F(\bfx^{(0)}) + \eta \cdot \sum_{\sfn=0}^{\sfk} (1-\eta)^{\sfk-\sfn} \cdot F(\bfx^{(\sfn)}).
\end{equation}
Applying \cref{eq::lFk1} on the right-hand side of \cref{eq::lFaksum0} results in
\begin{equation}
\sum_{\alpha\in\calS} \lFak = \lFk.
\end{equation}
This completes the proof.

\vspace{0.5em}
\begin{enumerate}[before=\itshape,left=0pt]
	\setcounter{enumi}{2}
	\item Proof of \cref{lemma::adaptive}\ref{lemma::adaptive::FakEak}
\end{enumerate}
\vspace{0.2em}

As a result of \cref{proposition::surrogate}, it is known that $\Fij - \Pijk - \Qijk \leq 0$. Then, \cref{eq::DEalpha} indicates that $\Delta\Ealphak\leq 0$ always holds. From \cref{eq::Fak}, it is straightforward to conclude $\Fak\leq \Eak$. This completes the proof. 

\vspace{1em}

\subsection{Proof of \cref{lemma::lipschitz_sum_prod}}\label{section::app::lemma::lipschitz_sum_prod}
\vspace{0.5em}
\begin{enumerate}[before=\itshape,left=0pt]
	\item Proof of \cref{lemma::lipschitz_sum_prod}\ref{lemma::lipschitz_sum}
\end{enumerate}
\vspace{0.2em}

Since $G(\cdot)$ and $H(\cdot):\Real{m\times n} \rightarrow \reals$ are Lipschitz continuous functions, then there exists $L>0$ such that
\begin{equation}\label{eq::lipschitz_GH}
	\big|G(\bfx) - G(\bfx')\big|\leq L\cdot\|\bfx - \bfx'\|\;\, \text{and}\;\, \big|H(\bfx) - H(\bfx')\big|\leq L\cdot\|\bfx - \bfx'\|.
\end{equation}
Then, we obtain
\begin{equation}
\begin{aligned}
   & \big|\big(G(\bfx) + H(\bfx)\big) - \big(G(\bfx') + H(\bfx')\big)\big| \\
= & \big|G(\bfx)  - G(\bfx') + H(\bfx)- H(\bfx')\big| \\
\leq & \big|G(\bfx)  - G(\bfx') \big| + \big| H(\bfx)- H(\bfx')\big|\\
\leq & 2L\cdot \|\bfx - \bfx'\|
\end{aligned}
\end{equation}
where the first inequality is from triangle inequality. This completes the proof.

\vspace{0.5em}
\begin{enumerate}[before=\itshape,left=0pt]
	\setcounter{enumi}{1}
	\item Proof of \cref{lemma::lipschitz_sum_prod}\ref{lemma::lipschitz_prod}
\end{enumerate}
\vspace{0.2em}
With $G(\bfx)$ and $H(\bfx)$ bounded, there exists $C>0$ such that
\begin{equation}\label{eq::bounded_GH}
	\big|G(\bfx)\big| \leq C\;\, \text{and} \;\, \big|H(\bfx)\big| \leq C
\end{equation}
for any $\bfx\in\Real{m\times n}$. Then,  we obtain $\big|G(\bfx)\cdot H(\bfx)\big|\leq C^2$, and thus, $\big| G(\bfx)\cdot H(\bfx) \big|$ is bounded. In addition, it can be shown that
\begin{equation}
	\begin{aligned}
		  &\big|G(\bfx)\cdot H(\bfx) - G(\bfx')\cdot H(\bfx')\big| \\
		=& \big|G(\bfx)\cdot H(\bfx)  - G(\bfx')\cdot H(\bfx)   +  G(\bfx')\cdot H(\bfx) - G(\bfx')\cdot H(\bfx') \big|\\
		\leq & \big|G(\bfx)\cdot H(\bfx)  - G(\bfx')\cdot H(\bfx)\big| + \big| G(\bfx')\cdot H(\bfx) - G(\bfx')\cdot H(\bfx') \big| \\
		= & \big|H(\bfx)\big|\cdot\big|G(\bfx)  - G(\bfx')\big| + \big|G(\bfx')\big| \cdot \big|  H(\bfx) -  H(\bfx') \big| \\
		\leq & C \cdot\big|G(\bfx)  - G(\bfx')\big| + C \cdot \big|  H(\bfx) -  H(\bfx') \big|\\
		\leq & 2C\cdot L \cdot\big\|\bfx - \bfx'\big\|
	\end{aligned}
\end{equation}
where the first inequality from triangle inequality, the second inequality is from \cref{eq::bounded_GH} and the third inequality is from \cref{eq::lipschitz_GH}. This completes the proof.

\subsection{Proof of \cref{lemma::lipschitz_sum_prod_k}}\label{section::app::lemma::lipschitz_sum_prod_k}
The proofs are almost the same as those of \cref{lemma::lipschitz_sum_prod} by simply replacing $\bfx$  with $\bfxk$ and $\bfx'$ with $\bfxkp$. 

\subsection{Proof of \cref{lemma::lipschitzF}}\label{section::app::lemma::lipschitzF}

We conclude from \cref{eq::Fobj} that $F(\bfx)$ has bounded and Lipschitz continuous Euclidean gradients if and only if $\Fij$ has  bounded and Lipschitz continuous Euclidean gradients. From \cref{eq::Fij}, the Euclidean gradient  $\nabla \Fij$ is given by
\begin{subequations}\label{eq::DFij}
	\begin{equation}
		\nabla_{\bfR_i}\Fij =  -\nabla \rho\big(\|\bfe_{ij}\|^2\big)\cdot\frac{(\bfl_j-\bft_i)(\bfl_j-\bft_i)^\top}{\|\bfl_j-\bft_i\|^2}\bfR_i\left(\bfe_{ij}\bfp_{ij}^\top + \bfp_{ij}\bfe_{ij}^\top\right),
	\end{equation}
	\begin{equation}
		\nabla_{\bft_i} \Fij = \nabla\rho\big(\|\bfe_{ij}\|^2\big)\cdot \frac{\big(\bfl_j - \bft_i\big)^\top\bfR_i \bfp_{ij}}{\big\|\bfl_j - \bft_i\big\|^2}\cdot\bfR_i\bfe_{ij},
	\end{equation}
	\begin{equation}
		\nabla_{\bfd_i} \Fij =\nabla \rho\big(\|\bfe_{ij}\|^2\big)\cdot\begin{bmatrix}
			0 & 0 & 0\\
			0 & 0 & 0\\
			1  & \|\bfu_{ij}\|^2 & \|\bfu_{ij}\|^4
		\end{bmatrix} \left(\bfI-\frac{\bfR_i^\top(\bfl_j-\bft_i)(\bfl_j-\bft_i)^\top\bfR_i}{\|\bfl_j-\bft_i\|^2}\right)\bfe_{ij},
	\end{equation}
	\begin{equation}
		\nabla_{\bfl_j} \Fij = - \nabla\rho\big(\|\bfe_{ij}\|^2\big) \cdot \frac{\big(\bfl_j - \bft_i\big)^\top\bfR_i \bfp_{ij}}{\big\|\bfl_j - \bft_i\big\|^2}\cdot\bfR_i\bfe_{ij}
	\end{equation}
\end{subequations}
where $\bfp_{ij}$ and $\bfe_{ij}$ are from \cref{eq::error,eq::ray}, respectively. Note that all of these equations above consist of the sum and product of  $\bfR_i$, $\frac{\bfl_j - \bft_i}{\|\bfl_j - \bft_i\|}$, $\frac{1}{\|\bfl_j - \bft_i\|}$, $\bfp_{ij}$, $\bfe_{ij}$, $\nabla \rho\big(\|\bfe_{ij}\|^2\big)$. As a result of to \cref{lemma::lipschitz_sum_prod}, we only need to prove that these functions are bounded in App. \hyperref[section::app::lemma3::bounded]{D.4.1}  and Lipschitz continuous in App. \hyperref[section::app::lemma3::lipschitz]{D.4.2}

\vspace{0.5em}
\begin{enumerate}[before=\itshape,left=0pt]
	\item Proof of Boundedness \label{section::app::lemma3::bounded}
\end{enumerate}
\vspace{0.2em}

\begin{itemize}
	\item $\bfR_i$: Since $\bfR_i\in \SOthree$ is on a compact manifold, it is bounded.
	\vspace{0.5em}
	\item$\dfrac{\bfl_j - \bft_i}{\|\bfl_j - \bft_i\|}$: Note that $\dfrac{\bfl_j - \bft_i}{\|\bfl_j - \bft_i\|}$ is a unit vector under \cref{assumption:nonzero}, and as a result, bounded.
	\vspace{0.5em}
	\item $\dfrac{1}{\|\bfl_j - \bft_i\|}$: From \cref{assumption:nonzero}, it is straightforward to show $\left|\dfrac{1}{\|\bfl_j - \bft_i\|}\right| < \dfrac{1}{\epsilon}$.
	\vspace{0.5em}
	\item $\bfp_{ij}$: With \cref{assumption::bounded_intr} that $\bfd_i\in\Real{3}$ is bounded, \cref{eq::ray} suggests that $\|\bfp_{ij}\|$ is bounded.
	\vspace{0.5em}
	\item $\bfe_{ij}$: According to \cref{eq::error}, we obtain $\|\bfe_{ij}\|\leq \|\bfp_{ij}\|$. Then, since $\|\bfp_{ij}\|$ is bounded, $\bfe_{ij}$ is also bounded.
	\vspace{0.5em}
	\item $\nabla \rho\big(\|\bfe_{ij}\|^2\big)$: The boundedness of $\nabla \rho\big(\|\bfe_{ij}\|^2\big)$ is immediate from \cref{assumption::loss}\ref{assumption::loss::drho}.
\end{itemize}

\vspace{0.5em}
\begin{enumerate}[before=\itshape,left=0pt]
	\setcounter{enumi}{1}
	\item Proof of Lipschitz Continuity \label{section::app::lemma3::lipschitz}
\end{enumerate}
\vspace{0.2em}

%
\begin{itemize}
\item $\bfR_i$: Since $\|\bfR_i - \bfR_i'\| \leq \|\bfR_i - \bfR_i'\|$, we conclude that $\bfR_i$ is Lipschitz continuous.
\vspace{0.5em}
\item $\dfrac{\bfl_j - \bft_i}{\|\bfl_j - \bft_i\|}$: If $\|\bfl_j - \bft_i\|\neq 0$  and $\|\bfl_j' - \bft_i'\|\neq 0$, it can be shown that
\begin{equation}
\begin{aligned}
       &\left\|\frac{\bfl_j - \bft_i}{\|\bfl_j - \bft_i\|} - \frac{\bfl'_j - \bft'_i}{\|\bfl'_j - \bft'_i\|}\right\| \\
\leq & \left\|\frac{\bfl_j - \bft_i}{\|\bfl_j - \bft_i\|} - \frac{\bfl'_j - \bft'_i}{\|\bfl_j - \bft_i\|}\right\| + \left\|\frac{\bfl'_j - \bft'_i}{\|\bfl_j - \bft_i\|} - \frac{\bfl'_j - \bft'_i}{\|\bfl'_j - \bft'_i\|}\right\| \\
= &  \frac{ \|\bfl_j - \bft_i - \bfl'_j + \bft'_i\| }{\|\bfl_j - \bft_i\| } + \frac{\left| \|\bfl_j - \bft_i\| - \|\bfl'_j - \bft'_i\| \right|}{\|\bfl_j - \bft_i\| }\\
\leq &  \frac{ 2\|\bfl_j - \bft_i - \bfl'_j + \bft'_i\| }{\|\bfl_j - \bft_i\| }\\
\leq & \frac{ 2\|\bfl_j -  \bfl'_j \| + 2\|\bft_i -\bft'_i\| }{\|\bfl_j - \bft_i\| }\\
\leq & \frac{2}{\epsilon}\|\bfl_j -  \bfl'_j \| + \frac{2}{\epsilon}\|\bft_i -\bft'_i\|
\end{aligned}
\end{equation}
where the first three inequalities are from the triangle inequality and the last inequality is from  \cref{assumption:nonzero}, \ie, $\|\bfl_j - \bft_i\| > \epsilon$.

\vspace{0.5em}
\item $\dfrac{1}{\|\bfl_j - \bft_i\|}$: If $\|\bfl_j - \bft_i\|\neq 0$  and $\|\bfl_j' - \bft_i'\|\neq 0$, it can be shown that
\begin{equation}
	\begin{aligned}
		&\left|\frac{1}{\|\bfl_j - \bft_i\|} - \frac{1}{\|\bfl'_j - \bft'_i\|}\right| \\
		= &   \frac{\left| \|\bfl_j - \bft_i\| - \|\bfl'_j - \bft'_i\| \right|}{\|\bfl_j - \bft_i\| \cdot \|\bfl'_j - \bft'_i\|}\\
		\leq &  \frac{ \|\bfl_j - \bft_i - \bfl'_j + \bft'_i\| }{\|\bfl_j - \bft_i\| \cdot \|\bfl'_j - \bft'_i\|}\\
		\leq & \frac{ \|\bfl_j -  \bfl'_j \| + \|\bft_i -\bft'_i\| }{{\|\bfl_j - \bft_i\| \cdot \|\bfl'_j - \bft'_i\|} }\\
		\leq & \frac{1}{\epsilon^2}\|\bfl_j -  \bfl'_j \| + \frac{1}{\epsilon^2}\|\bft_i -\bft'_i\|
	\end{aligned}
\end{equation}
where the first two inequalities are from the triangle inequality and the last inequality is from  \cref{assumption:nonzero}, \ie, $\|\bfl_j - \bft_i\| > \epsilon$.

\vspace{0.5em}
\item $\bfp_{ij}$: As a result of \cref{eq::ray}, it is straightforward to show that $\bfp_{ij}$ is Lipschitz continuous with respect to $\bfd_i\in\Real{3}$.

\vspace{0.5em}
\item $\bfe_{ij}$: In \cref{eq::error}, it can be seen that $\bfe_{ij}$ consists of the sum and product of  $\bfR_i$, $\frac{\bfl_j - \bft_i}{\|\bfl_j - \bft_i\|}$, $\bfp_{ij}$, all of which have been proved to bounded and Lipschitz continuous under \cref{assumption:nonzero,assumption::bounded_intr}. Then, \cref{lemma::lipschitz_sum_prod} suggests that $\bfe_{ij}$ is Lipschitz continuous.

\vspace{0.5em}
\item $\nabla \rho\big(\|\bfe_{ij}\|^2\big)$: It can be shown that
\begin{equation}
\begin{aligned}
        &  \left|\nabla \rho\big(\|\bfe_{ij}\|^2\big) - \nabla \rho\big(\|\bfe'_{ij}\|^2\big)\right| \\
 \leq & L \left| \|\bfe_{ij}\|^2 - \|\bfe'_{ij}\|^2 \right| \\
 = & L  \left|\| \bfe_{ij}\| + \|\bfe'_{ij}\|\right| \cdot \left|\|\bfe_{ij}\| - \|\bfe'_{ij}\|\right| \\
 \leq &  L \left( \|\bfe_{ij}\| + \|\bfe'_{ij}\| \right) \cdot \|\bfe_{ij} - \bfe'_{ij}\|
\end{aligned}
\end{equation}
where the first inequality is from the Lipschitz continuity of $\nabla \rho(\cdot)$ in \cref{assumption::loss}\ref{assumption::loss::lipschitz} and the last inequality is from the triangle inequality. Note that we have proved before that $\bfe_{ij}$ is bounded and Lipschitz continuous. Thus,  the equation above indicates that $\nabla \rho\big(\|\bfe_{ij}\|^2\big)$ is Lipschitz continuous.
\end{itemize}
From the discussions above, we conclude that $\nabla \Fij$ is bounded and Lipschitz continuous, which, as discussed before, further suggests that $F(\bfx)$ is bounded and Lipschitz continuous. This completes the proof.
\vspace{0.5em}

\subsection{Proof of \cref{lemma::lipschitzE}}\label{section::app::lemma::lipschitzE}

Recall from \cref{eq::Ealpha} that $E(\bfx|\bfxk)$ consists of $\Fij$, $\Pijk$, $\Qijk$,  and we have proved in App. \hyperref[section::app::lemma::lipschitzF]{D.4} that $\Fij$ has bounded and Lipschitz continuous Euclidean gradients. Therefore, we only need to prove that $\nabla\Pijk$ and $\nabla\Qijk$ are bounded at $\bfxk$ and $\bfxkp$, and there exists a constant $L>0$ such that 
\begin{equation}\label{eq::Plipschitz}
\big\|\nabla P_{ij}\big(\bfcikp |\bfxk \big) - \nabla P_{ij}\big(\bfcik |\bfxk \big) \big\| \leq L \big\|\bfcikp - \bfcik\big\|
\end{equation}
and
\begin{equation}\label{eq::Qlipschitz}
\big\|\nabla Q_{ij}\big(\bfljkp |\bfxk \big) - \nabla Q_{ij}\big(\bfljk |\bfxk \big) \big\| \leq L \big\|\bfljkp - \bfljk\big\|
\end{equation}
for any $\sfk\geq 0$. As a result of \cref{eq::P,eq::Q}, $\nabla \Pijk$ and $\nabla \Qijk$ are given by 
\begin{equation}\label{eq::dPR}
\nabla_{\bfR_i} \Pijk =   2\wijk\cdot \big(\bfR_i\bfp_{ij} + \lambdaijk\cdot\bft_i - \iterate{\bfg}{ij}{k}\big)\bfp_{ij}^\top,
\end{equation}
\begin{equation}\label{eq::dPt}
	\nabla_{\bft_i} \Pijk =2\lambdaijk \cdot  \wijk \cdot \big(\bfR_i\bfp_{ij} + \lambdaijk\cdot\bft_i - \iterate{\bfg}{ij}{k}\big),
\end{equation}
\begin{equation}\label{eq::dPd}
	\nabla_{\bfd_i} \Pijk =2\lambdaijk \cdot  
	\begin{bmatrix}
		0 & 0 & 0\\
		0 & 0 & 0\\
		1 & \|\bfu_{ij}\|^2 & \|\bfu_{ij} \|^4
	\end{bmatrix} \bfR_i^\top \wijk \cdot  \big(\bfR_i\bfp_{ij} + \lambdaijk\cdot\bft_i - \iterate{\bfg}{ij}{k}\big),
\end{equation}
\begin{equation}\label{eq::dQl}
\nabla_{\bfl_j} \Qijk = 2 \wijk \cdot \big(\lambdaijk\cdot\bfl_j - \bfgijk\big).
\end{equation}
In these equations above, \cref{eq::w,eq::gamma} and \cref{assumption::bounded_intr,assumption:nonzero,assumption::loss}  indicate that
\begin{equation}\label{eq::bounded_w}
	0 \leq \wijk \leq 1
\end{equation}
and  $\lambdaijk$ is bounded for any $\sfk\geq 0$.
Then, as a result of \cref{eq::dPt,eq::dQl},  $\nabla_{\bft_i}  \Pijk$ and $\nabla_{\bfl_j} \Qijk$ are Lipschitz continuous. 
In addition,  $\bfR_i\bfp_{ij}$ is Lipschitz continuous according to \cref{lemma::lipschitz_sum_prod}. Since the sum of Lipschitz continuous functions is Lipschitz continuous,   $\bfR_i\bfp_{ij} + \lambdaijk\cdot\bft_i - \iterate{\bfg}{ij}{k}$ in \cref{eq::dPR,eq::dPd} is Lipschitz continuous. Furthermore, the boundedness of $\lambdaijk$ suggests that  there exists a constant $M>0$ such that
\begin{equation}\label{eq::lipschiz_Pij_err}
\big\|\big(\bfRikp\bfpijkp + \lambdaijk\cdot\bftikp -\bfgijk \big) - \big(\bfRik\bfpijk + \lambdaijk\cdot\bftik -\bfgijk \big)\big\| \leq M\cdot \|\bfcikp - \bfcik\|
\end{equation}
for any $\sfk\geq 0$.  Recall that we have proved that $\bfR_i$ and $\bfp_{ij}$ are bounded and Lipschitz continuous in App. \hyperref[section::app::lemma::lipschitzF]{D.2}. According to \cref{lemma::lipschitz_sum_prod_k},  we might conclude from \cref{eq::dPR,eq::dPt,eq::dPd,eq::dQl} that  $\nabla\Pijk$ and $\nabla\Qijk$  are bounded at $\bfxk$ and $\bfxkp$ and \cref{eq::Plipschitz,eq::Qlipschitz} hold for any $\sfk\geq 0$ by proving $\bfR_i\bfp_{ij} + \lambdaijk\cdot\bft_i - \iterate{\bfg}{ij}{k}$  and $\lambdaijk\cdot\bfl_j - \bfgijk$ are  bounded at $\bfxkp$ and $\bfxk$. In the following, we prove the boundedness of $\bfR_i\bfp_{ij} + \lambdaijk\cdot\bft_i - \iterate{\bfg}{ij}{k}$ and $\lambdaijk\cdot\bfl_j - \bfgijk$ at $\bfxkp$ and $\bfxk$.

With \cref{eq::g}, it is straightforward to show
\begin{equation}\label{eq::approx0}
\left\|\bfRik\bfpijk + \lambdaijk\cdot\bftik - \iterate{\bfg}{ij}{k}\right\| = \half\left\|\bfRik\bfpijk - \lambdaijk\cdot\big(\bfljk - \bftik\big)\right\| = \half\left\|\bfpijk - \lambdaijk\cdot{\bfRik}^\top\big(\bfljk - \bftik\big)\right\|
\end{equation}
where the last equality is from ${\bfRik}^\top\bfRik=\bfI$. Substituting \cref{eq::gamma} into \cref{eq::approx0} and simplifying the resulting equation with \cref{eq::error}, we obtain
\begin{equation}
\left\|\bfRik\bfpijk + \lambdaijk\cdot\bftik - \iterate{\bfg}{ij}{k}\right\| = \half\left\|\bfeijk\right\|.
\end{equation}
Note that we have proved in App. \hyperref[section::app::lemma3::bounded]{D.4.1} that $\|\bfe_{ij}\|$ is bounded under \cref{assumption:nonzero,assumption::bounded_intr}. Then, the equation above suggests that there exists a constant $C>0$ such that
\begin{equation}\label{eq::bounded_dPij1}
\left\|\bfRik\bfpijk + \lambdaijk\cdot\bftik - \iterate{\bfg}{ij}{k}\right\|\leq C
\end{equation}
for any $\sfk\geq 0$. In addition, it can be shown that
\begin{equation}\label{eq::lipschitz_dPij1}
\begin{aligned}
	& \big\|\bfRikp\bfpijk + \lambdaijk\cdot\bftikp - \iterate{\bfg}{ij}{k} \big\|  - \big\|\bfRik\bfpijk + \lambdaijk\cdot\bftik - \iterate{\bfg}{ij}{k} \big\| \\
\leq &  \left| \big\|\bfRikp\bfpijk + \lambdaijk\cdot\bftikp - \iterate{\bfg}{ij}{k} \big\|  - \big\|\bfRik\bfpijk + \lambdaijk\cdot\bftik - \iterate{\bfg}{ij}{k} \big\|  \right| \\
\leq & \big\|\big(\bfRikp\bfpijkp + \lambdaijk\cdot\bftikp -\bfgijk \big) - \big(\bfRik\bfpijk + \lambdaijk\cdot\bftik -\bfgijk \big)\big\| \\
\leq & M\cdot \|\bfcikp - \bfcik\|
\end{aligned}
\end{equation}
where the first and second inequalities are from triangle inequality, and the last inequality is from \cref{eq::lipschiz_Pij_err}. Then, as a result of \cref{eq::bounded_dPij1,eq::lipschitz_dPij1}, we obtain
\begin{equation}\label{eq::bounded_dPijkp0}
\big\|\bfRikp\bfpijk + \lambdaijk\cdot\bftikp - \iterate{\bfg}{ij}{k} \big\|  \leq  C  +  M\cdot \|\bfcikp - \bfcik\|
\end{equation}
for any $\sfk\geq 0$. Moreover, we have proved that $\|\bfxkp-\bfxk\|\rightarrow 0$ in App. \hyperref[section::app::prop3::x]{C.3.2}. Thus, $\|\bfxkp - \bfxk\|$ as well as $\|\bfcikp - \bfcik\|$ are bounded, i.e., there exists $N>0$ such that
\begin{equation}\label{eq::bounded_dPijkp1}
\|\bfcikp - \bfcik\| \leq \|\bfxkp - \bfxk\| \leq N
\end{equation}
for any $\sfk\geq 0$. Applying \cref{eq::bounded_dPijkp1} on the right-hand side of \cref{eq::bounded_dPijkp0}, we obtain
\begin{equation}\label{eq::bounded_dPijkp2}
\big\|\bfRikp\bfpijk + \lambdaijk\cdot\bftikp - \iterate{\bfg}{ij}{k} \big\|  \leq  C +  MN
\end{equation}
for any $\sfk\geq 0$. From \cref{eq::bounded_dPijkp2,eq::lipschitz_dPij1}, we conclude that $\bfR_i\bfp_{ij} + \lambdaijk\cdot\bft_i - \iterate{\bfg}{ij}{k}$ are bounded at $\bfxkp$ and $\bfxk$. In a similar way, it can  be shown that $\lambdaijk\cdot\bfl_j - \bfgijk$ are bounded at $\bfxkp$ and $\bfxk$. 

We have proved that \cref{eq::dPR,eq::dPd,eq::dPt,eq::dQl} consist of bounded and Lipschitz continuous functions. As discussed before, this suggests that $\nabla E(\bfxk|\bfxk)$ and $\nabla E(\bfxkp|\bfxk)$ are bounded and \cref{eq::lipschitz_E} holds for any $\sfk\geq 0$. In addition, with almost the same procedure, we can also prove that $\nabla E(\bflxk|\bflxk)$ and $\nabla E(\bfxkp|\bflxk)$ are bounded and and \cref{eq::lipschitz_lE} holds for any $\sfk\geq 0$. This completes the proof.

\end{appendices}

\end{document}